%% file: main.tex
\documentclass[10pt,a4paper]{article}

\usepackage[margin=1in]{geometry}
\usepackage[T1]{fontenc}    %
\usepackage{lmodern}        %
\usepackage{microtype}      %

\usepackage{setspace}
\usepackage[utf8]{inputenc} %
\usepackage[T1]{fontenc}    %
\usepackage{hyperref}       %
\usepackage{url}            %
\usepackage{booktabs}       %
\usepackage{amsfonts}       %
\usepackage[misc]{ifsym}   %
\usepackage{nicefrac}       %
\usepackage{microtype}      %
\usepackage{xcolor}         %

\usepackage[dvipsnames]{xcolor}
\hypersetup{
    colorlinks,
    linkcolor={Blue},
    citecolor={BlueViolet},
    urlcolor={BlueViolet}
}

\usepackage[square,numbers]{natbib}

\usepackage{bbm}
\usepackage{enumitem}
\usepackage{mathtools}
\usepackage{hyperref}
\usepackage{float}
\usepackage{tikz}
\usepackage{thmtools}

\usepackage{comment}

\usepackage[misc]{ifsym}

\usepackage{xcolor}

 \newcommand{\ch}[1]{}
 \newcommand{\as}[1]{}
 \newcommand{\kam}[1]{}

\DeclareMathOperator{\LS}{\mathrm{LS}}
\DeclareMathOperator{\GS}{\mathrm{GS}}
\DeclareMathOperator{\RS}{\mathrm{RS}}

\DeclareMathOperator{\median}{\mathrm{median}}
\DeclareMathOperator{\MST}{\mathrm{MST}}
\DeclareMathOperator{\SMT}{\mathrm{SMT}}
\DeclareMathOperator{\SVM}{\mathrm{SVM}}
\DeclareMathOperator{\MSE}{\mathrm{MSE}}
\DeclareMathOperator{\ERM}{\mathrm{ERM}}

\usepackage{microtype}
\usepackage{graphicx}
\usepackage{subcaption}
\usepackage{booktabs} %

\usepackage{hyperref}

\usepackage{graphicx}
\usepackage{subcaption}
\usepackage{caption}
\usepackage{array}
\usepackage{algorithm}
\usepackage{algorithmic}
\usepackage{wrapfig}

\usepackage{amsmath}
\usepackage{amssymb}
\usepackage{mathtools}
\usepackage{amsthm}

\usepackage[capitalize,noabbrev]{cleveref}
\usepackage{array}
\newcolumntype{H}{>{\setbox0=\hbox\bgroup}c<{\egroup}@{}}

\usepackage{amartya_ltx}
\usepackage{multirow}
\usepackage{makecell}
\usepackage{nicefrac}

\theoremstyle{plain}
\newtheorem{theorem}{Theorem}[section]

\newtheorem{lemma}[theorem]{Lemma}
\crefname{lemma}{Lemma}{Lemmas}
\Crefname{lemma}{Lemma}{Lemmas}
\usepackage{thmtools}
\usepackage{thm-restate}
\usepackage{algorithmic}

\crefname{assumption}{Assumption}{Assumptions}
\Crefname{assumption}{Assumption}{Assumptions}

\newtheorem{fact}{Fact}

\crefname{fact}{Fact}{Facts}
\Crefname{fact}{Fact}{Facts}

\newtheorem{corollary}[theorem]{Corollary}
\theoremstyle{definition}
\newtheorem{definition}[theorem]{Definition}
\newtheorem{example}[theorem]{Example}
\theoremstyle{remark}
\newtheorem{remark}[theorem]{Remark}

\AddToHook{env/lemma/begin}{\crefalias{theorem}{lemma}} %
\AddToHook{env/assumption/begin}{\crefalias{theorem}{assumption}} %
\AddToHook{env/example/begin}{\crefalias{theorem}{example}} %
\AddToHook{env/proposition/begin}{\crefalias{theorem}{proposition}} %
\AddToHook{env/definition/begin}{\crefalias{theorem}{definition}} %
\AddToHook{env/corollary/begin}{\crefalias{theorem}{corollary}} %
\AddToHook{env/remark/begin}{\crefalias{theorem}{remark}} %

\renewcommand\footnotemark{\textsuperscript{\Letter}}

\title{Less Noise, Same Certificate: Retain Sensitivity for Unlearning}

\date{}
\author{}
\author{%
  Carolin Heinzler\thanks{\textsuperscript{\Letter} Corresponding authors: \href{mailto:chei@di.ku.dk}{chei@di.ku.dk} and \href{mailto:amsa@di.ku.dk}{amsa@di.ku.dk}} 
  \quad
  Kasra Malihi  \quad
  Amartya Sanyal\textsuperscript{\Letter}
  \\\normalsize{Department of Computer Science, University of Copenhagen}
}

\begin{document}

\maketitle\vspace{-30pt}

\begin{abstract}%

\input{content/abstract}

\end{abstract}

\section{Introduction}
\label{sec:intro}
\input{content/intro.tex}

\section{Retain Sensitivity}
\label{sec:retain}

\input{content/retain}

\section{Passive Unlearning}
\label{sec:compute}
\input{content/retain_vs_global}

\section{Active Unlearning Algorithms}
\label{sec:active_unlearn}
\input{content/active_unlearn}

\section{Discussion and Conclusion}\label{sec:discussion}

\input{content/discussion}

\section{Acknowledgements}
AS acknowledge support from VILLUM FONDEN via the Young Investigator program
(72069) and the Novo Nordisk Foundation for support via the Startup
grant (NNF24OC0087820). We also thank Luka Radić and Giorgio Racca for their helpful comments.

\newpage
\bibliographystyle{plainnat}
\bibliography{references_new}

\crefalias{section}{appendix}
\newpage
\appendix
\onecolumn

\input{appendix/appendix.tex}

\input{appendix/dataset_description}

\end{document}

%% file: content/abstract.tex
Certified machine unlearning aims to provably remove the influence of a deletion set $U$ from a model trained on a dataset $S$, by producing an unlearned output that is statistically indistinguishable from retraining on the retain set $R \coloneqq S\setminus U$. Many existing certified unlearning methods adapt techniques from Differential Privacy (DP) and add noise calibrated to \emph{global sensitivity}, i.e., the worst-case output change over all adjacent datasets.
We show that this DP-style calibration is often overly conservative: unlike DP, unlearning requires no indistinguishability across different retain sets, only between each deletion output and retraining on the \emph{same retain set} $R$.
Motivated by this distinction, we define \emph{retain sensitivity} as the worst-case output change over deletions $U'$ from $R\cup U'$ while keeping $R$ fixed. While insufficient for DP, retain sensitivity is  exactly sufficient for unlearning, allowing for the same certificates with less noise. 
We validate these reductions in noise theoretically and empirically across several problems, including the weight of minimum spanning trees, PCA, and ERM. Finally, we refine the analysis of two widely used certified unlearning algorithms through the lens of retain sensitivity, leveraging the regularity induced by $R$ to further reduce noise and improve utility.

%% file: content/intro.tex
Machine learning models are increasingly deployed in settings where parts of the training data may later need to be removed. 
Individuals may exercise deletion rights (e.g., EU GDPR and the right to erasure (Art.17; \citet{gdpr2016})), and training pipelines may unintentionally contain poisoned \cite{schoepf2024potion}, copyrighted \cite{dou2025avoiding}, or otherwise impermissible examples \cite{thiel2023identifying} that must subsequently be deleted. While retraining from scratch is the gold standard for deletion, it is often computationally prohibitive or operationally infeasible. These pressures have motivated the study of \emph{certified machine unlearning} \cite{cao2015towards}: given a model trained on $S$ and a deletion set $U \subset S$, the goal is to efficiently produce an unlearned model whose distribution is statistically indistinguishable from that of retraining on the retain set $R \coloneqq S \setminus U$.

Certified unlearning is closely connected to \emph{Differential Privacy} (DP), which requires a learning algorithm’s output distribution to be indistinguishable across all neighbouring datasets \cite{dwork2006calibrating}. 
Many DP mechanisms achieve this by adding noise calibrated to the algorithm’s \emph{global sensitivity}, i.e., the largest possible change in output between datasets that differ in one datapoint. This adjacency-based indistinguishability implies a natural \emph{passive} route to unlearning: a differentially private model already satisfies the unlearning guarantee for deleting any single sample without any post-hoc update, and extends to the deletion of multiple datapoints by invoking group privacy.
Notably, it also avoids releasing a second model, and thus trivially prevents model-differencing attacks, which can recover information about $U$ or $R$ when both the pre- and post-deletion models are observed \cite{chen2021machine, courasia2023forget, bertran_reconstruction_2024, cohen2026protectingundeletedmachineunlearning}.
However, because global sensitivity must account for the worst-case change across \emph{all} potential datasets, this passive approach often incurs a substantial utility cost due to the injection of excessive noise \cite{bassily2014private}.

To improve utility, many \emph{active} certified unlearning methods perform an explicit update and add calibrated noise, often using DP-style sensitivity analyses \cite{guo2019certified, neel2021descent, sekhari2021RememberWhatYou, allouah2024utility} (see \Cref{sec:related_work} for background on passive vs.\ active unlearning and its connections to DP). In these works, however, the noise is still calibrated to \emph{global} sensitivity.
Certified unlearning has a narrower scope: it compares the unlearned output after deleting $U$ with retraining on the same retained data $R=S\setminus U$. This suggests that the minimum noise required for certification may be calibrated to the specific properties of $R$, rather than to a worst-case over all possible datasets. 
We therefore ask: 

\begin{center}
    \emph{What is the minimum noise scale required for certified unlearning relative to retraining on $R$?}
\end{center}

We study this question in a setting where we assume that the unlearning algorithm has access to the full retain set $R$ and we introduce \emph{Retain Sensitivity} (RS). We define RS at each retain set $R$ as the worst-case change in the algorithm's output between $R$ and $R \cup U'$ over all possible deletion sets $U'$. 
By construction, retain sensitivity is always upper-bounded by the dataset-dependent notion of local sensitivity at $R$ and can be significantly smaller than the global sensitivity of a function.
This shows a fundamental separation between DP and certified unlearning: while local-sensitivity is known to be insufficient for DP \cite{nissim2007smooth, dwork2009differential}, we show that retain sensitivity is a sufficient -- and, in some cases, necessary -- noise scale to certify unlearning. We further show that this reduction in noise can lead to improved utility for unlearning.

\noindent\textbf{Contributions.} Our contributions are threefold:
\begin{itemize}[leftmargin=5mm,topsep=1pt,itemsep=1pt]
    \item First, in~\Cref{sec:retain}, we build intuition and formally define retain sensitivity to establish it as a sufficient and necessary quantity  to calibrate noise for passive and active unlearning algorithms.
    \item Second, in \Cref{sec:compute} we derive retain sensitivity bounds for canonical problems (weight of MST, PCA, SVM and ERM). We prove and empirically demonstrate that these bounds substantially reduce noise compared to global sensitivity.
    Across these problems, the gains come from stability across neighbouring datasets around the retained set $R$: better conditioning (e.g., weight distribution in a graph, eigengap, margin, curvature) limits how much any $U'$ can change the output.
    \item Finally, in~\Cref{sec:active_unlearn}, we adapt two widely used certified \emph{active} unlearning algorithms,
    Descent-to-Delete~\citep{neel2021descent} and Newton Update~\citep{sekhari2021RememberWhatYou},
    to use retain sensitivity calibration. 
    In particular, we replace worst-case strong convexity with data-dependent curvature bounds, which yields less noise; we further confirm empirically that this lower noise translates into improved utility under the same unlearning certificate.
\end{itemize}

\subsection{Related work}

For background on certified machine unlearning and differential privacy, we refer to \Cref{sec:related_work}. Here we discuss closely related work.

Concurrent work by \cite{regehr2026nearoptimalmachineunlearningutility} explicitly exploits a shared retain-set reference distribution to achieve near-optimal excess risk. Our work develops the same underlying perspective through retain sensitivity, providing a general framework for retain-dependent noise calibration.

Concurrent work by \cite{benarroch2026certifiedperinstanceunlearningusing} uses individual sensitivity bounds to obtain per-instance unlearning guarantees, allowing both the certificate parameters and the reference procedure to depend on the forget set. In contrast, we establish the stronger standard unlearning guarantee for fixed $(\varepsilon,\delta)$ uniformly over admissible datasets and deletion requests, while allowing the noise scale to depend on the shared retain set $R$.

Subsequent work on auditing by \cite{sarmasarkar2026auditingunlearningalgorithms} also exploits the shared retain-set reference: $(\varepsilon,\delta)$-unlearning implies $(2\varepsilon,(1+e^\varepsilon)\delta)$-indistinguishability between unlearned output distributions corresponding to different forget sets $U,U'$ with the same retain set $R$.
They use this shared reference for auditing, whereas we use it directly for noise calibration in unlearning algorithms.\looseness=-1

%% file: content/retain.tex
\subsection{Preliminaries and Background}

Let \(\mathcal{X}\) be the input domain, \(\mathcal{Y}\) the label set, and \(\mathcal{Z}=\mathcal{X}\times\mathcal{Y}\), and $\mathcal{Z}^*$ the set of finite samples over $\mathcal{Z}$. For datasets \(S,S'\in\mathcal{Z}^*\), let \(d_{\pm}(S,S')\) be the add/remove distance. Unless stated otherwise, \(\|\cdot\|\) denotes the relevant output norm, e.g. \(\|\cdot\|_2\) for vectors and \(\|\cdot\|_{\mathrm{op}}\) for matrices. A learner \(\mathcal{A}:\mathcal{Z}^*\to\mathcal{W}\) maps a dataset to a model, and an unlearning mechanism \(\bar{\mathcal{A}}\) maps a deletion request \(U\subseteq S\), the trained model \(\mathcal{A}(S)\), and side information \(T(S)\) to an updated output \(\bar{\mathcal{A}}(U,\mathcal{A}(S),T(S))\).\\
For distributions $\mathcal{M}_0,\mathcal{M}_1$, write  \(\mathcal{M}_0\approx^{\varepsilon,\delta}\mathcal{M}_1\) if they are \((\varepsilon,\delta)\)-indistinguishable, i.e. for all measurable \(W\) both inequalities
    $\Pr[\mathcal{M}_i\in W]
    \le
    e^\varepsilon \Pr[\mathcal{M}_j\in W]+\delta,$ for $ i\neq j\in\{0,1\}$ hold.

Standard unlearning formulations consider a dataset \(S\) of size at most \(n+m\), a forget set \(U\subseteq S\) with \(\lvert U\rvert\le m\), and an implicit \emph{retain set} \(R=S\setminus U\). We use the guarantee of \citet{sekhari2021RememberWhatYou}, restated equivalently in terms of \(R\), to make explicit that the guarantee is for fixed retained data. \looseness=-1

\begin{definition}[$(\varepsilon,\delta)$-Unlearning {\cite{sekhari2021RememberWhatYou}}]\label{def:unlearning_indist}
Let $n,m\geq 1$ and $\varepsilon>0,\delta\in(0,1)$. A learning--unlearning algorithm pair ($\mathcal{A},\bar{\mathcal{A}}$) satisfies $(\varepsilon,\delta)$-unlearning if for every pair of datasets $R, U$ with $|U|\leq m$ and $\lvert R\rvert+\lvert U\rvert=n+m$,\looseness=-1
\begin{equation*}
\bar{\mathcal{A}}(U,\mathcal{A}(R\cup U),T(R\cup U))\approx^{\varepsilon,\delta} \bar{\mathcal{A}}(\emptyset,\mathcal{A}(R),T(R)).
\end{equation*}
\end{definition}

Throughout, we fix \(\lvert R\rvert=n\) and consider unlearning exactly \(m\le n-1\) points. We restrict to methods where the learner \(\cA\) and update map are deterministic, with randomness only from additive noise, though we do not view this as inherent to our results. This covers the two methods we study \citep{neel2021descent,sekhari2021RememberWhatYou} and the standard DP paradigm of releasing a deterministic quantity plus calibrated noise. Within this class, we call a method \emph{passive} if it leaves the current model unchanged, and \emph{active} if it first approximates retraining before adding noise.

\begin{definition}[Active/Passive Unlearning]\label{def:approx-based}
A learning--unlearning pair $(\cA,\bar{\cA})$ is \emph{active} if there exist a map $\bar{\cA}_0$ and noise rule $\cD$ such that for all $R,U$,
$$\bar{\cA}\br{U,\cA\br{R\cup U},T\br{R\cup U}}
= \bar{\cA}_0\br{U,\cA\br{R\cup U},T\br{R\cup U}}+\nu,$$ 
with $\nu\sim \cD(U, T(R\cup U))$ drawn independently in each call. %
It is \emph{passive} if moreover $$\bar{\cA}_0\br{U,\cA\br{R\cup U},T\br{R\cup U}}=\cA(R\cup U)$$ for all $R,U$.
\end{definition}

Throughout, we assume the noise distribution $\mathcal D$ to be Gaussian, though any $\mathcal D$ satisfying the required $(\varepsilon,\delta)$-unlearning guarantee suffices. Lastly, we formally introduce differential privacy.

\begin{definition}[$(\varepsilon,\delta)$-Differential Privacy (DP) {\cite{dwork2006calibrating}}]\label{def:dp_indist}
Let $\varepsilon > 0, \delta\in (0,1)$. A randomised mechanism $\mathcal{M}$ with range $\mathcal{W}$ is $(\varepsilon,\delta)$-differentially private if $\mathcal{M}(S)\approx^{\varepsilon,\delta}\mathcal{M}(S')$ for all neighbouring datasets $S, S'$ with $d_{\pm}(S,S')=1$.
\end{definition}

A common way to privatise the output of a function $f$ is \emph{output perturbation}~\citep{chaudhuri2011differentially} where the noise scale is set by a sensitivity bound. We first recall the notion of \emph{global sensitivity}. Throughout, we define sensitivities for changes involving at most $m$ records; $m=1$ recovers the standard notions used in differential privacy.

\begin{definition}[Global Sensitivity]
    Let $m\in\mathbb{Z}^+$ and $f:\mathcal{Z}^*\rightarrow \mathcal{W}$, define global sensitivity:
    \begin{equation*}\label{eq:global_sens_dp}
        \GS_f^{(m)}:=\max_{S,S':\, d_{\pm}(S,S')\leq m}\lVert f(S)-f(S')\rVert.
    \end{equation*}

\end{definition}

We will use the following fact about the Gaussian mechanism repeatedly:

\begin{fact}\label{fact:gaussian_mech}\cite{dwork2014AlgorithmicFoundationsDifferential}
Let $\varepsilon\in (0,1],\delta\in (0,1)$. If a mechanism $\mathcal{M}(S)=f(S)+\mathcal{N}(0,\sigma^2 I_d)$ satisfies $\sigma \ge \GS_f^{(1)}\cdot  c_{\varepsilon,\delta}$, for $c_{\varepsilon,\delta}=\tfrac{1}{\varepsilon}\sqrt{2\log(1.25/\delta),}$ then $\mathcal{M}$ is $(\varepsilon,\delta)$-DP.
\end{fact}

Global sensitivity yields worst-case noise and can be overly conservative on typical datasets. 
A data-dependent alternative is \emph{local sensitivity}, defined by
\(\LS_f^{(m)}(S):=\max_{S':\,d_{\pm}(S,S')\le m}\|f(S)-f(S')\|\).
While this can be much smaller at a given dataset \(S\), calibrating noise directly to local sensitivity does not in general yield DP, since the noise scale itself depends on the dataset and can vary across neighbouring datasets \citep{nissim2007smooth,dwork2009differential}.

\subsection{Retain Sensitivity}

For unlearning, unlike DP, the guarantee is stated for a fixed retained set $R=S\setminus U$: to satisfy \Cref{def:unlearning_indist}, we need to hide what has been removed, but not what remains.
This motivates the following new sensitivity notion:

\begin{definition}[Retain Sensitivity] 
    Let $m\in\mathbb{Z}^+$ and $f:\mathcal{Z}^*\rightarrow \mathcal{W}$, define \emph{retain sensitivity} at $R$:
    \begin{equation*}                
        \RS_f^{(m)}(R):=\max_{Z\in\mathcal{Z}^*:\, \lvert Z\rvert\leq m}\lVert f(R\cup Z)-f(R)\rVert.
    \end{equation*}
\end{definition}
This directly implies ~\(\RS_f^{(m)}(R)\leq \LS_f^{(m)}(R)\leq \GS_f^{(m)}\) for any dataset $R$, any $m\geq 1$ and function $f$. \looseness=-1

\begin{remark}
A related one-sided notion is \emph{down sensitivity} \cite{asi2020instance}, which measures the maximum change under deletions only: $\mathrm{DS}_f^{(m)}(S)=\max_{Z\subseteq S, \lvert Z\rvert\leq m}\|f(S)-f(S\setminus Z)\|$. While this may appear well-suited to certifying deletions, it depends on the full dataset \(S=R\cup U\) and in particular on \(U\). 
Hence, this does not certify a guarantee determined only by \(R\): the unlearning output must be indistinguishable from retraining on \(R\), but noise calibrated based on \(U\) does not establish that indistinguishability when retraining has no access to \(U\).
\end{remark}

To use retain sensitivity for active and passive unlearning, we consider $f=\bar{\mathcal{A}_0}$ for the deterministic update map $\bar{\mathcal{A}}_0$ in the active case, which reduces to $f=\mathcal{A}$ in the passive setting:

\begin{definition}[Retain Sensitivity for Unlearning]\label{def:retain_unlearn}
     For a learning--unlearning algorithm pair $(\mathcal{A},\bar{\mathcal{A}})$, define \emph{retain sensitivity for unlearning} at $R$:
    \begin{equation*}
            \RS_{(\mathcal{A},\bar{\mathcal{A}})}^{(m)}(R):=\max_{Z\in\mathcal{Z}^*:\, \lvert Z\rvert\leq m}\lVert \bar{\mathcal{A}}_0(\emptyset, \mathcal{A}(R),T(R))
            -\bar{\mathcal{A}}_0(Z,\mathcal{A}(R\cup Z),T(R\cup Z))\rVert.
    \end{equation*}

\end{definition}

Thus, \(\RS_{(\mathcal{A},\bar{\mathcal{A}})}^{(m)}(R)\) measures the largest change induced by unlearning any set \(Z\) of at most \(m\) points from an extended dataset \(R\cup Z\), while keeping the retained data \(R\) fixed. In particular, for the actual request \(U\) with \(S=R\cup U\), it holds that $$\RS_{(\mathcal{A},\bar{\mathcal{A}})}^{(m)}(R)
\geq
\left\|
\bar{\mathcal{A}}_0(U,\mathcal{A}(S),T(S))
-
\bar{\mathcal{A}}_0(\emptyset,\mathcal{A}(R),T(R))
\right\|.$$

\subsection{Unlearning using Retain Sensitivity}

We now state the unlearning certificate. Assuming the full-information setting $T(R\cup U)=R\cup U$ (in order for unlearning and retraining both to reconstruct $R$), the noise required for $(\varepsilon,\delta)$-unlearning can be calibrated to the retain sensitivity evaluated at the retained set $R$. 
This calibration can be performed at unlearning time: once the unlearning algorithm receives the unlearning request $U$, it can recover the retain set $R=S\setminus U$ and use it to set the corresponding noise scale; retraining on $R$ uses the same scale.

\begin{theorem}[Sufficiency] \label{thm:retain_unlearn}
Let $(\mathcal{A},\bar{\mathcal{A}})$ be a learning--unlearning pair with Gaussian noise, i.e. $\bar{\cA}(U,\cA(R\cup U),R \cup U)=\bar{\cA}_0(U,\cA(R\cup U),R\cup U)+\nu$ with $\nu\sim\mathcal{N}(0,\sigma^2 I_d)$ sampled independently and $\sigma$ may depend on the available information, i.e. $(U,R\cup U)$ or $(\emptyset, R)$. If ${\sigma= \sigma (R)=\nicefrac{\RS_{(\mathcal{A},\bar{\mathcal{A}})}^{(m)}(R)}{\varepsilon}\sqrt{2\log(\nicefrac{1.25}{\delta})}}$, then $(\mathcal{A},\bar{\mathcal{A}})$ satisfies $(\varepsilon,\delta)$-unlearning for all $\varepsilon,\delta\in(0,1)$.\looseness=-1
\end{theorem}

\begin{proof}
    Let 
\(\mu_1:=\bar{\mathcal{A}}_0(U,\mathcal{A}(R\cup U),R\cup U)\) and 
\(\mu_2:=\bar{\mathcal{A}}_0(\emptyset,\mathcal{A}(R),R)\). 
In both executions, the noise is calibrated to the retained set after the unlearning request: \((R\cup U)\setminus U=R\) for unlearning, and \(R\setminus\emptyset=R\) for retraining. Hence both outputs add the same Gaussian noise and have common covariance \(\sigma(R)^2 I_d\). Moreover, by the definition of \(\RS_{(\mathcal{A},\bar{\mathcal{A}})}^{(m)}(R)\), the choice \(Z=U\) gives
\(\|\mu_1-\mu_2\|\le \RS_{(\mathcal{A},\bar{\mathcal{A}})}^{(m)}(R)\).
The privacy loss\footnote{Define the privacy loss as $L=\log(\tfrac{p_1(X)}{p_2(X)})$, for $X\sim \mathcal{N}(\mu_1,\sigma(R)^2 I_d)$ where $p_1, p_2$ are the densities of $\bar{\mathcal{A}}(U,\mathcal{A}(R\cup U),R\cup U)$ and $\bar{\mathcal{A}}(\emptyset,\mathcal{A}(R),R))$ resp.} between two Gaussians with common covariance \(\sigma(R)^2 I_d\) is normally distributed with mean 
\(\|\mu_1-\mu_2\|^2/(2\sigma(R)^2)\) and variance 
\(\|\mu_1-\mu_2\|^2/\sigma(R)^2\). Thus, applying the standard Gaussian mechanism bound in \Cref{fact:gaussian_mech} yields the claimed \((\varepsilon,\delta)\)-unlearning guarantee.\looseness=-1
\end{proof}

Since both retraining and unlearning can reconstruct the same retain set, the mechanism can use the same noise law, e.g. \(\mathcal{N}(0,\sigma(R)^2 I_d)\). DP, by contrast, compares arbitrary neighbouring datasets with no shared retained core; making the noise scale depend on the input, as with local sensitivity, can itself reveal which dataset was used and thereby break privacy.\\
\Cref{thm:retain_unlearn} immediately yields the following corollary for a passive unlearning algorithm:

\begin{corollary}
    Let $(\mathcal{A},\bar{\mathcal{A}})$ be a passive learning--unlearning algorithm pair, i.e. 
    $\bar{\mathcal{A}}(U,\mathcal{A}(R\cup U),T(R\cup U))=\mathcal{A}(R\cup U)+\nu$ 
    with $\nu\sim\mathcal{N}(0,\sigma^2 I_d)$  sampled independently and $\sigma$ may depend on the available information, i.e. $(U,R\cup U)$ or $(\emptyset, R)$.\\
    If ${\sigma= \sigma (R)=\frac{\RS_\mathcal{A}^{(m)}(R)}{\varepsilon}\sqrt{2\log(\frac{1.25}{\delta})}}$, then $(\mathcal{A},\bar{\mathcal{A}})$ satisfies $(\varepsilon,\delta)$-unlearning for any $\varepsilon,\delta\in(0,1)$.
\end{corollary}

\begin{remark}[Necessity]
Retain sensitivity also captures the noise scale that a Gaussian unlearning mechanism must mask. Fix \(R\) and suppose the mechanism outputs \(\bar{\mathcal A}_0(\cdot)+\nu\) with \(\nu\sim\mathcal N(0,\sigma^2 I_d)\). For any \(|Z|\le m\), the unlearning and retraining outputs are Gaussians with the same covariance and mean shift
$\left\|
\bar{\mathcal A}_0(Z,\mathcal A(R\cup Z),T(R\cup Z))
-
\bar{\mathcal A}_0(\emptyset,\mathcal A(R),T(R))
\right\|$.
If this shift is too large relative to \(\sigma\), the Gaussian mean-shift lemma \citep{balle2018improving} gives a test distinguishing the two executions, violating \((\varepsilon,\delta)\)-unlearning. Thus, uniform certification over all \(|Z|\le m\) requires noise large enough to mask the worst such shift, namely \(\RS_{(\mathcal A,\bar{\mathcal A})}^{(m)}(R)\).

\end{remark}

%% file: content/retain_vs_global.tex
In this section we compare retain sensitivity and global sensitivity
across several canonical problems in statistics, machine learning, and
theoretical computer science. Although these problems are standard, their global sensitivity is dictated by worst-case datasets, often requiring substantial noise and making them well-studied examples in differential privacy \cite{nissim2007smooth, dwork2009differential, dwork2014gauss, Haghifam2024privategeometricmedian}. In contrast, retain sensitivity is governed by the stability of the fixed retain set  $R$ and can be orders of magnitude smaller. Consequently, for passive unlearning of $U$, noise calibrated to retain sensitivity can be dramatically smaller than DP-style output perturbation calibrated to global sensitivity, highlighting the separation between DP and passive unlearning.

Our primary focus in this work is on \emph{how much noise is
necessary} to certify unlearning and we do not address computational efficiency here.  An
important question for future work is how to compute the key data-dependent
statistic in each problem (e.g., local spacing for the median, weight distribution in the MST, an eigengap for PCA, a margin for SVM, or curvature
for ERM) \emph{efficiently} from the retain set \(R\),
without incurring the costs of full retraining or expensive
post-processing in both passive and active unlearning.

\subsection{Median}\label{app:subsec_median}

We begin with the one-dimensional median, which provides a clean illustration: global sensitivity depends on the domain bound, while retain sensitivity depends on the local spacing around the median.\\
Let $x_{(1)}\le\cdots\le x_{(n)}$ be the sorted values and define
$\median(\{x_i\}_{i=1}^n):=(x_{(\lfloor (n+1)/2\rfloor)}+x_{(\lceil (n+1)/2\rceil)})/2$. Assume $n\geq3$ is 
odd (even-$n$ is analogous).

\begin{restatable}{lemma}{rsmedian}
Let \(t=(n+1)/2\), s.t. \(\median(R)=x_{\br{t}}\).
If \(x_i\in\bs{0,B}\), then the ratio between the retain sensitivity and global sensitivity of the median is  $$\tfrac{\RS_{\median}^{(m)}(R)}{\GS_{\median}^{(m)}}
\asymp\tfrac{\max\{x_{(t+\lceil m/2\rceil)}-x_{(t)}, x_{(t)}-x_{(t-\lceil m/2\rceil )}\}}{B}.$$
\end{restatable}

The proof, deferred to \Cref{app:subsec_median}, compares local and worst-case median shifts. Thus, retain sensitivity can be much smaller on average than global sensitivity when the data are well spaced around the median, which holds for several common distributions, as the next example illustrates.
\begin{example}
Let \(\cD\) be a continuous distribution with CDF \(F\) and density
\(f\), and let \(t=F^{-1}\br{\nicefrac{1}{2}}\) be its population
median. Assume that \(f\) is continuous and positive in a neighbourhood of
\(t\). For \(R\overset{i.i.d.}{\sim}\cD^n\) and fixed $m$, the expected retain
sensitivity scales as \(\bE\bs{\RS_{\median}^{(m)}(R)} \asymp
\nicefrac{m}{nf(t)}\).  This follows by applying the
probability integral transform \(U_i=F(X_i)\sim
\mathrm{Unif}\bs{0,1}\), using the fact that the
relevant order-statistic gaps around \(1/2\) are of order \(m/n\), and a first-order Taylor expansion of
\(F^{-1}\) around \(\nicefrac{1}{2}\).
\end{example}

\subsection{Minimum Spanning Tree (MST) Weight}

Releasing the weight of a minimum spanning tree is a classical problem and a standard example where local neighbouring datasets can be much better behaved than the worst case \citep{nissim2007smooth}. For a weighted graph \(G=(V,E,w)\) with weights in \([0,B]\), let \(f(G):=w(\MST(G))\) denote the weight of a MST.

We focus here on edge-addition adjacency and defer vertex adjacency to \Cref{sec:appendix_vertex}. For a retained graph \(R=(V_R,E_R,w)\), let \(A(R):=\binom{V_R}{2}\setminus E_R\) be the set of addable edges, and write \(R+F\) for adding a set \(F\subseteq A(R)\) with arbitrary weights in \([0,B]\). In the unlearning setting \(R=S\setminus U\)
with \(|U|=m\), the removed edges satisfy \(U\subseteq A(R)\), and hence
\(|A(R)|\ge m\). The corresponding retain sensitivity is
$\RS_f^{(m)}(R)
:=
\max_{F\subseteq A(R): |F|\le m}
\left|f(R)-f(R+F)\right|$,
where the maximisation also ranges over the weights of the added edges.

\begin{lemma}\label{lem:mst_edge}
	For a graph \(R=\br{V_R,E_R}\) fix $T\in \MST(R)$. For $m\geq 1$, define the family of jointly replaceable edges $I_m(T)=\{X\subseteq E(T)\mid \lvert X\rvert \leq m, T-X+A(R) \text{ connected}\}$. Then
    \begin{equation*}
        \tfrac{\RS_{f}^{(m)}\br{R}}{\GS_{f}^{(m)}}=\tfrac{\max_{X\in I_m(T)} 
\sum_{h\in X} w(h)}{mB}.
    \end{equation*}
\end{lemma}

\begingroup
\renewcommand{\proofname}{Proof sketch}
\begin{proof}
The global sensitivity is \(\GS_f^{(m)}=mB\): each added edge can reduce the MST weight by at most \(B\), and this bound is tight (see \Cref{lem:gs-mst}). For retain sensitivity, adding a cheap edge creates a cycle with the current MST and can replace the heaviest edge on that cycle. Thus, for \(m=1\), the best change is the heaviest replaceable edge of \(T\). For general \(m\), the same reasoning gives the maximum total weight of a jointly replaceable set of at most \(m\) MST edges. The formal argument is deferred to \Cref{lem:MST-rs}. Moreover, \Cref{lem:MST-rs} proves that such a set can be constructed greedily, by iteratively finding the heaviest currently replaceable edge and adding a zero-weight edge that takes its place in the MST.\looseness=-1

\end{proof}
\endgroup

In \Cref{subfig:MST}, we visualise the uniform bound
\(\RS_f^{(m)}(R)/\GS_f^{(m)} \le w(h_1)/B\), where \(h_1\) is the heaviest replaceable edge in an MST of \(R\). The bound holds for all \(m\) and is tight for \(m=1\) by \Cref{lem:mst_edge}. Across four real-world weighted networks \citep{de2015identifying,kunegis2013konect,pitoski2021network,kumar2016edge}, we observe substantial variation: the ratio is close to one for Bitcoin but orders of magnitude smaller for Migration. This shows how a few large outliers (as in the Migration graph) can inflate global sensitivity while leaving retain sensitivity much smaller.

\begin{figure*}[t]
  \vskip 0.2in
  \begin{center}
  \begin{subfigure}[T]{0.24\textwidth}
      \centering
      \vskip 0pt
      \includegraphics[width=\linewidth]{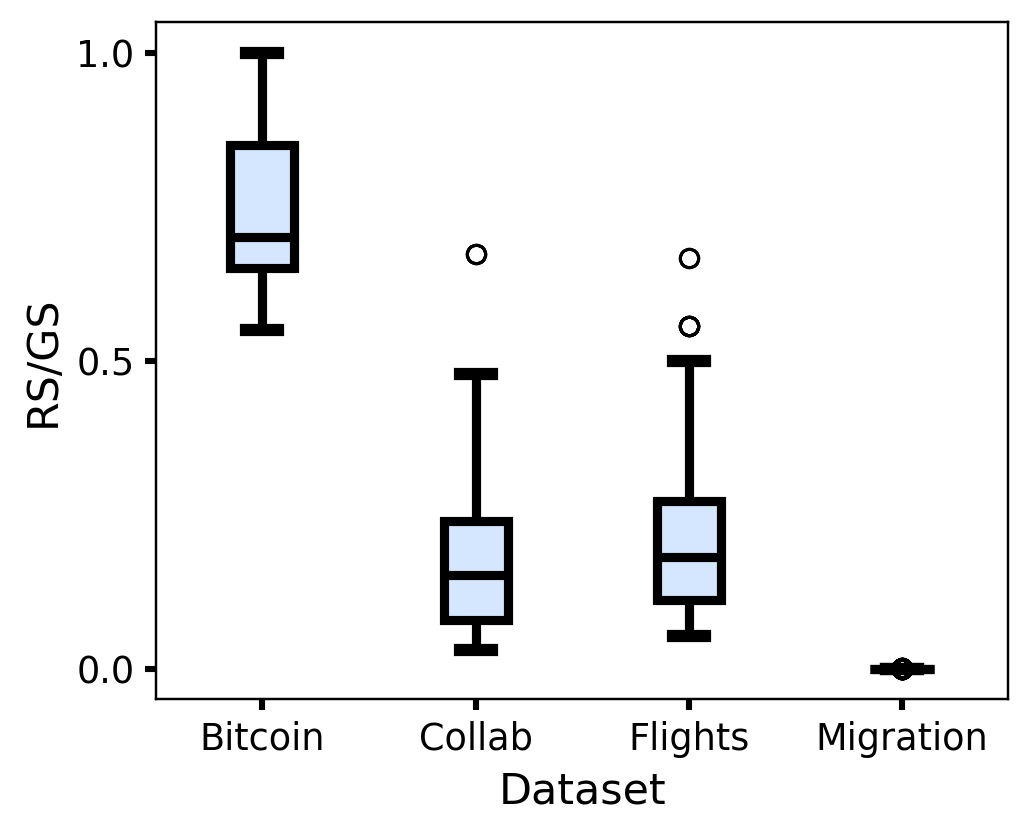}
      \caption{MST}
      \label{subfig:MST}
  \end{subfigure}
\begin{subfigure}[T]{0.24\textwidth}
      \centering
      \includegraphics[width=\linewidth]{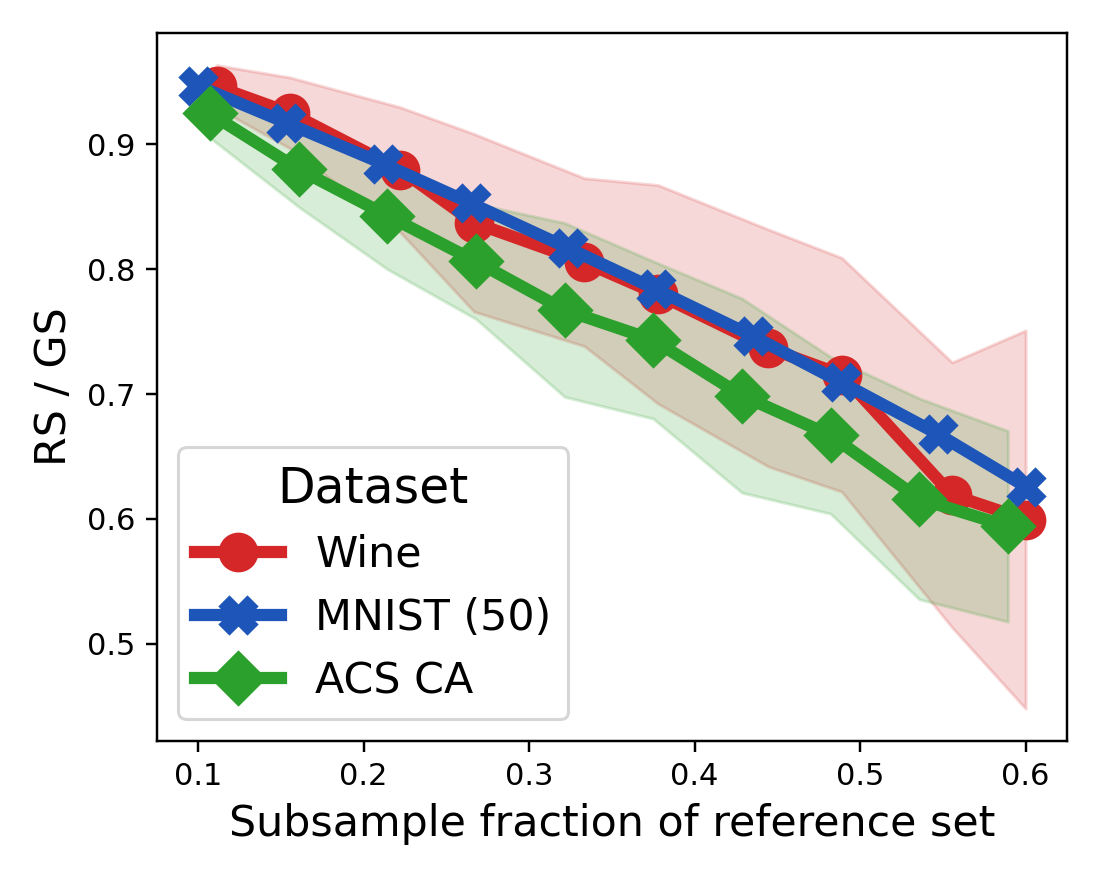}
      \caption{SVM}
      \label{subfig:passive_svm}
  \end{subfigure}
  \begin{subfigure}[T]{0.24\textwidth}
      \centering
      \vskip 0pt
      \includegraphics[width=\linewidth]{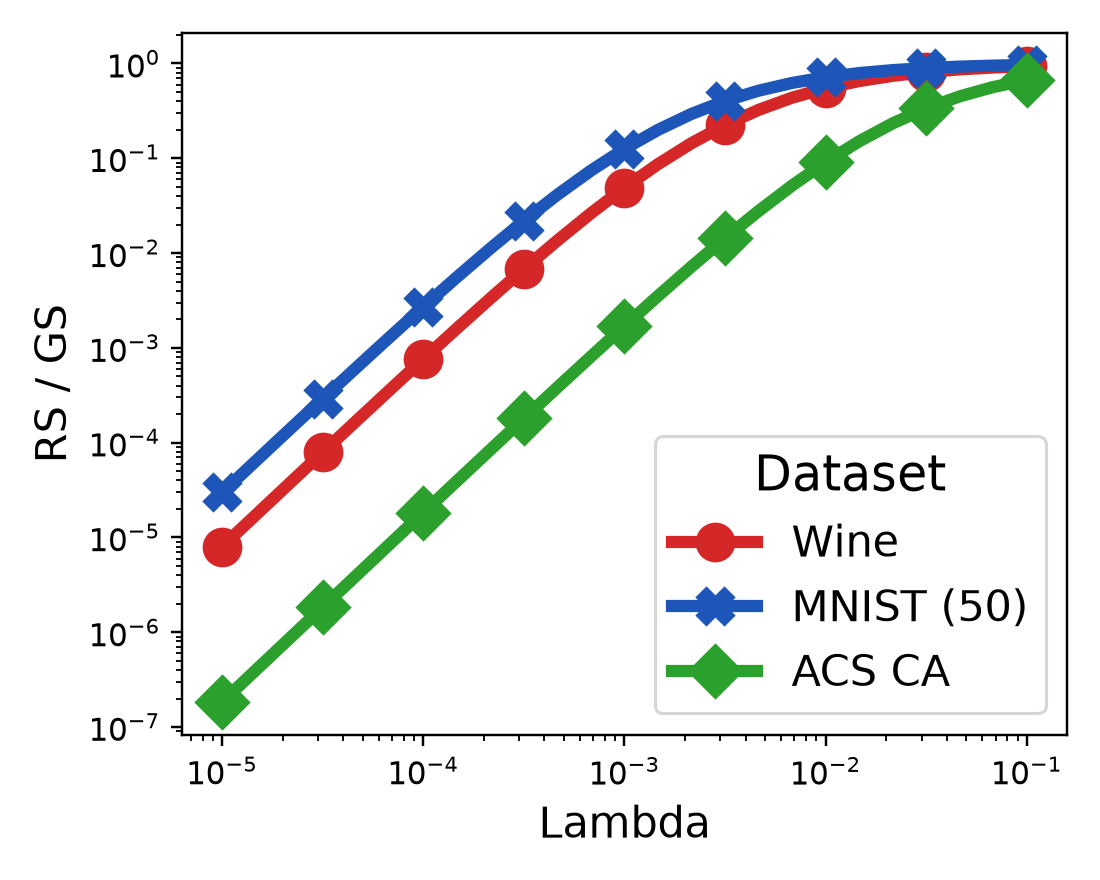}
      \caption{Passive MSE}
      \label{subfig:passive_erm_mse}
  \end{subfigure}
    \begin{subfigure}[T]{0.24\textwidth}
      \centering
      \includegraphics[width=\linewidth]{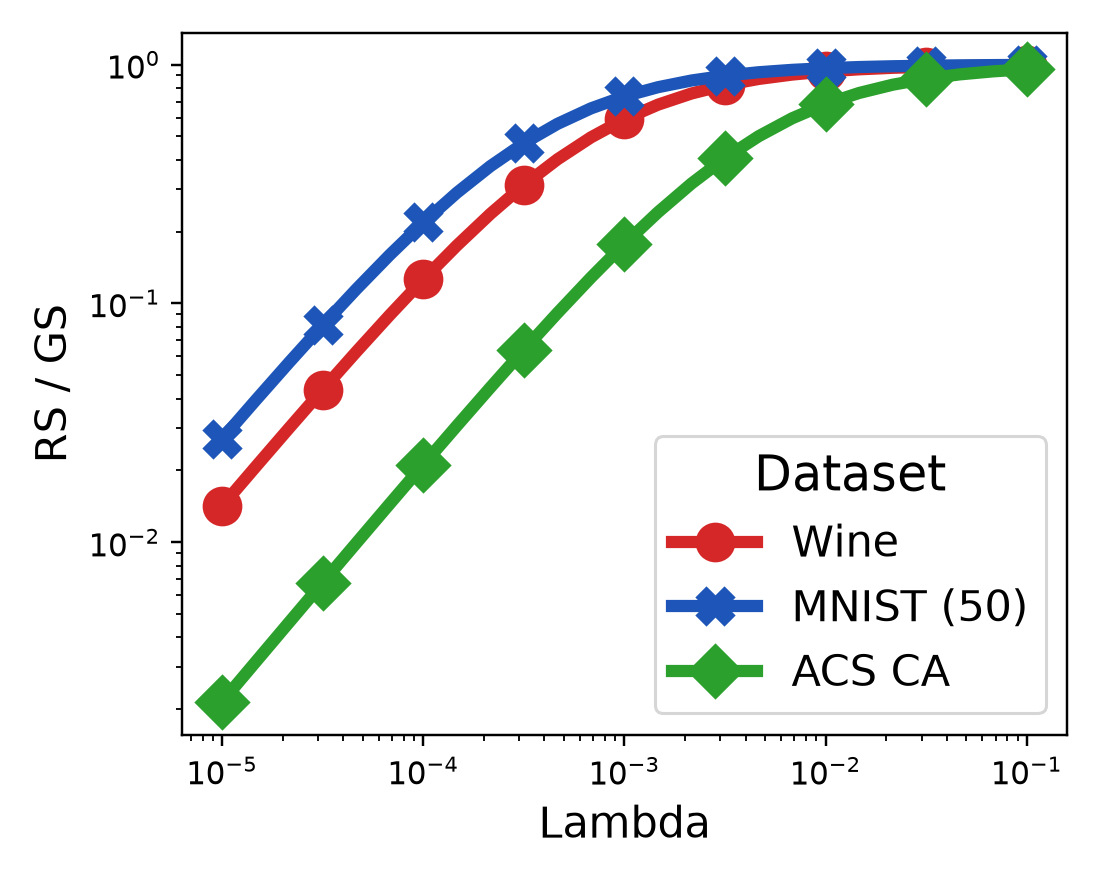}
      \caption{Passive Log Loss}
      \label{subfig:passive_erm_logloss}
  \end{subfigure}

    \caption{\textbf{Retain vs.\ global sensitivity (Passive):}~
            In all cases, the $y$-axis is the ratio $\nicefrac{\RS^{(m)}}{\GS^{(m)}}$ where smaller value means larger advantage of using $\RS^{(m)}$ compared to $\GS^{(m)}$; this gap is largest when the
	        retained data are well-conditioned, while the ratios approach $1$ in
	        regimes where worst-case and data-dependent bounds
	        coincide. For details on experiments see \Cref{app:experiments}.}
\label{fig:RS_vs_GS_passive}
	  \end{center}

\vspace{-0.5cm}
\end{figure*}

\subsection{Principal Component Analysis (PCA)}

Differentially private PCA is challenging because the global
sensitivity of the top-$k$ subspace can remain large even for a single-record change when the eigengap
is not bounded away from zero, making naive output perturbation vacuous. DP mechanisms
therefore rely on instance-dependent stability (e.g., via PTR~\citep{dwork2014gauss}) or more elaborate algorithms~\citep{dungler2025iterative} to add 
less noise when the eigengap is large. These methods can be computationally heavy and often substantially more complex than output perturbations.  In the certified unlearning setting, we
instead calibrate noise to the retain sensitivity of the rank-\(k\)
projector on the retained set \(R\), which similarly yields smaller
noise when the spectrum is well-separated.

Let \(R\in\mathbb{R}^{n\times d}\) have rows \(x_i\) with \(\|x_i\|\le B\), and let
\(\hat\Sigma_R=\frac1n R^\top R\) be the empirical second-moment. 
We consider releasing the rank-\(k\) PCA projector \(P_k(\hat\Sigma_R)\), i.e. the orthogonal projector onto the top-\(k\) eigenspace of \(\hat\Sigma_R\). Its retain sensitivity is $\RS_{P_k}^{(m)}(R)
    :=
    \max_{|X|\le m}
    \left\|P_k(\hat\Sigma_{R\cup X})-P_k(\hat\Sigma_R)\right\|_F$.

\begin{restatable}{lemma}{rspca}\label{lem:pca_rs}
    Let \(\mathrm{gap}_k(R) := \lambda_k (\hat\Sigma_R)-\lambda_{k+1}(\hat\Sigma_R)>0\). 
    Then 
    \begin{equation*}
        \RS_{P_k}^{(m)}(R)\leq \tfrac{4m\sqrt{2}B^2}{(n+m)\mathrm{gap}_k(R)}.
    \end{equation*}
\end{restatable}

\begingroup
\renewcommand{\proofname}{Proof sketch}
\begin{proof}
The result follows from Davis--Kahan applied to the perturbation
\(\hat\Sigma_{R\cup X}-\hat\Sigma_R\), whose operator norm is at most
\(2mB^2/(n+m)\). See \Cref{sec:app_pca} for the full proof.
\end{proof}
\endgroup
As a direct corollary, a passive mechanism releasing a noisy projector can be certified using retain sensitivity.
Concretely, by \Cref{thm:retain_unlearn,lem:pca_rs}, adding Gaussian noise with scale $\sigma(R)=c_{\epsilon,\delta}\cdot \nicefrac{4m\sqrt{2}B^2}{(n+m)\mathrm{gap}_k(R)}$ is sufficient for \(\br{\epsilon,\delta}\)-unlearning when the released statistic is the rank-\(k\) projector. For more details and a utility guarantee, see~\Cref{sec:app_pca}. 

\subsection{Support Vector Machine (SVM) - Hard Margin}

Kernelised SVMs are a classical example where retain sensitivity can be much smaller than worst-case sensitivity. The reason is that the hard-margin SVM solution is determined only by support vectors: adding a point affects \(w_R\) only if it becomes margin-relevant. Thus, when the retained set already contains points close to the true margin, additional samples cannot substantially change the margin, and the retain sensitivity can be far below the global sensitivity.

Let \(k\) be a kernel with RKHS \(\mathcal{H}\), and let \(w_R\in\mathcal{H}\) denote the hard-margin kernel SVM solution trained on \(R\). Let \(\gamma\) and \(\gamma_R\) be the true and empirical margins. We study
$\RS_{\mathrm{SVM}}^{(m)}(R)
    :=
    \sup_{Z\subseteq\supp{\mathcal{D}}: |Z|\le m}
    \|w_R-w_{R\cup Z}\|_{\mathcal{H}}$. Assume the domain is unbounded.

\begin{restatable}{lemma}{rssvmhardkernel}\label{lem:retain_svm_hard_kernel}
Let the true margin of \(\cD\) satisfy \(\gamma>0\). For any dataset
\(R\) with empirical margin \(\gamma_R\geq \gamma\), the ratio of
retain sensitivity given \(R\) over global sensitivity satisfies
\begin{equation*}
\tfrac{\RS_{\mathrm{SVM}}^{(m)}(R)}{\GS_{\SVM}^{(m)}}\leq \tfrac{\sqrt{\tfrac{1}{\gamma^2}-\tfrac{1}{\gamma_R^2}}}{1/\gamma}.
\end{equation*}
\end{restatable}

\begingroup
\renewcommand{\proofname}{Proof sketch}
\begin{proof}
Assuming the retain sensitivity bound $\RS_{\mathrm{SVM}}^{(m)}(R)\le \sqrt{1/\gamma^2-1/\gamma_R^2}$ for all $m\geq 1$, we obtain for the global sensitivity $\GS_{\mathrm{SVM}}^{(m)}=\sup_R \RS_{\mathrm{SVM}}^{(m)}(R)\le 1/\gamma$. Moreover, on an unbounded domain the empirical margin $\gamma_R$ can be arbitrarily large, hence $\GS_{\mathrm{SVM}}^{(m)}=1/\gamma$. For the proof of the retain-sensitivity bound, see \Cref{lem:svm} in the Appendix.
\end{proof}
\endgroup
We illustrate this gap empirically in \Cref{subfig:passive_svm} by training a hard-margin SVM on three common datasets (wine quality \cite{Cortez2009ModelingWP}, Californian census income data from 2018 \cite{ding2021retiring} and MNIST \cite{lecun2002gradient}). As the retain fraction grows,the empirical margin $\gamma_R$ approaches the true margin $\gamma$ and the ratio of RS/GS decreases. Intuitively, when the empirical margin $\gamma_R$ is already close to the true margin $\gamma$ , the effect of unlearning becomes negligible.

\subsection{Empirical Risk Minimiser (ERM)}\label{sec:erm}

We now study retain sensitivity for ERM, a central learning primitive that underlies a wide range of machine learning algorithms, and which we will also use in \Cref{sec:active_unlearn} for active unlearning. The key driver of the improvement over global sensitivity is \emph{strong convexity}, often enforced via an $\ell_2$ regulariser with parameter $\lambda$. In practice, the regulariser $\lambda$ is tuned to optimize test performance; it is therefore a \emph{fixed modeling choice} rather than a knob we can turn to make unlearning easier. However, global-sensitivity bounds for ERM scale as $1/\lambda$, so calibrating unlearning noise via global sensitivity becomes prohibitively large when $\lambda$ is small, which is precisely the regime that parameter tuning often selects. This makes unlearning methods that remain accurate for small $\lambda$ particularly important.

Define the empirical risk of a predictor \(w\in \cW\) on a dataset
\(R=\bc{z_1,\dots,z_n}\) by $\widehat F_R(w)=\nicefrac{1}{n}\sum_{i=1}^n f(w,z_i)$,
where \(f(w,z)\) is a differentiable loss function. Let  \(w_R=\argmin_{w\in\cW}\widehat{F}_R(w)\) denote the ERM
solution and assume \(f\) is
\(L\)-Lipschitz. \\
We use the observation that the empirical risk $\hat{F}_R$ is an objective defined by the dataset $R$ over parameters $w\in\mathcal{W}$. Using retain sensitivity, this allows us to use a \emph{data-dependent} strong convexity parameter $\lambda_R$ (depending on $R$ but uniform over $\mathcal{W}$), instead of a \emph{global} strong convexity parameter $\lambda$ (uniform over both $R$ and $\mathcal{W}$), which yields tighter, instance-specific sensitivity bounds.

\begin{lemma}\label{lem:erm_general_ratio}
If $\hat{F}_R$ is $\lambda_R$-strongly convex (see \Cref{def:strong_convex}) with $\lambda_R\ge \lambda>0$, then 
	    \begin{equation*}
	        \tfrac{\RS_{\ERM}^{(m)}(R)}{\Delta_{\GS}^{(m)}}\leq\tfrac{mL/n\lambda_{R}}{mL/n\lambda}=\tfrac{\lambda}{\lambda_R},
	    \end{equation*}
	    where $\GS_{\ERM}^{(m)}\leq \Delta_{\GS}^{(m)}=\tfrac{mL}{n\lambda}$.
\end{lemma}

\begingroup
\renewcommand{\proofname}{Proof sketch}
\begin{proof}
By standard stability of $\lambda_R$-strongly convex ERM, we have $\norm{w_R-w_{R\cup Z}}\leq \nicefrac{mL}{n\lambda_R}$ for any $\lvert Z\rvert\leq m$, hence $\RS_{\ERM}^{(m)}(R)\le mL/(n\lambda_R)$.
By definition, $\GS_{\ERM}^{(m)}=\sup_{R}\RS_{\ERM}^{(m)}(R)$, so using the worst-case 
$\inf_{R}\lambda_R\ge \lambda$ yields $\GS_{\ERM}^{(m)}\le mL/n\lambda$ (detailed proof in \Cref{lem:erm_general}).\looseness=-1
\end{proof}
\endgroup

Concrete instantiations for MSE and logistic loss can be found in \Cref{app:mse_logloss} in the Appendix. 
Under bounded data and parameters, and with an \(\ell_2\)-regulariser \(\lambda\ge0\), both losses have retain-set curvature of the form
    $\lambda_R
    =
    \nicefrac{1}{n}\lambda_{\min}(X_R^\top X_R)\cdot C+\lambda$,
where \(C=1\) for MSE and \(C=C_{R_w}>0\) for logistic regression. 
This makes the key comparison explicit: retain sensitivity uses the empirical curvature of the retained set $\lambda_R$ through the smallest eigenvalue of the empirical covariance, whereas the global-sensitivity baseline only uses the worst-case curvature from the regulariser \(\lambda\). Since the sensitivity bound scales as \(1/\lambda_R\), retain sensitivity is never worse than the global bound and improves over it whenever the empirical covariance has minimum eigenvalue bounded away from zero.
More importantly, even in unregularised regimes with \(\lambda=0\), the retain bound can still be finite whenever $\nicefrac{1}{n} \lambda_{\min}(X_R^\top X_R)>0$,
while the corresponding global-sensitivity bound becomes infinite.\looseness=-1

Empirically, \Cref{subfig:passive_erm_mse,subfig:passive_erm_logloss} show orders-of-magnitude gaps for small \(\lambda\), especially when \(\lambda<1\). This is consistent with \(\lambda_R\) being dominated by the empirical curvature term \(\nicefrac{1}{n}\lambda_{\min}(X_R^\top X_R)\cdot C\). For high-dimensional MNIST, we use a fixed Gaussian JL projection as part of the training pipeline \citep{kenthapadi2012privacy,paul2013randomsvm,wojcik2019training}; when the projected retain-set covariance has positive minimum eigenvalue, we obtain a meaningful \((\epsilon,\delta)\)-unlearning guarantee.\looseness=-1

%% file: content/active_unlearn.tex
Active unlearning algorithms access the retain set at deletion time and apply a deterministic update toward the retrained model. For gradient- and Newton-type methods, the effectiveness of this update depends on the curvature and conditioning of the retained objective \(\widehat F_R\): better conditioning yields faster contraction and requires less noise to certify. 
Retain sensitivity is therefore a natural lens, allowing the certificate for fixed \((\epsilon,\delta)\) to use retain-set geometry rather than a worst-case bound. 
We illustrate this for two ERM unlearning mechanisms: Descent-to-Delete \citep{neel2021descent} and the Newton-step update \citep{sekhari2021RememberWhatYou}.\looseness=-1

\subsection{Descent-to-Delete}

We next revisit Descent-to-Delete (D2D; \citet{neel2021descent}) from a retain-sensitivity perspective. Given \(R'=R\cup U\), D2D starts from the trained ERM \(w_{R'}\), applies \(I\) projected gradient steps on the retained objective \(\widehat F_R\), and then adds Gaussian noise. Writing
    $G_R(w):=\mathrm{Proj}_{\cW}\bigl(w-\eta_R\nabla \widehat F_R(w)\bigr)$,
the deterministic approximation step is $\bar{\cA}_0(U,w_{R'},R')=G_R^I(w_{R'})$, and $\bar{\cA}_0(\emptyset,w_R,R)=w_R $.\\
The retain-based analysis uses curvature of the fixed objective $\widehat F_R$. Let $\lambda_R,\beta_R$ and $\gamma_R$ denote its strong-convexity, smoothness, and projected-gradient contraction parameters. Typically, $\lambda_R\ge\lambda$, $\beta_R\le\beta$, and thus $\gamma_R\le\gamma$, improving the analysis. Formal definitions are deferred to \Cref{app:d2d}.\looseness=-1
\begin{lemma}\label{lem:d2d_step_ratio}
Fix a target \((\epsilon,\delta)\) and a noise level \(\sigma\). Let
\(I_R(\epsilon,\delta,\sigma)\) be the minimum number of D2D gradient steps required under the retain-sensitivity analysis on \(R\), and let
\(I(\epsilon,\delta,\sigma)\) be the corresponding number under the global-sensitivity analysis.
Then
    \begin{equation*}
        \nicefrac{I_R\br{\epsilon,\delta,\sigma}}{I\br{\epsilon,\delta,\sigma}}
        =\tfrac{\ln\br{\nicefrac{C_n}{\lambda_R}}\ln\br{\gamma}}{\ln\br{\nicefrac{C_n}{\lambda}}\ln\br{\gamma_R}}
    \end{equation*}
    for \(C_n=mL/n\sigma\, b\br{\epsilon,\delta}\), with
    \(b\br{\epsilon,\delta}=\sqrt{2\log\br{\nicefrac{1}{\delta}}+2\epsilon}-\sqrt{2\log\br{\nicefrac{1}{\delta}}}\).
\end{lemma}

\begingroup
\renewcommand{\proofname}{Proof sketch}
\begin{proof}
    The proof follows the same argument as given in \citet{neel2021descent} but keeps \(\sigma\) fixed
    and solves for the minimum \(I\) that yields a given
    \(\br{\epsilon,\delta}\). We combine: (i) contraction of projected
    gradient descent under \(\br{\lambda_R,\beta_R}\), (ii) the
    retain-sensitivity ERM stability bound from
    \Cref{lem:erm_general_ratio}, and (iii) the Gaussian mean-shift
    characterization which yields the term \(b\br{\epsilon,\delta}\)
    from~\citet{bun2016concentrated}.  Further details are in
    \Cref{lem:d2d_retain}.
\end{proof}
\endgroup

\Cref{lem:d2d_step_ratio} identifies two gains from calibrating D2D to the retain set \(R\). First, larger \(\lambda_R\) reduces the initial ERM perturbation caused by the deletion. Second, improved conditioning of \(\widehat F_R\) yields a smaller contraction factor \(\gamma_R\), so fewer gradient steps are needed before adding noise.\\
For the utility evaluation in \Cref{fig:active_d2d_accuracy} we show less noise implies better accuracy: we fix \((\epsilon,\delta)\) and vary the number of D2D iterations to calibrate the noise level \(\sigma\) required to certify unlearning. Across the additional MSE and logistic-regression experiments in \Cref{fig:d2d-utility-logloss-mse} in the Appendix, for \(\lambda\in\{10^{-1},\dots,10^{-5}\}\), we observe that retain-sensitivity calibration can recover near-retraining performance, especially for small \(\lambda\). In contrast, global-sensitivity calibration often requires so much noise that the classifier performs close to random chance.

\begin{figure*}[t]
\centering

\begin{subfigure}[t]{0.48\textwidth}

    \begin{minipage}[t]{0.48\linewidth}
        \centering
        \includegraphics[width=\linewidth]{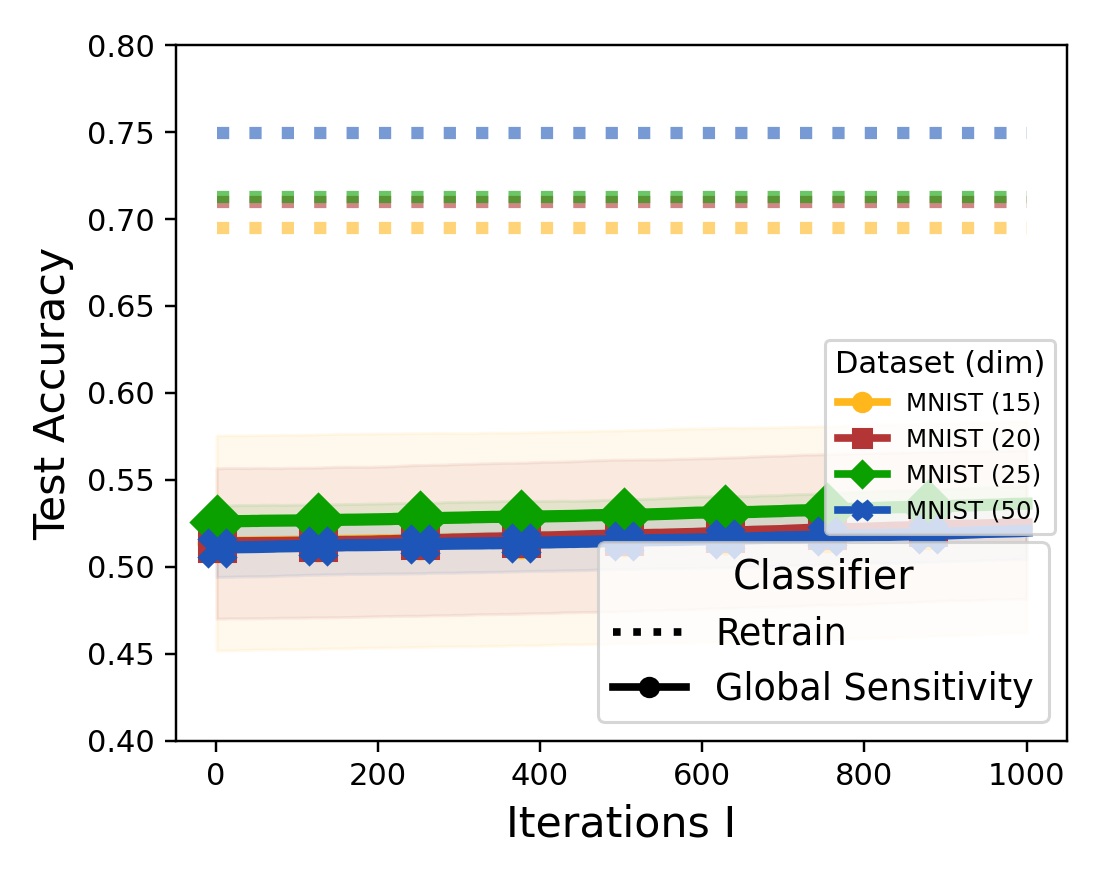}
    \end{minipage}
    \hfill
    \begin{minipage}[t]{0.48\linewidth}
        \centering
        \includegraphics[width=\linewidth]{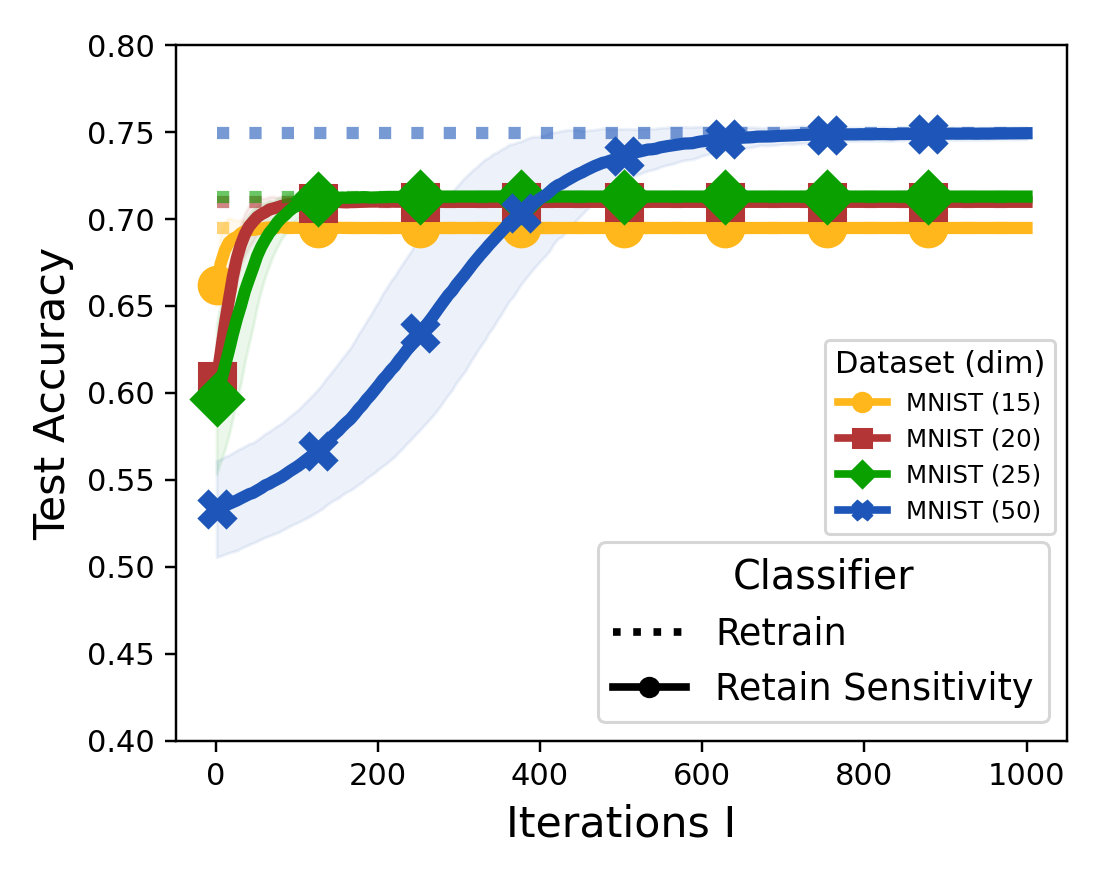}
    \end{minipage}

    \caption{D2D acc. $\GS$ (left) vs $\RS$ (right),  $\lambda=0.0001$\looseness=-1}
    \label{fig:active_d2d_accuracy}
\end{subfigure}
\hfill
\begin{subfigure}[t]{0.48\textwidth}
    \centering

    \begin{minipage}[t]{0.48\linewidth}
        \centering
        \includegraphics[width=\linewidth]{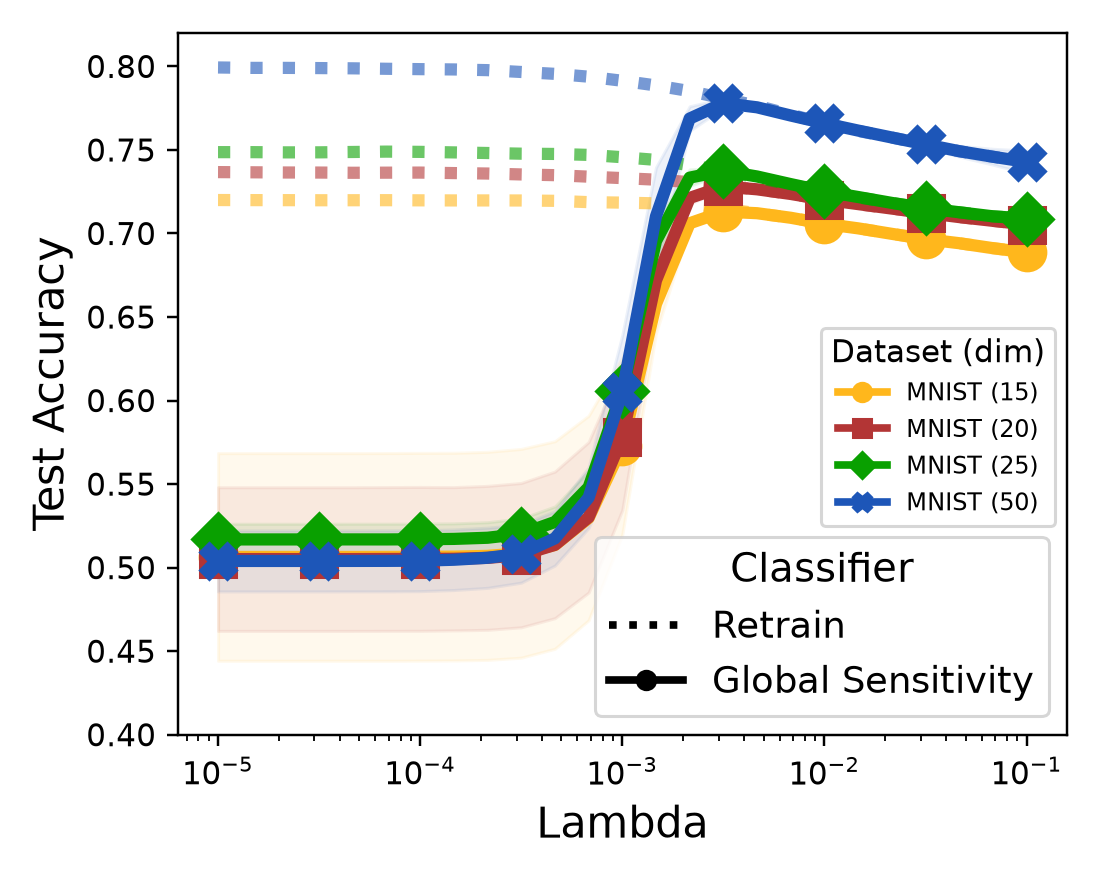}
    \end{minipage}
    \hfill
    \begin{minipage}[t]{0.48\linewidth}
        \centering
        \includegraphics[width=\linewidth]{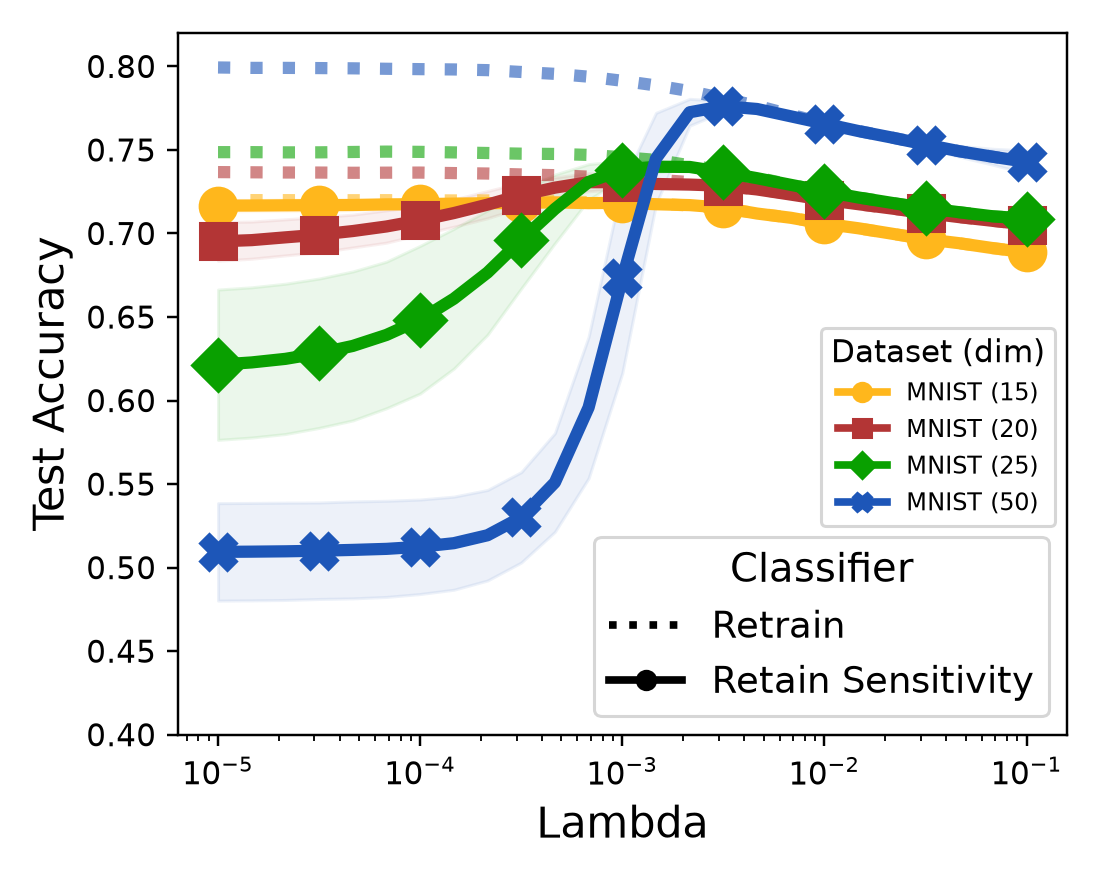}
    \end{minipage}

    \caption{Newton acc. $\GS$ (left) vs $\RS$ (right)}
    \label{fig:active_newton_accuracy}
\end{subfigure}

    \caption{
        \textbf{Retain vs. global sensitivity (Active):} 
Test accuracy of active unlearning for a fixed \((\epsilon,\delta)\) guarantee, comparing noise calibrated to $\GS^{(1)}$ vs. $\RS^{(1)}$ against retraining. We use LogLoss on MNIST with random projections to dimensions \(\{50,25,20,15\}\), which control the data-dependent curvature \(\lambda_R\). For D2D (\Cref{fig:active_d2d_accuracy}), we vary the number of gradient steps \(I\) and calibrate \(\sigma\) to certify unlearning after \(I\) steps. For Newton (\Cref{fig:active_newton_accuracy}), we vary the \(\ell_2\)-regulariser \(\lambda\) after one Newton update. Details and additional D2D utility evaluations can be found in \Cref{app:experiments,fig:d2d-utility-logloss-mse,fig:active_noise}.\looseness=-1
        }
    \label{fig:RS_vs_GS_active}
\vspace{-0.3cm}
\end{figure*}

\subsection{Newton Step Update}
We next consider the Newton-step update of \citet{sekhari2021RememberWhatYou}, which applies one Newton correction toward the retrained solution before adding noise. For \(R'=R\cup U\), the deterministic update is \(\bar{\cA}_0\br{\emptyset,w_R,R}=w_R\) and $\bar{\cA}_0\br{U,w_{R'},R'} =
    w_{R'}+\nicefrac{1}{n}\widehat{H}^{-1}\sum_{z\in U}\nabla f\br{w_{R'}, z},$
where \(\widehat H=\nabla^2\widehat F_R(w_{R'})\). As in \Cref{sec:erm}, we use the retain-set strong convexity \(\lambda_R\ge\lambda\). This improves both the ERM stability term \(\|w_{R'}-w_R\|\) and the inverse-Hessian bound, reducing the required noise scale.\looseness=-1

\begin{lemma}\label{lem:newton-rs}
Assume \(\widehat F\) is \(\lambda\)-strongly convex, \(f\) is \(L\)-Lipschitz, and the Hessian is \(M\)-Lipschitz. If \(\widehat F_R\) is \(\lambda_R\)-strongly convex with \(\lambda_R\ge\lambda\), then the ratio of retain sensitivity to upper bound on the global
    sensitivity \(\Delta_{\GS}^{(m)}\) of the Newton-step update
    approximation is given as
    \begin{equation*}
        \tfrac{\RS_{(\cA,\bar{\cA})}^{(m)}\br{R}}{\Delta_{\GS}^{(m)}}\leq \br{\nicefrac{\lambda}{\lambda_R}}^3
    \end{equation*}
    \vspace{-0.2cm}
    Moreover, calibrating noise to the upper bound
    \(\RS_{\br{\cA,\bar{\cA}}}^{(m)}\br{R}\leq \tfrac{m^2L^2M}{2n^2\lambda_R^3}\)
    ensures \((\varepsilon,\delta)\)-unlearning. \looseness=-1
\end{lemma}

\begingroup
\renewcommand{\proofname}{Proof sketch}
\begin{proof}
    The argument follows \citet{sekhari2021RememberWhatYou}: we view
    \(\bar w\) as a one-step Newton correction of \(w_{R'}\) toward
    \(w_R\), and bound the remaining error by controlling the Taylor
    remainder. We improve the upper bound, using the retain
    sensitivity bound \(\norm{w_{R'}-w_R}\leq \nicefrac{mL}{n\lambda_R}\)
    from \Cref{lem:erm_general_ratio}. Additionally, we
     use \(\norm{(\nabla^2\hat{F}_R(w_{R'}))^{-1}}\leq
    \nicefrac{1}{\lambda_R}\) instead of \(\nicefrac{1}{\lambda}\). Details are
    deferred to the proof of \Cref{lem:rwywtf}.\looseness=-1
\end{proof}
\endgroup

Thus, the Newton-step update amplifies curvature gains: replacing the worst-case curvature \(\lambda\) by the retain-set curvature \(\lambda_R\) improves the required noise scale by the cubic factor \((\lambda/\lambda_R)^3\).
This amplifies the advantage observed in the ERM experiments
\Cref{subfig:passive_erm_logloss}, and and less noise translates into clear utility gains in 
\Cref{fig:active_newton_accuracy}. As in the D2D setting, projecting to lower dimensions improves the empirical curvature \(\lambda_R\) through \(\lambda_{\min}\), reducing retain sensitivity and allowing the unlearned model to approach exact retraining accuracy even for small \(\lambda\). By contrast, GS-calibrated noise yields nearly random predictions in this regime. These experiments show why accuracy at small \(\lambda\) matters: GS-based unlearning can preserve utility for large regularisation, but breaks down precisely in the low-\(\lambda\) regime that often gives the best retraining accuracy and is therefore selected during model tuning.\looseness=-1

%% file: content/discussion.tex
An interesting consequence of analyzing unlearning via retain sensitivity is that it sharpens the inherent unlearning guarantees of any $(\varepsilon_{\mathrm{DP}},\delta)$-DP model. If  $\RS^{(m)}(R)/\GS^{(m)}\le C<1$, then Gaussian noise calibrated to $\GS^{(1)}$ already implies $\varepsilon_{\mathrm{Unlearn}}\le mC\,\varepsilon_{\mathrm{DP}}$ without retraining, using group privacy $\GS^{(m)}\leq m\GS^{(1)}$. This motivates retain-sensitivity analyses of already DP models to avoid model replacement or extra noise addition for unlearning, reducing exposure to differencing attacks \cite{bertran_reconstruction_2024}.\looseness=-1

This perspective clarifies an important subtlety of certified unlearning as stated in \Cref{def:unlearning_indist}. Because the noise required to satisfy the guarantee may scale with the retained data \(R\), certified unlearning need not protect the retained records themselves. This complements the mechanism-level concerns raised by \cite{cohen2026protectingundeletedmachineunlearning} from a noise perspective, and highlights the importance of specifying whether unlearning is intended merely to match retraining, or also to protect deleted and/or retained records from differencing attacks.

In this work, we assume \emph{full side information}, $T(R)=R$, which suffices for unlearning via retain sensitivity: given the side information and unlearning request, both the learner and unlearner can reconstruct $R$. All our problem-specific retain-sensitivity bounds depend only on statistics of $R$, suggesting that full access to $R$ might not be necessary. Determining how much side information is needed in general is closely related to the space complexity of unlearning studied by \citep{cherapanamjeri2025space}, and raises a broader question: \emph{what is the minimal side information required to efficiently compute retain sensitivity?}\looseness=-1

To conclude, this work introduces the notion of retain sensitivity and establishes it as a sufficient (and in some cases necessary) quantity for calibrating noise in both passive and active unlearning. We illustrate its benefits both theoretically and empirically on existing algorithms and problems. 
While this deepens our conceptual understanding of unlearning, a key next step is to translate these gains into practice by deriving efficient retain-sensitivity estimates and developing new certified unlearning mechanisms for modern large-scale models.\looseness=-1

%% file: appendix/appendix.tex
\section{Appendix: Related Work}\label{sec:related_work}

\noindent\textbf{Machine unlearning}, introduced by \cite{cao2015towards}, broadly studies how to remove the influence of a subset of training points from an already trained model. A large empirical literature proposes practical but generally \emph{non-certified} procedures, including post-hoc parameter editing heuristics \cite{golatkar2020eternal}, suppressions \cite{foster2024fast}, and approximate approaches \cite{Chundawat2023zeroshot}. In contrast, \emph{certified machine unlearning} aims to provide an explicit deletion guarantee after the unlearning procedure, most commonly a \emph{statistical} guarantee phrased as indistinguishability from retraining \cite{guo2019certified,sekhari2021RememberWhatYou}.
In this work, we focus on certified machine unlearning with a statistical guarantee.\\
Within certified unlearning, one can distinguish \emph{exact} and \emph{approximate} guarantees. Exact schemes aim to reproduce the retrained model, e.g.\ via sharding/slicing/checkpointing strategies in training such as SISA \cite{bourtoule2021machine, chowdhury2024towards}, which often come at a high computational and overhead burden. Approximate schemes instead allow deviation from the retrained model, with the goal of less memory intensive and computationally efficient algorithms while certifying \emph{statistical indistinguishability} between the unlearned and retrained outputs \cite{sekhari2021RememberWhatYou}; this is the setting we study.

We also distinguish between \emph{passive/lazy} and \emph{active} unlearning algorithms:\\
\textbf{Passive (or lazy) unlearning algorithms} are explored in \cite{hu2025online,huang2023tight} and inject calibrated noise to the model trained on the full dataset to hide the effect of the unlearnt set. They closely relate to DP, as has been demonstrated in \cite{huang2023tight}: when there is no side information ($T(S)=\emptyset$) then worst-case unlearning cannot fundamentally improve over what differential privacy (DP) already provides \cite{huang2023tight}. Our contribution highlights the opposite regime: when the unlearning algorithm has full ($T(S)=S$) access to the retained data, the required unlearning noise can scale with a \emph{data-dependent} quantity (retain sensitivity) that can be substantially smaller than global sensitivity, which is the quantity underlying DP guarantees.\\
\textbf{Active unlearning algorithms} often estimate the influence the forget set has on the model through first- \cite{neel2021descent, allouah2024utility} and second-order methods \cite{koh2020UnderstandingBlackboxPredictions, ginart2019making, sekhari2021RememberWhatYou} and subtract them from the classifier followed by a noise addition step to give the unlearning certificate. However, many theoretical methods rely on strong regularity assumptions (e.g., strong convexity, smoothness/Lipschitzness, and sometimes Hessian regularity) \cite{sekhari2021RememberWhatYou}. A growing line of work seeks to relax these conditions and extend certification to more complex, especially non-convex, settings \cite{chien2024langevin,mu2024rewind, zhang2024towards}. In our work, we partially relax the need for uniform strong convexity by showing that strong convexity can arise \emph{data-dependently} when the empirical Hessian is well-conditioned (i.e., its smallest eigenvalue is bounded away from zero).

Finally, \textbf{Differential Privacy } \cite{dwork2006calibrating} is closely connected to certified unlearning through the shared indistinguishability viewpoint: DP implies a deletion guarantee for any single point. Several works make this connection explicit, and separations have been explored under accuracy and adaptivity constraints \cite{sekhari2021RememberWhatYou,gupta2021adaptive, courasia2023forget}. We show that with full sample access, there is an inherent conceptual separation in noise necessary for privacy compared to unlearning, characterized by retain sensitivity.

\section{Appendix: Supplement Definitions}\label{app:supp}

\begin{definition}[$L$-Lipschitz loss]
A loss $f: \mathcal{W}\times\mathcal{Z}\to\mathbb{R}$ is \emph{$L$-Lipschitz in $w$} (uniformly over $z$) if for all $w,w'\in \mathcal{W}$ and all $z\in\mathcal{Z}$,
\[
|f(w,z)-f(w',z)|\le L\|w-w'\|_2.
\]
Equivalently, if $f(\cdot,z)$ is differentiable for every $z$, then $\|\nabla_w f(w,z)\|_2\le L$ for all $w\in \mathcal{W}$ and $z\in\mathcal{Z}$.
\end{definition}

\begin{definition}[$\lambda$-strongly convex loss]
A differentiable loss $f:\mathcal{W}\times\mathcal{Z}\to\mathbb{R}$ is \emph{$\lambda$-strongly convex in $w$} (uniformly over $z$) if for all $w,w'\in \mathcal{W}$ and all $z\in\mathcal{Z}$,
\[
f(w',z)\ge f(w,z)+\langle \nabla_w f(w,z),\, w'-w\rangle + \frac{\lambda}{2}\|w'-w\|_2^2.
\]
Equivalently, if $f(\cdot,z)$ is twice differentiable for every $z$, then $\nabla_w^2 f(w,z)\succeq \lambda I$ for all $w\in \mathcal{W}$ and $z\in\mathcal{Z}$ (i.e., $\lambda_{\min}(\nabla_w^2 f(w,z))\ge \lambda$).
\end{definition}

\begin{definition}[Data-Dependent Strong Convexity]\label{def:strong_convex}
	     For a \emph{fixed} dataset \(R\), the empirical risk \(\widehat{F}_R\) is \(\lambda_R\)-strongly convex if there exists \(\lambda_R>0\) s.t. for all \(w,w'\in \cW\),
	    \begin{equation*}
	        \widehat{F}_R(w)-\widehat{F}_R(w')\geq \ip{\nabla \widehat{F}_R(w')}{w-w'} +\tfrac{\lambda_R}{2}\norm{w-w'}^2 
	    \end{equation*}
    	    Equivalently, if \(\widehat{F}_R(w)\) is twice differentiable, this holds with \(\lambda_R=\inf_{w\in \cW}\lambda_{\min}\br{\nabla_w^2\widehat{F}_R(w)}\) and \(\widehat{F}_R\) is \(\lambda_R\)-strongly convex iff \(\lambda_R>0\).
\end{definition}

\begin{definition}[$\beta$-smooth loss]
A differentiable loss $f: \mathcal{W}\times\mathcal{Z}\to\mathbb{R}$ is \emph{$\beta$-smooth in $w$} (uniformly over $z$) if for all $w,w'\in \mathcal{W}$ and all $z\in\mathcal{Z}$,
\[
\|\nabla_w f(w,z)-\nabla_w f(w',z)\|_2 \le \beta \|w-w'\|_2.
\]
Equivalently, if $f(\cdot,z)$ is twice differentiable for every $z$, then $\|\nabla_w^2 f(w,z)\|_2 \le \beta$ for all $w\in \mathcal{W}$ and $z\in\mathcal{Z}$.
\end{definition}

\begin{definition}[$\beta_z$-smooth loss]
Let $z\in \mathcal{Z}$. A differentiable loss $f: \mathcal{W}\times\mathcal{Z}\to\mathbb{R}$ is \emph{$\beta_z$-smooth in $w$} if for all $w,w'\in \mathcal{W}$,
\[
\|\nabla_w f(w,z)-\nabla_w f(w',z)\|_2 \le \beta \|w-w'\|_2.
\]
Equivalently, if $f(\cdot,z)$ is twice differentiable, then $\|\nabla_w^2 f(w,z)\|_2 \le \beta$ for all $w\in \mathcal{W}$.
\end{definition}

\begin{definition}[Hessian $M$-Lipschitz loss]
A twice differentiable loss $f: \mathcal{W}\times\mathcal{Z}\to\mathbb{R}$ has \emph{$M$-Lipschitz Hessian in $w$} (uniformly over $z$) if for all $w,w'\in \mathcal{W}$ and all $z\in\mathcal{Z}$,
\[
\|\nabla_w^2 f(w,z)-\nabla_w^2 f(w',z)\|_2 \le M \|w-w'\|_2.
\]
\end{definition}

\begin{fact}[Contraction for strongly convex and smooth GD \cite{nesterov2004introductory} Thm 2.1.14]\label{fact:gd_contraction}
Let $F:W\to\mathbb{R}$ be $m$-strongly convex and $M$-smooth (with respect to $\|\cdot\|_2$), and let $w^\star=\arg\min_{w\in \mathcal{W}}F(w)$.
Consider projected gradient descent with constant step size $\eta=\frac{2}{M+m}$: $w_{t+1}=\mathrm{Proj}_\mathcal{W}(w_t-\eta \nabla F(w_t))$. Then for all $t\ge 0$, with $\gamma=\frac{M-m}{M+m}=\frac{\kappa-1}{\kappa+1}$ and  $\kappa=\frac{M}{m}$:
$$
\|w_{t+1}-w^\star\|_2 \le \gamma \, \|w_t-w^\star\|_2,
\qquad
$$
\end{fact}

\section{Appendix: Median}

\rsmedian*

\begin{proof}
The global sensitivity of the median on samples in \(\bs{0,B}\) satisfies \(\GS_{\median}^{(1)}=\nicefrac{B}{2}\) and \(\GS_{\median}^{(m)}=B\) for every \(m\ge 2\). For \(m=1\), consider \(R=\bc{0,\ldots,0,B,\ldots,B}\) with \(\median(R)=0\); after adding one additional point \(B\), the median becomes \(\nicefrac{B}{2}\). For \(m\ge 2\), an analogous construction shows that adding two or more points equal to \(B\) can shift the median from \(0\) to \(B\).\\

For retain sensitivity of the median, let $Z=\bc{x_{n+1},\dots, x_{n+m}}$ be a multiset in $[0,B]$.
The median of $R\cup Z$ can move upward by at most $\lceil m/2\rceil$ order statistics of $R$ or downward by at most $\lceil m/2\rceil$ order statistics of $R$. This establishes  $\RS_{\median}^{(m)}(R)\leq \max\{x_{(t+\lceil m/2\rceil)}-x_{(t)}, x_{(t)}-x_{(t-\lceil m/2\rceil)}\}$. \\
For the lower bound, consider $Z_0=\bc{0,\dots,0}$, i.e. the multiset of $m$ zeros. This gives $\median(R\cup Z_0)=1/2(x_{(t-\lceil m/2\rceil)}+x_{(t-\lfloor m/2\rfloor)})$ and thus a shift of at least $1/2(x_{(t)}-x_{(t-\lceil m/2\rceil)})$. Similarly  $Z_B=\bc{B,\dots,B}$, gives $\median(R\cup Z_B)=1/2(x_{(t+\lceil m/2\rceil)}+x_{(t+\lfloor m/2\rfloor)})$ and thus a shift of at least $1/2(x_{(t+\lceil m/2\rceil)}-x_{(t)})$ (for $m$ even this bound is exact). \\

Putting the upper and lower bounds together gives $\RS_{\median}^{(m)}(R)\asymp \max\{x_{(t+\lceil m/2\rceil)}-x_{(t)}, x_{(t)}-x_{(t-\lceil m/2\rceil)}\}$ and dividing by $\GS_{\median}^{(m)}\asymp B$ proves the claim.

\end{proof}

\section{Appendix: MST}

A spanning tree of an undirected graph \(G=\br{V,E}\) is a connected,
acyclic subgraph \(T\) such that \(T=\br{V,E_T}\) with \(E_T\subseteq
E\). We consider weighted graphs \(G=\br{V,E,w}\) where each edge
\(e\in E\) has a weight \(w(e)\in[0,B]\). The minimum spanning tree
problem (MST) is to find a spanning tree \(T\) of minimal total weight
\(w(T)=\sum_{e\in E_T}w(e)\). Denote an MST of a graph \(G\) by
\(\mathrm{MST}(G)\) and its weight by \(f(G):=w\br{\mathrm{MST}(G)}\).

\subsection{Edge adjacency}

\begin{lemma}\label{lem:gs-mst}
    Releasing the weight of a MST has $\GS_f^{(m)}=mB$.
\end{lemma}

\begin{proof}
Adding an edge can only introduce a cheaper alternative, so the MST weight is monotone non-increasing. Each added edge can reduce the MST weight by at most \(B\), since all edge weights lie in \([0,B]\). Hence \(\GS_f^{(m)}\le mB\).
For tightness, take a graph where all available edges have weight $B$, and let the $m$ missing edges form a forest. Then, every MST has weight \(\br{n-1}B\). After adding the $m$ missing edges with weight \(0\), the MST uses all of them and thus has weight \((n-1-m)B\). Thus the change is $mB$ and $\GS_f^{(m)}= mB$.
\end{proof}

\begin{lemma}\label{lem:MST-exchange}
    Let $I_1\subset I_2\subseteq E(T)$ and $g\in E(T)\setminus I_2$, where $T=\MST(R)$.
    Assume $I_2\in I_m(T)$ and $I_1+\{g\}\in I_m(T)$. Then there exists $h\in I_2\setminus I_1$ such that $I_2-\{h\}+\{g\}\in I_m(T)$.
\end{lemma}

\begin{proof}
    Let $H=T+A(R)$, then for sets of size at most $m$, $X\in I_m(T)$ iff $H-X$ is connected. Therefore by assumption $H-I_2$ is connected. \\
    If $H-(I_2+\bc{g})$ is still connected, then any $h\in I_2\setminus I_1$ works to prove the claim.\\
    Else $g$ is a bridge of $H-I_2$ and its deletion splits $H-I_2$ into 2 connected components $C_1,C_2$. But as $H-(I_1+\bc{g})$ is connected, there must be an edge in $H-(I_1+ \bc{g})$ crossing the cut $(C_1,C_2)$, while there is no such edge in $H-(I_2+\bc{g})$. Thus that edge must be in $I_2\setminus I_1$; call it $h$. Then $H-(I_2-\bc{h}+ \bc{g})$ is connected and thus $I_2-\bc{h}+\bc{g}\in I_m(T)$.
\end{proof} 

\begin{lemma}\label{lem:MST-rs}
For a graph \(R=\br{V_R,E_R}\) fix $T\in \MST(R)$. Then
    \begin{equation*}
        \RS_f^{(m)}(R)=\max_{X\in I_m(T)}\sum_{h\in X}w(h).
    \end{equation*}
Moreover, an optimal set \(X_{\mathrm{opt}}\in I_m(T)\) can be found greedily: consider the edges of \(T\) in non-increasing order of weight and add
an edge \(h\) whenever $X+\{h\}\in I_m(T)$.
The resulting set maximises \(\sum_{h\in X}w(h)\) over \(I_m(T)\).
\end{lemma}

\begin{proof}
    We first prove the upper bound $\RS_f^{(m)}\leq \max_{X\in I_m(T)}\sum_{h\in X}w(h)$. Let $F\subseteq A(R)$ and $\lvert F\rvert\leq m$. First, note that the addition of a new edge $e$ to the graph $R$ can change the MST by at most one edge replacement. To maximise the change, we set the weight of edges in $F$ to zero. By the cycle property, after adding an edge $e\in F$, the heaviest edge in the cycle $T+ \bc{e}$ is replaced; call $X\in E(T)$ the tree edges replaced by the addition of the set $F$. Then $\lvert X\rvert \leq \lvert F\rvert\leq m$ and the final tree is contained in $T- X+ F$ and is thus connected. As $F\subseteq A(R)$, it follows that $T- X+ A(R)$ is connected. Thus $X\in I_m(T)$ and $f(R)-f(R+ F)=\sum _{h\in X}w(h)$. Taking the maximum over the right hand side proves the upper bound. \\
    
For the lower bound, consider any $X\in I_m(T)$. Since $T$ is a tree, $T-X$ has exactly $\lvert X\rvert+1$ connected components. Furthermore, as $T-X+A(R)$ is connected, there exist exactly $\lvert X\rvert$ edges $F_X\subseteq A(R)$ connecting these components such that $T-X+F_X$ is a spanning tree.
Adding the edges $F_X$ with weight $0$ to $R$ implies that $T-X+F_X$ is a spanning tree of weight $w(T)-\sum_{h\in X}w(h)$. Thus
    $f(R+F_X)
    \leq
    f(R)-\sum_{h\in X}w(h)$,
which implies $\RS_f^{(m)}(R)\geq \sum_{h\in X}w(h)$.
Taking the maximum over all $X\in I_m(T)$ proves the lower bound.\\
    
    Lastly, we show that the set $X_{\mathrm{greedy}}$ found through the greedy construction described below, achieves the optimum over $I_m(T)$: $\sum_{h\in X_{\mathrm{greedy}}}w(h)=\max_{X\in I_m(T)}\sum_{h\in X}w(h)$. \\
    The greedy algorithm proceeds as follows: Set $X_0=\emptyset$ and then for $i=1\dots, m$ find \\$h_i\in \argmax_{h\in E(T)\setminus X_{i-1}:X_{i-1}+ h\in I_m(T)} w(h)$ and set $X_i=X_{i-1}+ h_i$. Stop when either \(|X_i|=m\) or no
feasible edge remains. We refer to this set as $X_{\mathrm{greedy}}$. \\
    Let $O$ be an optimal set and choose $O$ among optimal sets to have maximum cardinality (at most of size $m$). 
    We show by induction that for the greedy set after $i$ steps $X_i=\bc{h_1,\dots, h_i}$, there exists an optimal solution $O_i$ s.t. $X_i\subseteq O_i$. 
    The case for $i=0$ is trivial.
    Now suppose $X_{i-1}\subseteq O_{i-1}$. If $h_i\in O_{i-1}$ then set $O_i=O_{i-1}$. Else apply the result from \Cref{lem:MST-exchange} with $I_1=X_{i-1}$, $I_2=O_{i-1}$, $g=h_i$ which shows there exists a $e\in O_{i-1}- X_{i-1}$ such that $O_i=O_{i-1}- \bc{e}+ \bc{h_i}\in I_m(T)$. The edge $e$ was also feasible for the greedy algorithm, as $X_{i-1}+  \bc{e}\subseteq O_{i-1}$ and thus $O_{i-1}\in I_m(T)$ implies $X_{i-1}+\bc{e}\in I_m(T)$, as $I_m(T)$ is downward closed. Therefore $w(h_i)\geq w(e)$ and therefore $w(O_i)=w(O_{i-1})+w(h_i)-w(e)\geq w(O_{i-1})$ and as $O_{i-1}$ was optimal this implies that $O_i$ also needs to be optimal. Thus $X_i\subseteq O_i$.\\
    At the end, there exists an optimal set \(O\) with \(X_{\mathrm{greedy}}\subseteq O\). If the inclusion were strict, then some \(e\in O\setminus X_{\mathrm{greedy}}\) would still be feasible for the greedy mechanism by downward closedness, contradicting the stopping rule. Hence \(O=X_{\mathrm{greedy}}\), and the greedy set is optimal.
\end{proof}

\subsection{Vertex adjacency}\label{sec:appendix_vertex}

Let $(\mathcal{X},d)$ be a metric space and assume the vertex set $\mathcal{V}\subseteq\mathcal{X}$ has bounded diameter: $\sup_{u,v\in \mathcal{V}} d(u,v) \;\le\; 2B$. We consider weighted graphs $G=\br{V,E,w}$ with $V\subseteq\mathcal{V}$ and
edge weights induced by the metric, that is $w\br{\{u,v\}}=d(u,v)$.\\

For a graph $G=\br{V,E,w}$ and a finite set $U\subseteq \mathcal{V}$, write $G+U$ for a weighted graph $\br{V\cup U,E',w'}$ such that $E\subseteq E'$, where $w'$ agrees with $w$ on $E$ and all new edge weights are induced by $d$. We call graphs $G$ and $G'$ $m$-vertex adjacent if $G'=G+U$ for some $U\subseteq\mathcal{V}$ with $1\leq \abs{U}\leq m$, or vice versa.

For a retained graph $R=\br{V_R,E_R,w}$, define retain sensitivity under the addition of at most $m$ vertices by
\begin{equation*}
    \RS_f^{(m)}(R)
    =
    \max_{\substack{U\subseteq\mathcal{V}\\ 1\leq \abs{U}\leq m}}
    \abs{f\br{R}-f\br{R+U}},
\end{equation*}
where the maximum is also over all admissible edge sets of $R+U$ containing $E_R$. Note that the addition of vertices $U$ to $V_R$ can result in adjacent graphs with edge set at most $E_R\cup \{\{u,x\}\mid u\in U,\ x\in V_R\cup U,\ x\neq u\}$.

\begin{definition}
    Let $(\mathcal{V},d)$ be a bounded metric space. A Steiner tree for a given set of terminals $T\subseteq \mathcal{V}$ is a finite tree $G=(V,E,w)$ with $T\subseteq V\subseteq \mathcal{V}$ and edge weights induced by $d$. 
    The points $V\setminus T$ are called Steiner points. 
    The Steiner tree has total length $w(G)=\sum_{e\in E}w(e)$ and we call a graph $G^\star=\SMT(T)$ a \emph{Steiner minimal tree} for $T$, if $G^\star$ is a Steiner tree for $T$ and attains minimal length $w(G^\star)$.
\end{definition}

\begin{lemma}\label{lem:mst_smt}
    For a complete metric graph $K_V$ over vertex set $V$ in $(\mathcal{V},d)$, the following holds
    \begin{equation*}
        w(\MST(K_{V\cup U}))\geq w(\mathrm{SMT}(V))
    \end{equation*}
    for any finite $U\subseteq\mathcal{V}$.
\end{lemma}
\begin{proof}
    Let $T$ be an MST of $K_{V\cup U}$. Then $T$ is a tree that connects all vertices in $V\cup U$, and in particular it connects all terminals in $V$ while using the vertices in  $U$ as an (optional) Steiner vertex. Hence, $T$ is a feasible solution to the Steiner tree problem with terminal set $V$ (where Steiner vertices are allowed from $\mathcal{V}$).
    Since $\SMT(V)$ is the minimum-weight feasible Steiner tree connecting $V$, we have $w(\mathrm{SMT}(V)) \le w(T) = w(\MST(K_{V\cup U}))$, which proves the claim.
\end{proof}

\begin{theorem}
    The ratio between retain sensitivity given a graph $R=(V_R,E_R,w)$ and global sensitivity of releasing the weight of a MST over $m$-vertex adjacent graphs is 
    \begin{equation*}
        \frac{\RS_f^{(m)}(R)}{\GS_{f}^{(m)}}\leq \frac{\max\{ m\max_{v\in\mathcal{V}}\min_{e\in \{\{v,v'\}\mid v'\in V_R\}} w(e), f(R)- \rho_df(K_{V_R})\}}{(n-2)B}
    \end{equation*}
    where $\rho_d=\inf_{T}\frac{w(\mathrm{SMT}(T))}{w(\MST(K_T))}$.
\end{theorem}

\begin{proof}
    The global sensitivity of releasing the weight of a MST with $n$ vertices is at least $(n-2)B$: Consider a star metric: let $V=\{v_1,\dots,v_n\}$ and add a center $c$. Set edge weights $d(c,v_i)=B$ for all $i$, and $d(v_i,v_j)=2B$ for all $i\neq j$. Then $\MST(V\cup\{c\})$ is the star through $c$ with total weight $nB$, while $\MST(V)$ uses only edges of weight $2B$ and thus has total weight $2B(n-1)$. Therefore, for $n\geq 3$: $\GS_f^{(m)}\geq \GS_f^{(1)}\geq \abs{2B(n-1)-nB}=(n-2)B$.

    For a graph $R$, consider the set $U\subseteq\mathcal{V}$, with $1\leq \abs{U}\leq m$, that is added to the graph such that the difference in MST weights is maximised. Let $f(R)=w(\MST(R))$ and $f(R+U)=w(\MST(R+U))$. For $u\in U$, define the lightest edge in the cut induced by $(V_R,u)$ as $w_1(u;V_R):=\min_{e\in \{\{u,v'\}\mid v'\in V_R\}}w(e)$.

    First, we show that the difference can at most increase by the sum of the lightest cut edges. Let $T_R=(V_R,E_{T_R})$ be a MST of $R$. For every $u\in U$, choose an edge $e_u=\{u,v_u\}$ with $v_u\in V_R$ such that $w(e_u)=w_1(u;V_R)$. Then $(V_R\cup U,E_{T_R}\cup\{e_u\mid u\in U\})$ is a spanning tree of $R+U$. Therefore, $f(R+U)\leq f(R)+\sum_{u\in U}w_1(u;V_R)\leq f(R)+m\max_{v\in\mathcal{V}}\min_{e\in \{\{v,v'\}\mid v'\in V_R\}}w(e)$, which shows the claim.

    For the second part we show that $f(R)-f(R+U)\leq f(R)-w(\SMT(V_R))\leq f(R)-\rho_df(K_{V_R})$. We will use that the weight of a MST over the potentially sparse graph $R=(V_R,E_R,w)$ is bigger than the MST over the complete graph $K_{V_R}$, as $E_R\subseteq\binom{V_R}{2}$. Putting this together with \Cref{lem:mst_smt}, we get
    \begin{equation*}
        f(R+U)\geq f(K_{V_R\cup U})\geq w(\SMT(V_R))\geq \rho_df(K_{V_R}),
    \end{equation*}
    where the last inequality follows by definition of $\rho_d\leq \frac{w(\SMT(V_R))}{f(K_{V_R})}$.

    Combining the two directions, we obtain 
    \[\RS_f^{(m)}(R)\leq \max\{m\max_{v\in\mathcal{V}}\min_{e\in \{\{v,v'\}\mid v'\in V_R\}}w(e), f(R)-\rho_df(K_{V_R})\}.\] Dividing by $\GS_f^{(m)}\geq (n-2)B$ gives the result.
\end{proof}

\section{Appendix: PCA}\label{sec:app_pca}

 Write \(\hat\Sigma_R = V_R
\Lambda_R V_R^\top\) with
\(\Lambda_R=\diag{\lambda_1\br{\hat\Sigma_R},\ldots,\lambda_d\br{\hat\Sigma_R}}\)
and
\(\lambda_1\br{\hat\Sigma_R}\ge\cdots\ge\lambda_d\br{\hat\Sigma_R}\ge
0\).  Define the rank-\(k\) projector \(P_k(\Sigma_R)=V_{R,k}V_{R,k}^\top\),
where \(V_{R,k}\) contains the top-\(k\) eigenvectors of
\(\hat\Sigma_R\).

\begin{lemma}[Davis-Kahan $\sin\Theta$]\label{lem:davis}\cite{davis1970rotation,yu2015useful}
Let $A,\tilde A\in\mathbb{R}^{n\times n}$ be symmetric and let
$V_k\in\mathbb{R}^{n\times k}$ (resp.\ $\tilde V_k$) contain orthonormal eigenvectors
corresponding to the $k$ largest eigenvalues of $A$ (resp.\ $\tilde A$).
Define the orthogonal projectors $P_k=V_kV_k^\top$ and $\tilde P_k=\tilde V_k\tilde V_k^\top$.
Assume the eigengap $\mathrm{gap}_k(A):=\lambda_k(A)-\lambda_{k+1}(A)>0$.
Then
\begin{equation*}
   \|\tilde P_k-P_k\|_2 \le \frac{2\|A-\tilde A\|_2}{\mathrm{gap}_k(A)}.
\end{equation*}

and $\|\tilde P_k-P_k\|_F \le \frac{2\sqrt{2}\|A-\tilde A\|_F}{\mathrm{gap}_k(A)}$.
\end{lemma}

\rspca*

\begin{proof}
Fix any $x_{n+1},\ldots,x_{n+m}$ with $\|x_{n+j}\|_2\le B$ for all $j\in[m]$ and let $R' := R\cup\{x_{n+1},\ldots,x_{n+m}\}$.
Write $\hat\Sigma_{R'}=\frac{1}{n+m}\sum_{i=1}^{n+m}x_ix_i^\top = \frac{n}{n+m}\hat\Sigma_R + \frac{1}{n+m}\sum_{j=1}^m x_{n+j}x_{n+j}^\top.$
Hence $\hat\Sigma_{R'}-\hat\Sigma_R = \frac{1}{n+m}\Big(\sum_{j=1}^m x_{n+j}x_{n+j}^\top-m\hat\Sigma_R\Big).$
Taking operator norms gives
\[
\|\hat\Sigma_{R'}-\hat\Sigma_R\|_2
\le \frac{1}{n+m}\Big(\sum_{j=1}^m\|x_{n+j}x_{n+j}^\top\|_2+m\|\hat\Sigma_R\|_2\Big)
\le \frac{1}{n+m}\Big(mB^2 + mB^2\Big)
= \frac{2mB^2}{n+m},
\]
where we used $\|\hat\Sigma_R\|_2 \le \frac{1}{n}\sum_{i=1}^n \|x_ix_i^\top\|_2 \le B^2$. The same upper bound holds for the Frobenius norm as $\lVert xx^\top\rVert_F=\lVert x\rVert^2\leq B^2$ and $\lVert \hat\Sigma_{R}-\hat\Sigma_{R'}\rVert_F\leq \nicefrac{2mB^2}{n+m}$.\\

Now apply Davis-Kahan (Lemma~\ref{lem:davis}) with $A=\hat\Sigma_R$ and $\tilde A=\hat\Sigma_{R'}$:
\begin{align*}
    &\|P_k(\hat\Sigma_{R'})-P_k(\hat\Sigma_R)\|_2 \le \frac{2\|\hat\Sigma_{R'}-\hat\Sigma_R\|_2}{\mathrm{gap}_k(R)}
\le \frac{4mB^2}{(n+m)\mathrm{gap}_k(R)} \text{ and }\\
&\|P_k(\hat\Sigma_{R'})-P_k(\hat\Sigma_R)\|_F \le \frac{2\sqrt{2}\|\hat\Sigma_{R'}-\hat\Sigma_R\|_F}{\mathrm{gap}_k(R)}
\le \frac{4\sqrt{2}mB^2}{(n+m)\mathrm{gap}_k(R)}.
\end{align*}
Since the bound is uniform over $x_{n+1},\ldots,x_{n+m}$, taking the maximum over additions yields the claim.
\end{proof}

\begin{lemma}\label{lem:pca_unlearn}
    Let $R\subseteq \mathbb{R}^{n\times d}$ with $\lvert R\rvert=n$ and $R'=R\cup\{z_{n+1},\dots, z_{n+m}\}$. Then the unlearning algorithm that outputs
    \begin{align*}
        &\bar{\mathcal{A}}(U, \mathcal{A}(R\cup U), T(R\cup U))= P_k(V_{R',k}V_{R',k}^\top +E)\\
        &\bar{\mathcal{A}}(\emptyset, \mathcal{A}(R), T(R))= P_k(V_{R,k}V_{R,k}^\top +E)
    \end{align*}
    with $E\in\mathbb{R}^{d\times d}$ symmetric, constructed by drawing i.i.d. $\mathcal{N}(0,\sigma(R)^2)$ noise on the upper triangular entries and mirroring to the lower triangle, satisfies $(\varepsilon,\delta)$-unlearning when $\sigma(R)=\frac{4m\sqrt{2}B^2}{(n+m)\mathrm{gap}_k(R)}c_{\varepsilon,\delta}$
\end{lemma}

\begin{proof}
    The claim follows directly as $\lVert V_{R,k}V_{R,k}^\top-V_{R',k}V_{R',k}\lVert_F\leq \RS^{(m)}(R)\leq 4m\sqrt{2}B^2/(n+m)\mathrm{gap}_k(R)$ by \Cref{lem:pca_rs}. Thus, choosing $\sigma(R)=\RS^{(m)}(R)\cdot c_{\varepsilon,\delta}$ and applying \Cref{thm:retain_unlearn} yields $(\varepsilon,\delta)$-unlearning. Finally, applying an additional rank-$k$ projection is data-independent post-processing and therefore preserves the unlearning guarantee.
\end{proof}

We can even give a utility guarantee for the unlearning algorithm from \Cref{lem:pca_unlearn}, following a simplified analysis from \cite{dwork2014gauss}, as we do not need to privatise the eigengap or do PTR in order to give the unlearning guarantee. 

\begin{lemma}
    The unlearning algorithm as described in \Cref{lem:pca_unlearn} satisfies the following utility guarantee with respect to the retrained projector $V_{R,k}V_{R,k}^\top$ with high probability
    \begin{equation*}
        \lVert V_{R,k}V_{R,k}^\top-P_k(V_{R',k}V_{R',k}^\top+E)\lVert_2=O\br{\frac{mB^2}{(n+m)\mathrm{gap}_k(R)}}+O\br{\frac{m\sqrt{d}B^2c_{\varepsilon,\delta}}{(n+m)\mathrm{gap}_k(R)}}
    \end{equation*}
\end{lemma}
\begin{proof}
    We have 
    \begin{align*}
        \lVert V_{R,k}V_{R,k}^\top-P_k(V_{R',k}V_{R',k}^\top+E)\lVert_2&=\lVert V_{R,k}V_{R,k}^\top- V_{R',k}V_{R',k}^\top +V_{R',k}V_{R',k}^\top - P_k(V_{R',k}V_{R',k}^\top+E)\lVert_2\\
        &\leq\lVert V_{R,k}V_{R,k}^\top- V_{R',k}V_{R',k}^\top\rVert +\lVert V_{R',k}V_{R',k}^\top - P_k(V_{R',k}V_{R',k}^\top+E)\lVert_2
    \end{align*}
    For the first term, by the proof of \Cref{lem:pca_rs}, we have shown that for the operator norm $\lVert V_{R,k}V_{R,k}^\top- V_{R',k}V_{R',k}^\top\rVert\leq \frac{4mB^2}{(n+m)\mathrm{gap}_k(R)}$.
    Now, we show for the second term $\lVert V_{R',k}V_{R',k}^\top-P_k(V_{R',k}V_{R',k}^\top+E)\rVert_2=O\br{\frac{m\sqrt{d}B^2c_{\varepsilon,\delta}}{(n+m)\mathrm{gap}_k(R)}}$ and the claim follows.
    We know (Corollary 2.3.6 from \cite{tao2023topics}) that with high probability $\lVert E\rVert=O\br{\frac{4m\sqrt{2d}B^2c_{\varepsilon,\delta}}{(n+m)\mathrm{gap}_k(R)}}$, where $E$ is the Gaussian noise matrix from \Cref{lem:pca_unlearn}. By definition $\lVert V_{R',k}V_{R',k}^\top+E -V_{R',k}V_{R',k}^\top\rVert=\lVert E\rVert$. By Weyl's inequality, we furthermore have $\lvert \lambda_i(V_{R',k}V_{R',k}^\top+E)-\lambda_i(V_{R',k}V_{R',k}^\top)\rvert\leq\lVert E\rVert$ for all $i\in\{1,\dots, d\}$. As $\lambda_i(V_{R',k}V_{R',k}^\top)=0$ for $i\geq k+1$, this implies in particular $ \lvert \lambda_{i}(V_{R',k}V_{R',k}^\top+E)\rvert\leq \lVert E\rVert$ for $i>k$. \\
    Let $\hat{V}_k\hat{V}_k^\top= P_k(V_{R',k}V_{R',k}^\top+E)$ and set $X= V_{R',k}V_{R',k}^\top+E- \hat{V}_k\hat{V}_k^\top$. Note that by the above argument and using that the one term is a projector (with eigenvalue 1 and 0) $\lVert X\rVert=\max\bc{\max_{i\leq k}\lvert \lambda_i-1\rvert,\max_{i>k}\lvert \lambda_i\rvert}\leq \lVert E\rVert$.
    
    By the above argument, we have $\lVert X\rVert\leq \lVert E\rVert$ and we obtain
    \begin{align*}
        \lVert V_{R',k}V_{R',k}^\top-P_k(V_{R',k}V_{R',k}^\top+E)\lVert_2&=\lVert (V_{R',k}V_{R',k}^\top+E-V_{R',k}V_{R',k}^\top) -X\rVert\\
        &\leq \lVert V_{R',k}V_{R',k}^\top+E-V_{R',k}V_{R',k}^\top\rVert+\lVert X\rVert\leq 2\lVert E\rVert.
    \end{align*}

\end{proof}

\section{Appendix: SVM}

Let \(k\) be a positive semidefinite kernel on \(\cX\) with associated RKHS \(\br{\cH,\langle\cdot,\cdot\rangle_{\cH}}\) and feature map \(\phi:\cX\to\cH\) such that
\(k(x,x')=\langle \phi(x),\phi(x')\rangle_{\cH}\).
For $R=\{(x_i,y_i)\}_{i=1}^n$ with $y_i\in\{-1,+1\}$, the hard-margin kernel SVM is
$w_R\in\arg\min_{w\in\cH}\tfrac12\|w\|_{\cH}^2$ such that
$y_i\langle w,\phi(x_i)\rangle_{\cH}\ge 1,~\forall i\in[n]$. Define the \emph{true (distributional) margin} of \(\cD\) by ${\gamma
= \sup_{u\in\cH: \norm{u}_{\cH}=1}\ \inf_{(x,y)\in \supp{\cD}}
y\,\langle u,\phi(x)\rangle_{\cH},}$ and the \emph{empirical margin}
\(\gamma_R\) of a sample \(R\) by taking the minimum over
\((x_i,y_i)\in R\) instead of \((x,y)\in \supp{\cD}\).

For this, we use a well-known lemma about the SVM solution.
\begin{lemma}\label{lem:svm_maxmargin}
    For a dataset $R$, the SVM solution of the kernelised hard-margin problem $w_R$ has $\lVert w_R\rVert_{\mathcal{H}}=\frac{1}{\gamma_R}$, where $\gamma_R$ is the empirical margin. 
\end{lemma}

\begin{proof}
    Let $u_R= \arg\max_{w \in {\mathcal{H}}:\lVert w\rVert_{\mathcal{H}}=1} \min_{(x_i,y_i)\in R} y_i\langle w,\phi(x_i)\rangle_{\mathcal{H}}$ the vector that realises the empirical margin $\gamma_R$. \\
    It is easy to see that $\frac{u_R}{\gamma_R}$ is feasible for the SVM problem and thus $\lVert w_R\rVert_{\mathcal{H}}\leq \lVert \frac{u_R}{\gamma_R}\rVert_{{\mathcal{H}}}=\frac{1}{\gamma_R}$.\\
    Conversely, for $w_R$ we have $ \min_{(x_i,y_i)\in R} y_i\langle w_R,\phi(x_i)\rangle_{\mathcal{H}}\geq 1$. Then by definition of $\gamma_R$, 
    \begin{equation*}
        \gamma_R\geq \min_{(x_i,y_i)\in R} y_i\langle \frac{w_R}{\lVert w_R\rVert_{\mathcal{H}}},\phi(x_i)\rangle_{\mathcal{H}}=\frac{1}{\lVert w_R\rVert_{\mathcal{H}}}\min_{(x_i,y_i)\in R} y_i\langle w_R,\phi(x_i)\rangle_{\mathcal{H}}\geq \frac{1}{\lVert w_R\rVert_{\mathcal{H}}}
    \end{equation*}
    and hence $\lVert w_R\rVert_{\mathcal{H}} \geq \frac{1}{\gamma_R}$. Putting both inequalities together proves the claim. 
\end{proof}

\begin{lemma}\label{lem:svm}
    Let $\gamma>0$ be the true margin of $\mathcal{D}$. For any dataset $R$ with empirical margin $\gamma_R\geq \gamma>0$, the retain sensitivity satisfies 
\begin{equation*}
\RS_{\mathrm{SVM}}^{(m)}(R)\leq \sqrt{\frac{1}{\gamma^2}-\frac{1}{\gamma_R^2}}.
\end{equation*}
\end{lemma}

\begin{proof}
Let $Z$ be any set of at most $m$ points and let $w_{R'}$ be the solution to the SVM optimisation problem with respect to $R'=R\cup Z$. Consider $C_R=\{w\in {\mathcal{H}}\mid y_i\langle w, \phi(x_i)\rangle_{\mathcal{H}}\geq 1 \quad \forall (x_i,y_i)\in R\}$, which is closed convex set as an intersection of half-spaces. Note that $w_{R'}\in C_R$ and therefore $\frac{1}{\gamma_R}=\lVert w_R\rVert_{\mathcal{H}}\leq\lVert w_{R'}\rVert_{\mathcal{H}}\leq \frac{1}{\gamma}$ where the first equality follows from \Cref{lem:svm_maxmargin}. \\
The solution to the SVM problem $w_R$, is then given as the projection of 0 onto this convex set $w_R=\Pi_{C_R}(0)$, as this is the unique $w\in C_R$ such that $\lVert w_R\rVert_{\mathcal{H}}$ is minimal.\\
We use the variational inequality characterisation of the projection, which ensures that for all $w\in C_R$ we have $\langle -w_R, w-w_R\rangle_{\mathcal{H}} \leq 0$, i.e. the vector from $w_R$ to $0$ makes a right/obtuse angle with every direction from $w_R$ into $C_R$. In particular, as $w_{R'}\in  C_R$, we have
\begin{equation}\label{eq:variation_proj}
    \langle w_R, w_{R'}-w_R\rangle_{\mathcal{H}} \geq 0.
\end{equation}
Therefore, as  $\lVert w_{R'}\rVert_{\mathcal{H}}^2=\lVert w_{R'} -w_R+w_R\rVert_{\mathcal{H}}^2=\lVert w_{R'}-w_R\rVert_{\mathcal{H}}^2+\lVert w_R\rVert_{\mathcal{H}}^2+ 2\langle w_{R'}-w_R, w_R\rangle_{\mathcal{H}}$, we get that 
\begin{align*}
    \lVert w_{R'}-w_R\rVert_{\mathcal{H}}^2
    &= \lVert w_{R'}\rVert_{\mathcal{H}}^2- \lVert w_R\rVert_{\mathcal{H}}^2 - 2\langle w_{R'}-w_R, w_R\rangle_{\mathcal{H}}\\
    &\leq \lVert w_{R'}\rVert_{\mathcal{H}}^2- \lVert w_R\rVert_{\mathcal{H}}^2 \leq \frac{1}{\gamma^2}-\frac{1}{\gamma_R^2}
\end{align*}
where the first inequality follows from \Cref{eq:variation_proj}. Since this holds for every $Z$ with $|Z|\leq m$, the claim follows by taking the square root and taking the supremum over such $Z$.
\end{proof}

\section{Appendix: ERM}\label{app:erm}

\begin{lemma}\label{lem:erm_general}
Let $f$ be $L$-Lipschitz,  $w_R$ and $w_{R'}$ the empirical risk minimisers, where $R'=R\cup \{z_{n+1},\ldots,z_{n+m}\}$ for $m\geq 1$, with $\hat{F}_R$ being $\lambda_R$-strongly convex and $\hat{F}_{R'}$ being $\lambda_{R'}$-strongly convex. Then,
    \begin{equation*}
        \lVert w_R-w_{R'}\rVert\leq \frac{\lVert \sum_{j=1}^m\nabla f(w_{R'},z_{n+i})\rVert}{n\lambda_R}\leq\frac{mL}{n\lambda_{R}}.
    \end{equation*}
    and by the same argument also $\lVert w_R-w_{R'}\rVert\leq mL/(n+m)\lambda_{R'}$.
\end{lemma}

\begin{proof}
    As $w_R$ and $w_{R'}$ are empirical risk minimisers w.r.t. datasets $R$ and $R'=R\cup \{z_{n+1},\dots,z_{n+m}\}$, first-order optimality gives $\langle\nabla\hat{F}_R(w_R),w_{R'}-w_R\rangle\geq 0$ and $\langle\nabla\hat{F}_{R'}(w_{R'}),w_{R'}-w_R\rangle\leq 0$. Using $(n+m)\nabla\hat{F}_{R'}(w_{R'})=n\nabla\hat{F}_R(w_{R'})+\sum_{i=1}^m\nabla f(w_{R'},z_{n+i})$, we have by $\lambda_{R}$-strong convexity that the gradients are $\lambda_R$ strongly monotone and thus
    \begin{align*}
        n\lambda_{R}\lVert w_{R}-w_{R'}\rVert^2&\leq
        n\langle\nabla\hat{F}_{R}(w_{R'})-\nabla\hat{F}_{R}(w_R),w_{R'}-w_R\rangle\\
        &\leq-\left\langle\sum_{i=1}^m\nabla f(w_{R'},z_{n+i}),w_{R'}-w_R\right\rangle
        \leq\left\lVert\sum_{i=1}^m\nabla f(w_{R'},z_{n+i})\right\rVert\lVert w_R-w_{R'}\rVert.
    \end{align*}
    
    As $f$ is $L$-Lipschitz and differentiable and therefore an application of the triangle inequality and $\lVert\nabla f(w,z)\rVert\leq L$ for all $w\in \mathcal{W}, z\in\mathcal{Z}$ immediately yields the claimed upper bound of $\frac{mL}{n\lambda_{R}}$.
    Repeating the same argument with roles of $R,R'$ interchanged gives the second claim. 
\end{proof}

\subsection{Examples of MSE and Logistic Loss}\label{app:mse_logloss}

\Cref{table:erm_rs_vs_gs} instantiates \Cref{lem:erm_general_ratio} for two common ERM objectives: mean squared error (MSE) and logistic regression. 
The following examples calculate the respective constants explicitly, which are also used in the ERM experiments \Cref{subfig:passive_erm_mse,subfig:passive_erm_logloss,fig:RS_vs_GS_active,fig:active_noise,fig:d2d-utility-logloss-mse}.

\begin{table}
	  \caption{Retain vs.\ global sensitivity for ERM.}
	  \label{table:erm_rs_vs_gs}
	  \begin{center}
	    \begin{small}
	      \begin{sc}
            \begin{tabular}{llc}
                \toprule
                \textbf{Loss} &  $\RS_{\ERM}^{(m)}(R)$ & $\GS_{\ERM}^{(m)}$ \\
                \midrule
                \makecell{MSE $\lambda=0$\\MSE $\lambda>0$}
                & $mL^{\MSE}/n\lambda_{R}^{\MSE}$ \eqref{eq:lse_passive} & \makecell{$\infty$\\$mL^{\MSE}/n\lambda$} \\
                \midrule
                \makecell{Log Loss $\lambda=0$ \\ Log Loss $\lambda>0$} & $mL^{\mathrm{LogL}}/n\lambda_{R}^{\mathrm{LogL}}$ \eqref{eq:logl_passive}  & \makecell{$\infty$ \\$mL^{\mathrm{LogL}}/n\lambda$} \\
                \bottomrule
            \end{tabular}
      \end{sc}
    \end{small}
  \end{center}
\end{table}

\begin{example}\label{ex:erm_general_mse}
    Assume there exists $B>0$ s.t. $\lVert x\rVert \leq B$ and labels $\lvert y\rvert\leq 1$ for all $(x,y)\in\mathcal{Z}$. \\
    We consider the ridge-regularised least squares for a dataset $R=\{(x_i,y_i)_{i=1}^n\}$: $\hat{F}_R(w)=\frac{1}{2n}\lVert Xw-y\rVert^2+\frac{\lambda}{2}\lVert w\rVert ^2$ for $\lambda\geq 0$. If $\frac{1}{n}X^\top X+\lambda I$ is invertible, the minimizer has the closed form $w_R=(\frac{1}{n}X^\top X+\lambda I)^{-1}\frac{1}{n}X^\top y$.\\
    We have,
    \begin{equation*}\label{eq:MSE_L}
        \lVert \nabla f(w_{R'}, z)\rVert
        \leq\underbrace{\frac{B^3}{\lambda_{\min}(\frac{1}{n+1}X^\top X+\lambda I)}+B}_{=:L^{\MSE}}
    \end{equation*} 
    where we use $\Vert w_{R'}\rVert\leq B/\lambda_{\min}(\frac{1}{n+1}X^\top X+\lambda I)$, as the domain is bounded by $B$. We compute $\lambda_R^{\MSE}=\inf_{w\in W}\lambda_{\min}(\nabla\hat{F}_R(w))=\lambda_{\min}(\frac{1}{n}X^\top X+\lambda I)$. 
    Plugging this into the upper bound from \Cref{lem:erm_general} gives
    \begin{equation}\label{eq:lse_passive}
        \RS_{\MSE}(R)\leq \frac{mL^{\MSE}}{n\lambda_{R}^{\MSE}}.
    \end{equation}
\end{example}

\begin{example}\label{ex:erm_general_logloss}
    Assume there exists $B,R_w>0$ s.t. $\lVert x\rVert \leq B$ and $\lVert w\rVert \leq R_w$ for all $w\in\mathcal{W}$ and labels $\lvert y\rvert\leq 1$. We consider ridge-regularised logistic risk minimization on a dataset $R=\{(x_i,y_i)\}_{i=1}^n$: $\hat{F}_R(w)=\frac{1}{n}\sum_{i=1}^n (\log(1+\exp(yw^\top x))-yw^\top x)+\frac{\lambda}{2}\lVert w\rVert ^2$ for $\lambda\geq 0$. 
    For $\lambda\geq 0$, we can upper bound $\lVert \nabla f(w_{R'},z)\rVert\leq B+\lambda R_w=L$.
    Furthermore, we get by a direct computation $\lambda_R^{\mathrm{LogL}}=\inf_{w:\lVert w\rVert \leq R}\lambda_{\min}(\nabla^2\hat{F}_R(w))\geq\lambda_{\min}(1/nX^\top X(2\cosh(BR_w/2))^{-2}+\lambda I)$ and thus by plugging into \Cref{lem:erm_general}
    \begin{equation}\label{eq:logl_passive}
        \RS_{\mathrm{LogL}}(R)\leq \frac{mL^{\mathrm{LogL}}}{n\lambda_{R}^{\mathrm{LogL}}}.
    \end{equation}
\end{example}

\section{Appendix: Active Unlearning}

\subsection{Descent to Delete}\label{app:d2d}

For details on D2D, see \Cref{alg:d2d}.\\
We define the global condition number $\kappa:=\nicefrac{\beta}{\lambda}$ and the corresponding contraction factor
$\gamma:=\nicefrac{\kappa-1}{\kappa+1}\in(0,1)$ for the standard stepsize $\eta=\nicefrac{2}{\lambda+\beta}$.
We capture data-dependence by assuming the empirical risk $\widehat{F}_R$ for a given $R$ is $\lambda_R$-strongly convex, with \( \lambda_R = \inf_{w\in\cW} \lambda_{\min}\br{\nabla^2 \widehat{F}_R(w)}\), and $\beta_R$-smooth, with \(\beta_R = \sup_{w\in\cW} \lambda_{\max}\br{\nabla^2 \widehat{F}_R(w)}\). Lastly, we define the data dependent condition number $\kappa_R$, contraction factor $\gamma_R$ and step size $\eta_R$ using the corresponding data-dependent quantities.

\begin{algorithm}[tb]

  \caption{Unlearning via Projected Gradient Descent (Descent-to-Delete \cite{neel2021descent})}
  \begin{algorithmic}[1]\label{alg:d2d}
    \item[] \textbf{Input: } Dataset and Trained model: $R'=R\cup\{z_{n+1},\dots, z_{n+m}\}$, $\mathcal{A}(R')=w_{R'}$, deletion requests: $U=\{z_{n+1},\dots, z_{n+m}\}$, set of constraints: $\mathcal{W}$, stepsize: $\eta_R$, unlearning guarantee: $(\varepsilon,\delta)$, noise level $\sigma$
    \STATE Set $I=I_R(\varepsilon,\delta,\sigma)$ as in \eqref{eq:iter_bound_retain}
    \STATE Construct retain set:\; $R = R'\setminus\{z_{n+1},\dots, z_{n+m}\}$\;
    \STATE Initialize:\; $w'_0 = w_{R'}$\;
    \FOR{$i=1$ {\bfseries to} $I$}
    \STATE Set $w'_i = \mathrm{Proj}_{\mathcal{W}}\!\left(w'_{i-1} - \eta_R \nabla \hat F_{R}(w'_{i-1})\right)$\;
    \ENDFOR
    \STATE Draw $\nu\sim\mathcal{N}(0,\sigma^2 I_d)$
    \item[]  \textbf{Output: }$\bar{w} = w'_I+ \nu$.
  \end{algorithmic}
\end{algorithm}

\begin{restatable}{lemma}{ermdescent}\label{lem:d2d_retain}
    The unlearning algorithm $\bar{\mathcal{A}}$ as described in \Cref{alg:d2d} is $(\varepsilon,\delta)$-unlearning for $I\geq I_R(\varepsilon,\delta)$ steps of gradient descent where 
\begin{equation}\label{eq:iter_bound_retain}
    I_R(\varepsilon,\delta,\sigma) := \max\{0,\left\lceil \frac{\ln\left(\frac{mL}{n\lambda_R\sigma b(\varepsilon,\delta)}\right)}{\ln(1/\gamma_R)} \right\rceil\}.
\end{equation}
with $b(\varepsilon,\delta)=\sqrt{2\log(1/\delta)+2\varepsilon}-\sqrt{2\log(1/\delta)}$.
\end{restatable}

\begin{proof}
    Consider the projection map from \Cref{alg:d2d} for a given dataset $R$ defined as $G_R(w)=\mathrm{Proj}(w-\eta\nabla\hat{F}_R(w))$, and let $R'=R\cup U$ with ERM $w_{R'}=\argmin_{w\in\mathcal{W}}\hat{F}_{R'}(w)$.\\
 By first-order optimality 
$w_R=\mathrm{Proj}_{\mathcal{W}}(w_R-\eta\nabla\hat{F}_R(w_R))$ and
nonexpansiveness of projection, the gradient-step contraction (\Cref{fact:gd_contraction}) therefore give for $\eta=\frac{2}{\lambda_R+\beta_R}$,
\[
\|G_R(w_{R'})-w_R\|_2
\leq \|(w_{R'}-w_R)-\eta(\nabla\hat{F}_R(w_{R'})-\nabla\hat{F}_R(w_R))\|_2
\leq \gamma_R\|w_{R'}-w_R\|_2.
\]
    We use \Cref{lem:erm_general} to upper bound $\lVert w_{R'}-w_R\rVert\leq\frac{mL}{\lambda_Rn}=:\Delta_{\RS}^{(m)}$ to obtain $\lVert G_R^I(w_{R'})-w_R\rVert\leq \gamma_R^I\Delta_{\RS}^{(m)}$ and this bound does not depend on the delete request $z\in U$, it is an upper bound for the retain sensitivity for unlearning, according to~\Cref{def:retain_unlearn}.\\
    By the Gaussian mechanism \Cref{fact:gaussian_mech}, to make the two Gaussian outputs with same noise $\sigma$ indistinguishable, for a fixed $(\varepsilon,\delta,\sigma)$ the largest possible mean shift\footnote{here we use a slightly tighter analysis of the tail of a Gaussian distribution instead of the constant from \Cref{fact:gaussian_mech}, where two Gaussians with mean shift $\GS$ and noise $\sigma$, the two Gaussians are $(\varepsilon,\delta)$-indistinguishable for $\varepsilon\geq \GS^2/2\sigma^2+(\GS/\sigma)\sqrt{2\log(1/\delta)}$ (see~\cite{bun2016concentrated}). Solving for $\GS$ gives the max shift for given $\varepsilon,\delta,\sigma$.} is $\sigma b(\varepsilon,\delta)$. Thus, after $I$ steps of gradient descent, we require 
    \begin{equation}\label{eq:d2d-noise}
         \gamma_R^I\Delta_{\RS}^{(m)} \leq \sigma b(\varepsilon,\delta).
    \end{equation}
    Solving for $I$ gives, that when $I\geq I_R( \varepsilon,\delta,\sigma):=\max\{0, \lceil\frac{\log(\Delta_{\RS}^{(m)}/\sigma b(\varepsilon,\delta)}{\log(1/\gamma_R)}\rceil\}$, we satisfy $(\varepsilon,\delta)$-unlearning with $\sigma$ noise, according to~\Cref{thm:retain_unlearn}.
\end{proof}

\begin{remark}
    We include additional utility evaluations of D2D in \Cref{fig:d2d-utility-logloss-mse} we consider a different version of this bound, i.e., we solve \Cref{eq:d2d-noise} for $\sigma$, instead of $I$. This way we can directly compare test accuracy after the same number of iterations of D2D, where noise scales with global sensitivity or with retain sensitivity. Implicitly, the result in \Cref{lem:d2d_step_ratio} (see also \Cref{fig:active_logloss_d2d,fig:active_mse_d2d} for experimental validation) can also be understood as a utility result: for a fixed noise level and unlearning guarantee, a $\RS$-style analysis requires less steps to reach this guarantee.\\
    Overall, we confirm empirically the theoretical results: especially for small regularisation parameter $\lambda$, using a $\RS$-style noise calibration can recover performance compared to retraining after much fewer iterations than the standard $\GS$-style noise calibration. However, when $\lambda_R$ is close to 0 (MNIST (50) with MSE loss for small $\lambda$), even $\RS$-style noise calibration cannot recover performance within an iteration budget of $I=1000$.
\end{remark}

\begin{figure}[t]
\centering

\begin{tabular}{@{}c@{\hspace{1.2em}}|@{\hspace{0.8em}}c@{}}

\begin{subfigure}[t]{0.46\textwidth}
    \centering
    \begin{tabular}{@{}c cc@{}}
        & \multicolumn{2}{c}{LogLoss}\\
        & Noise $\sim\GS$ & Noise $\sim\RS$\\[0.3em]

        \rotatebox{90}{\hspace{.6cm}\scriptsize $\lambda=0.1$} &
        \includegraphics[width=0.40\linewidth]{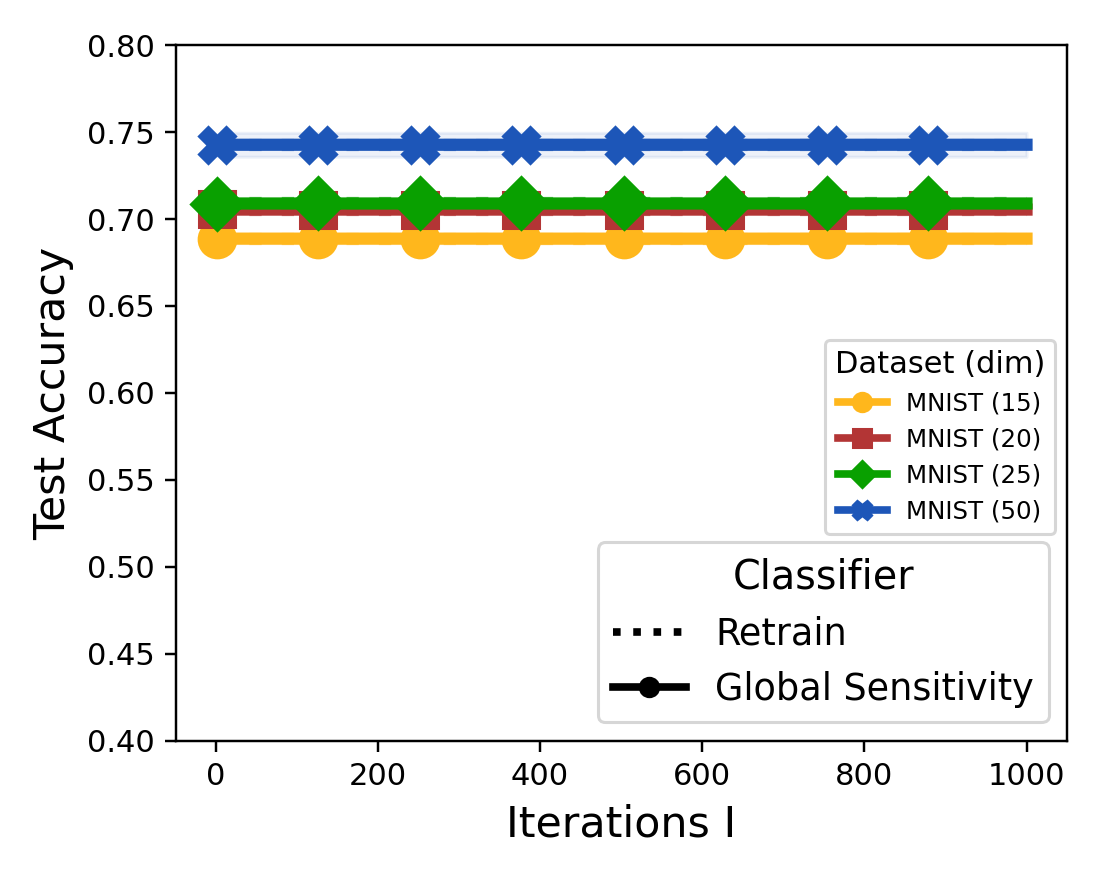} &
        \includegraphics[width=0.40\linewidth]{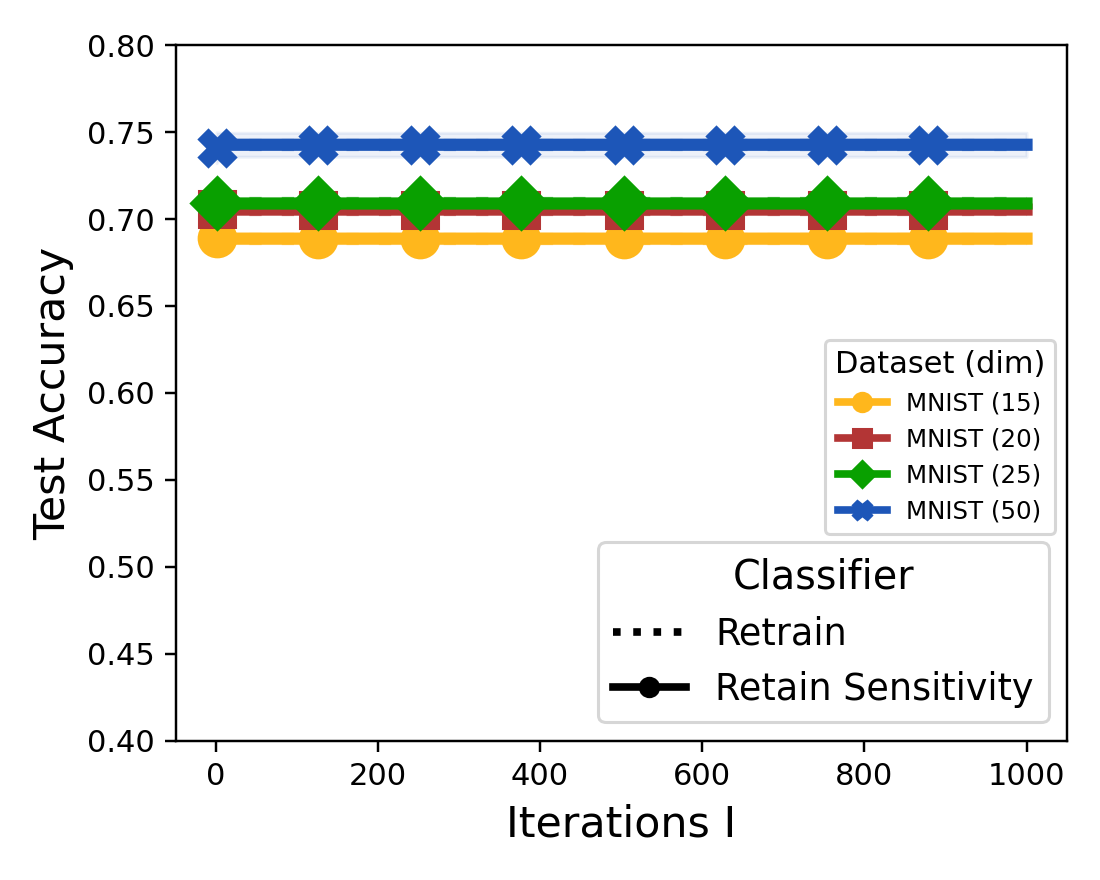}\\ \\[0.5em]

        \rotatebox{90}{\hspace{.6cm}\scriptsize $\lambda=0.01$} &
        \includegraphics[width=0.40\linewidth]{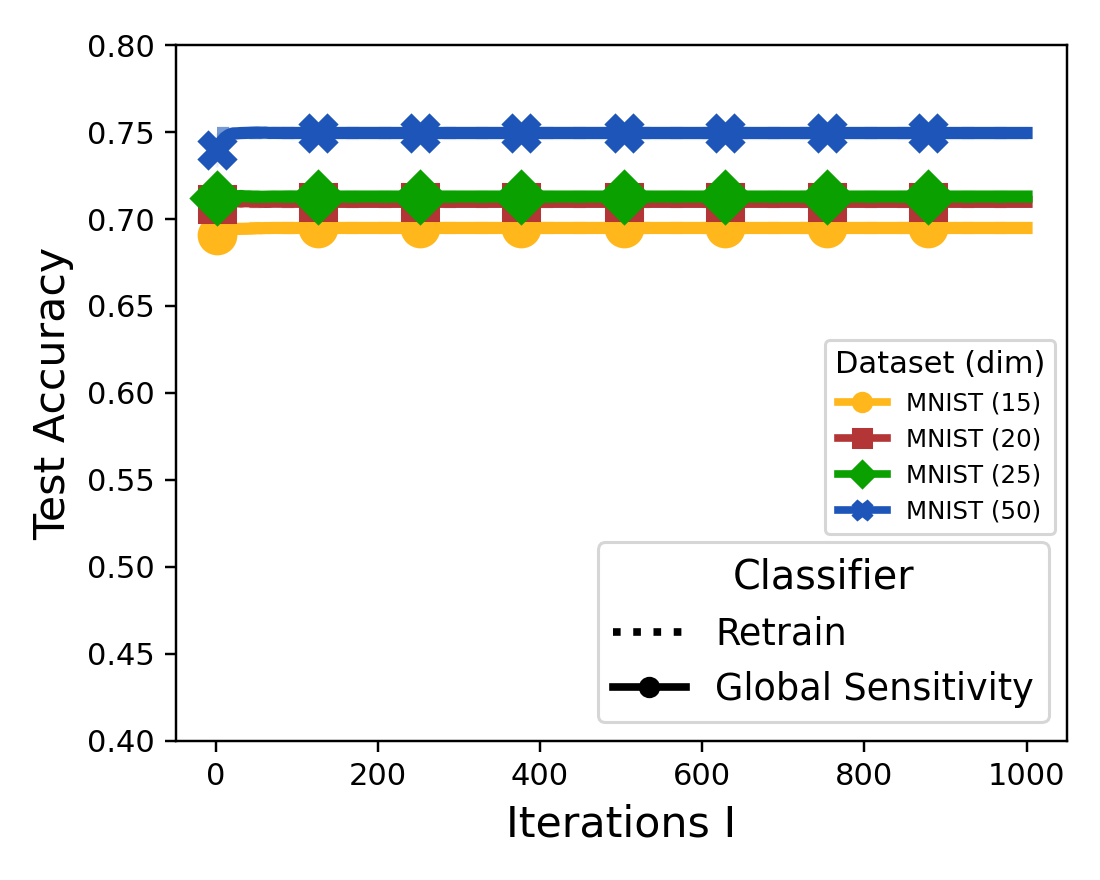} &
        \includegraphics[width=0.40\linewidth]{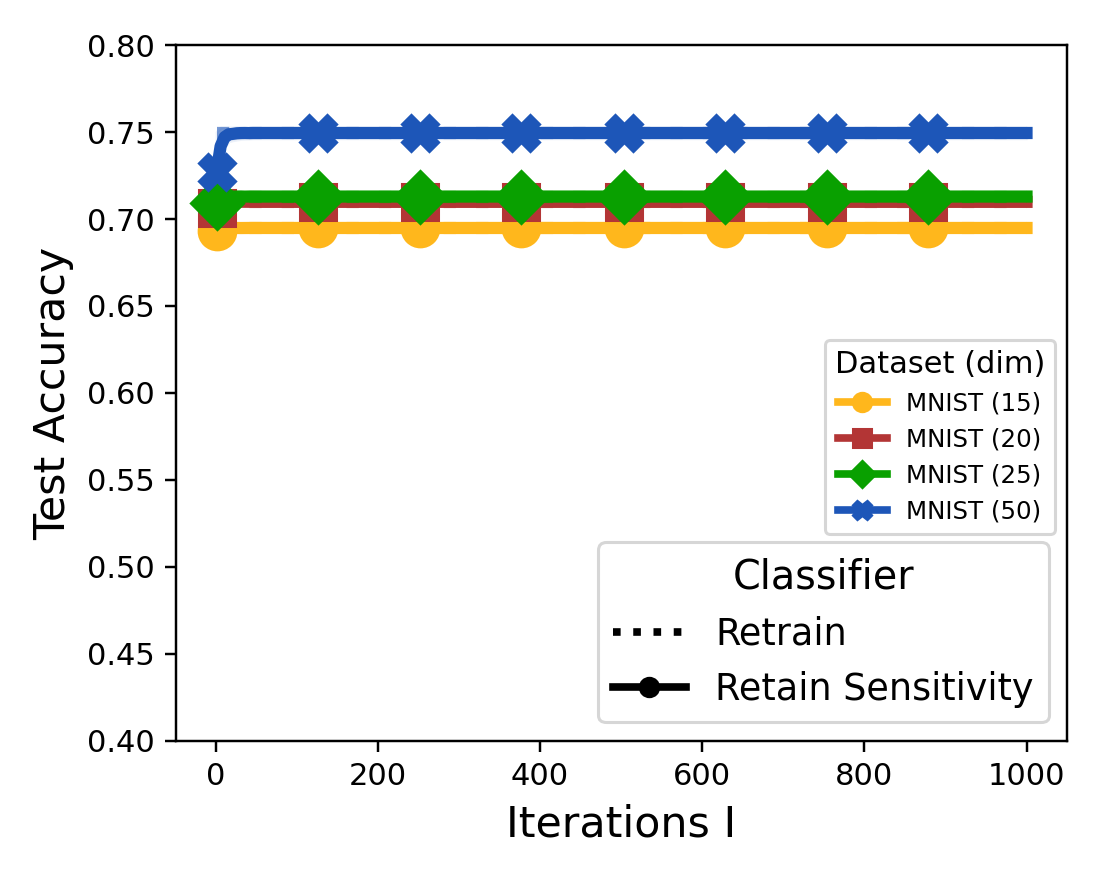}\\[0.5em]

        \rotatebox{90}{\hspace{.6cm}\scriptsize $\lambda=0.001$} &
        \includegraphics[width=0.40\linewidth]{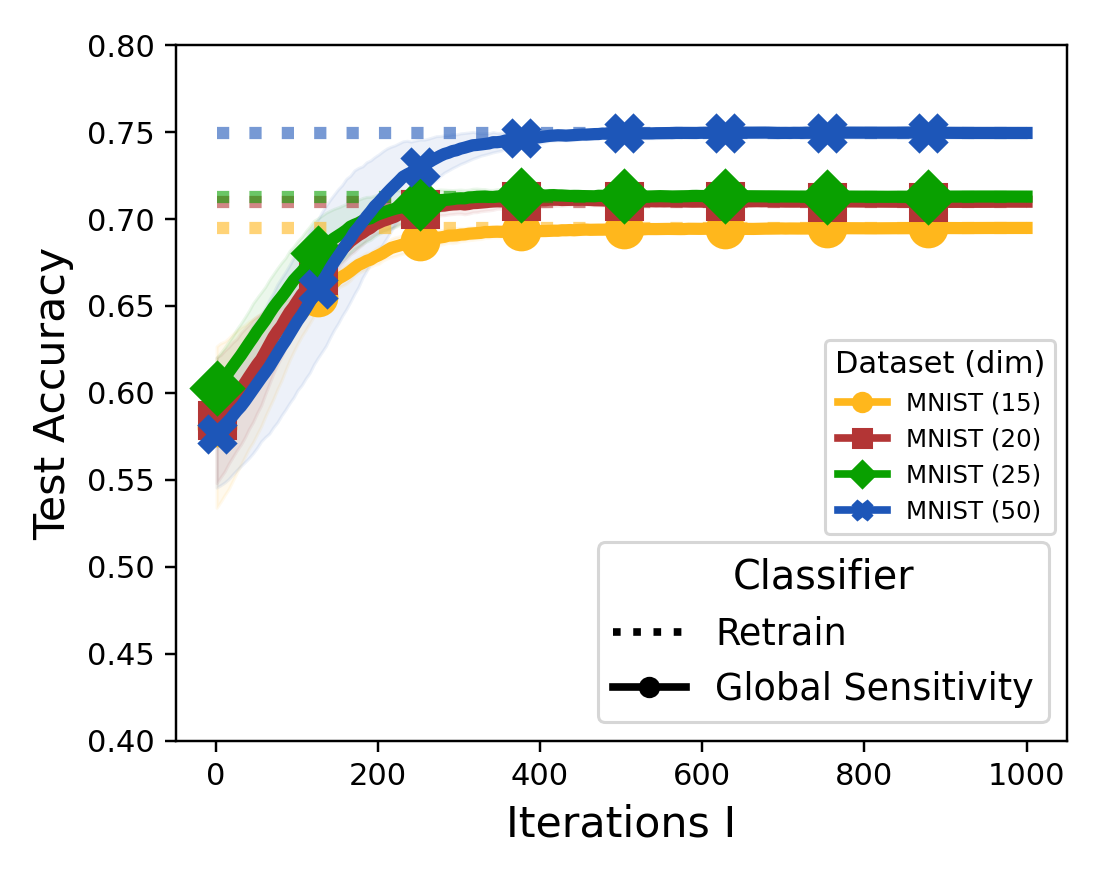} &
        \includegraphics[width=0.40\linewidth]{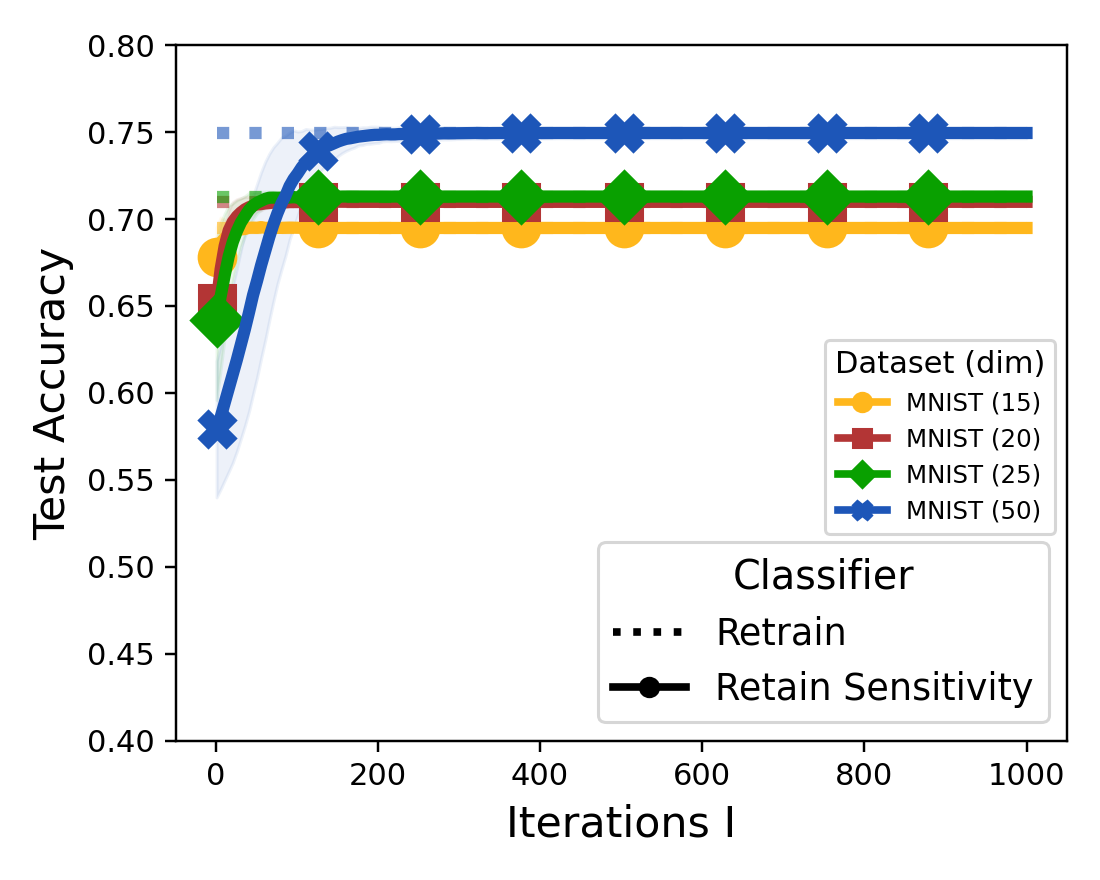}\\

        \rotatebox{90}{\hspace{.6cm}\scriptsize $\lambda=0.0001$} &
        \includegraphics[width=0.40\linewidth]{new_plots/run_20261012T080227Z_d2d_n10000_rp15-20-25-50_trials5_eps1e0_delta1e-9_bx1_rw1_seed42/acc_GS_lambda_0.0001_utility_d2d_n10000_rp15-20-25-50_trials5_eps1e0_delta1e-9_bx1_rw1_seed42_LogLoss.png} &
        \includegraphics[width=0.40\linewidth]{new_plots/run_20261012T080227Z_d2d_n10000_rp15-20-25-50_trials5_eps1e0_delta1e-9_bx1_rw1_seed42/acc_RS_lambda_0.0001_utility_d2d_n10000_rp15-20-25-50_trials5_eps1e0_delta1e-9_bx1_rw1_seed42_LogLoss.png}\\

        \rotatebox{90}{\hspace{.6cm}\scriptsize $\lambda=0.00001$} &
        \includegraphics[width=0.40\linewidth]{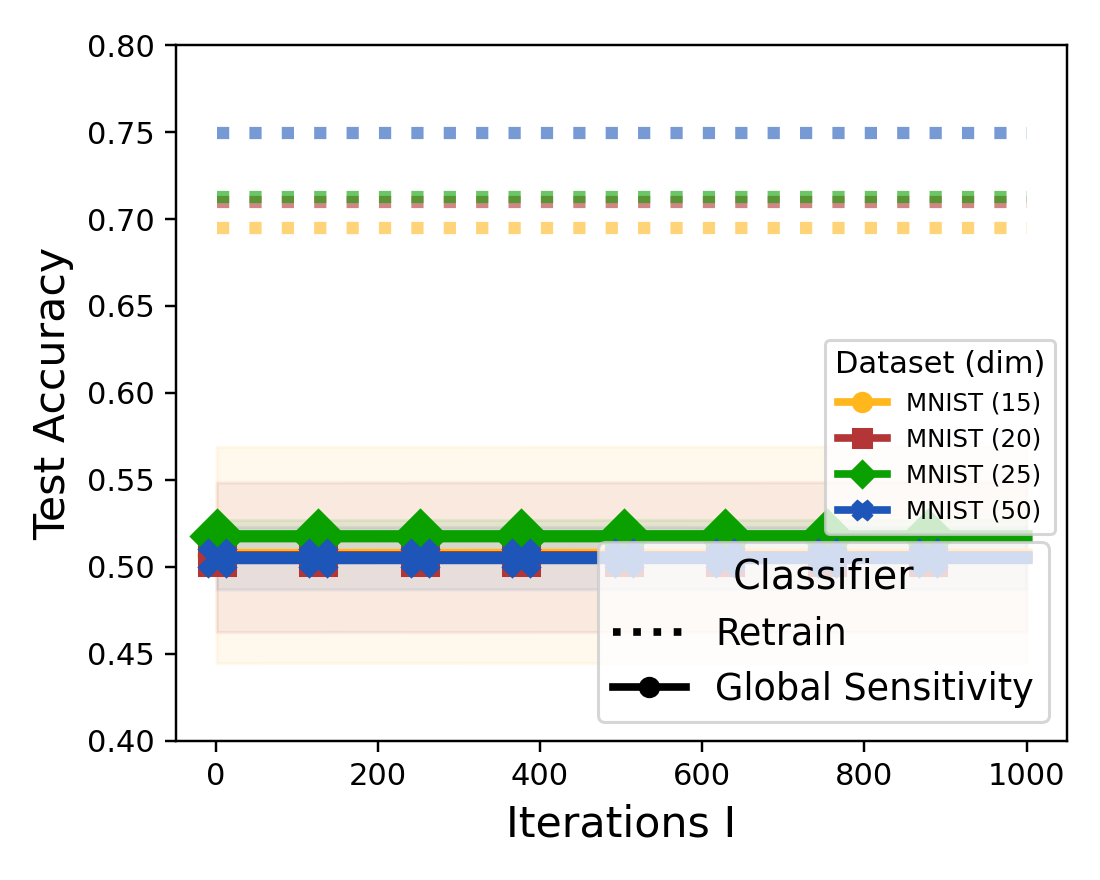} &
        \includegraphics[width=0.40\linewidth]{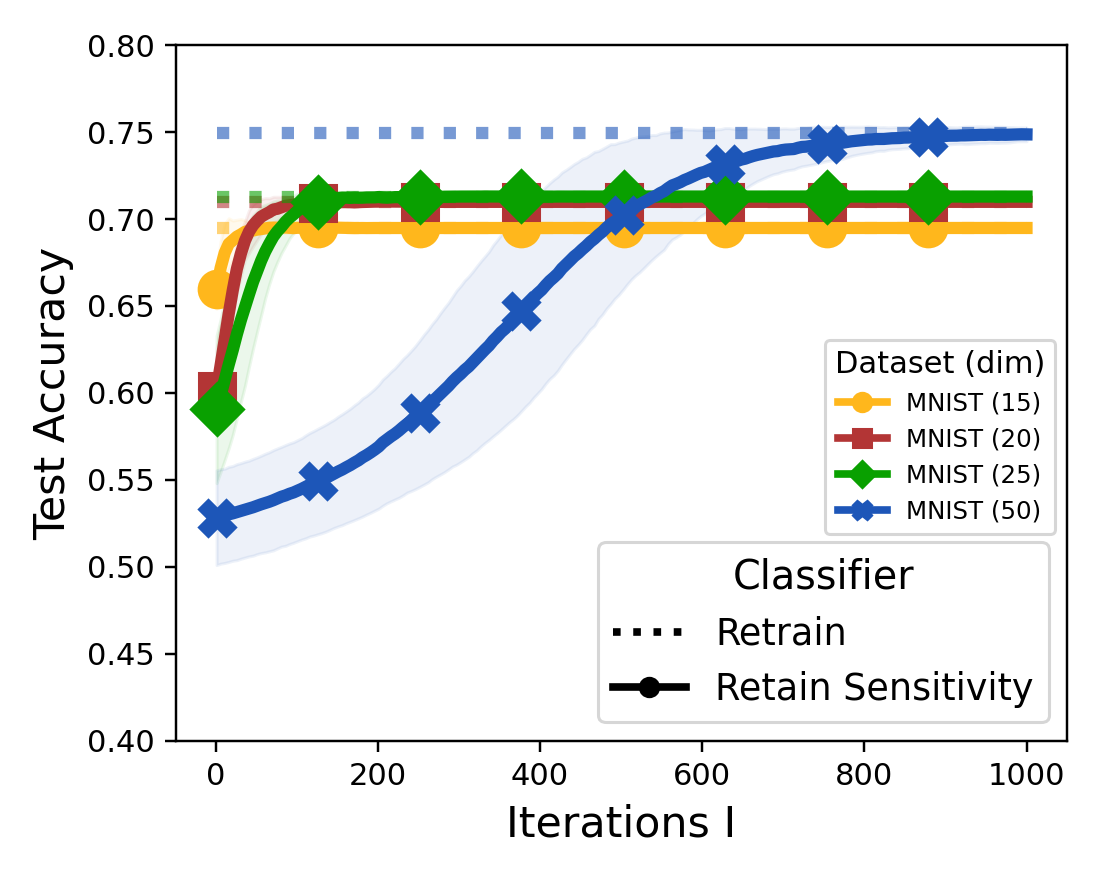}\\
        
    \end{tabular}
    \caption{\textbf{D2D (log loss) utility:} Test accuracy when noise is calibrated to global sensitivity (GS) versus retain sensitivity (RS) for log loss. }
    \label{fig:d2d-logloss-utility}
\end{subfigure}

&

\begin{subfigure}[t]{0.46\textwidth}
    \centering
    \begin{tabular}{@{}c cc@{}}
        & \multicolumn{2}{c}{MSE}\\
        & Noise $\sim\GS$ & Noise $\sim\RS$\\[0.3em]

         \rotatebox{90}{\hspace{.6cm}\scriptsize $\lambda=0.1$} &
        \includegraphics[width=0.40\linewidth]{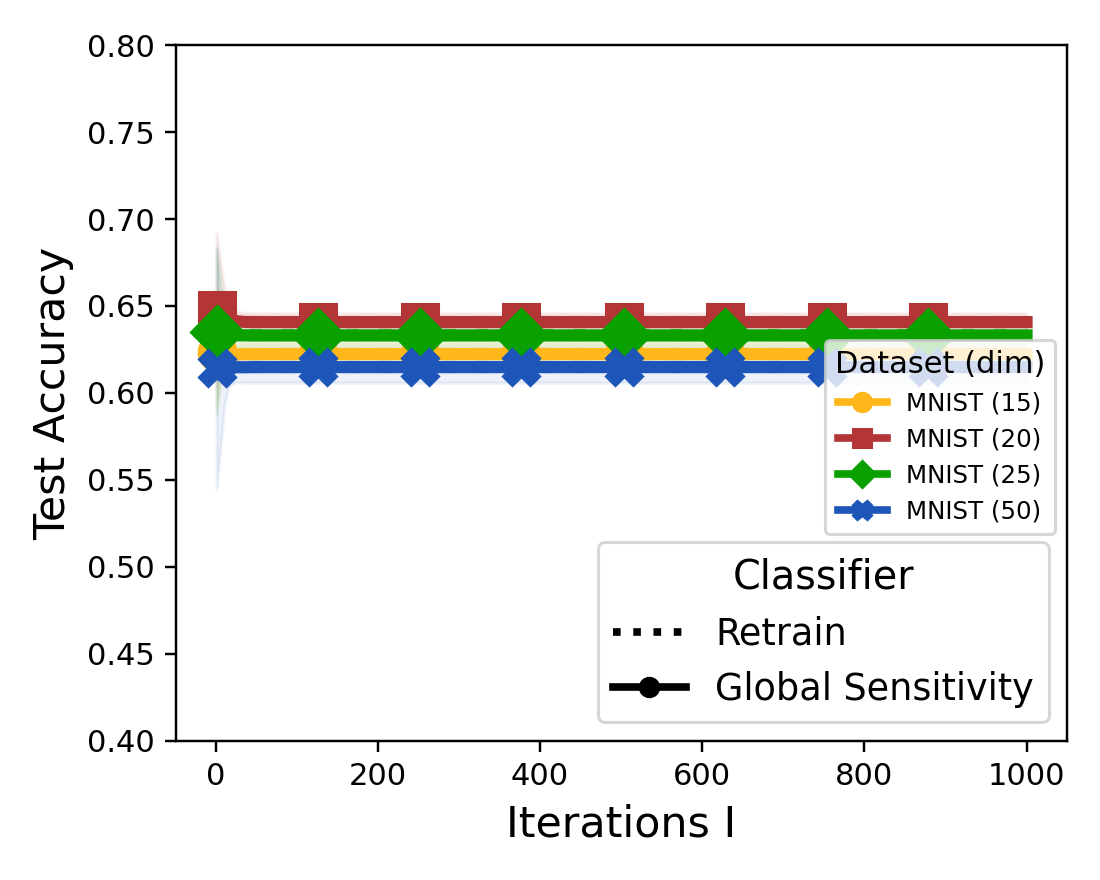} &
        \includegraphics[width=0.40\linewidth]{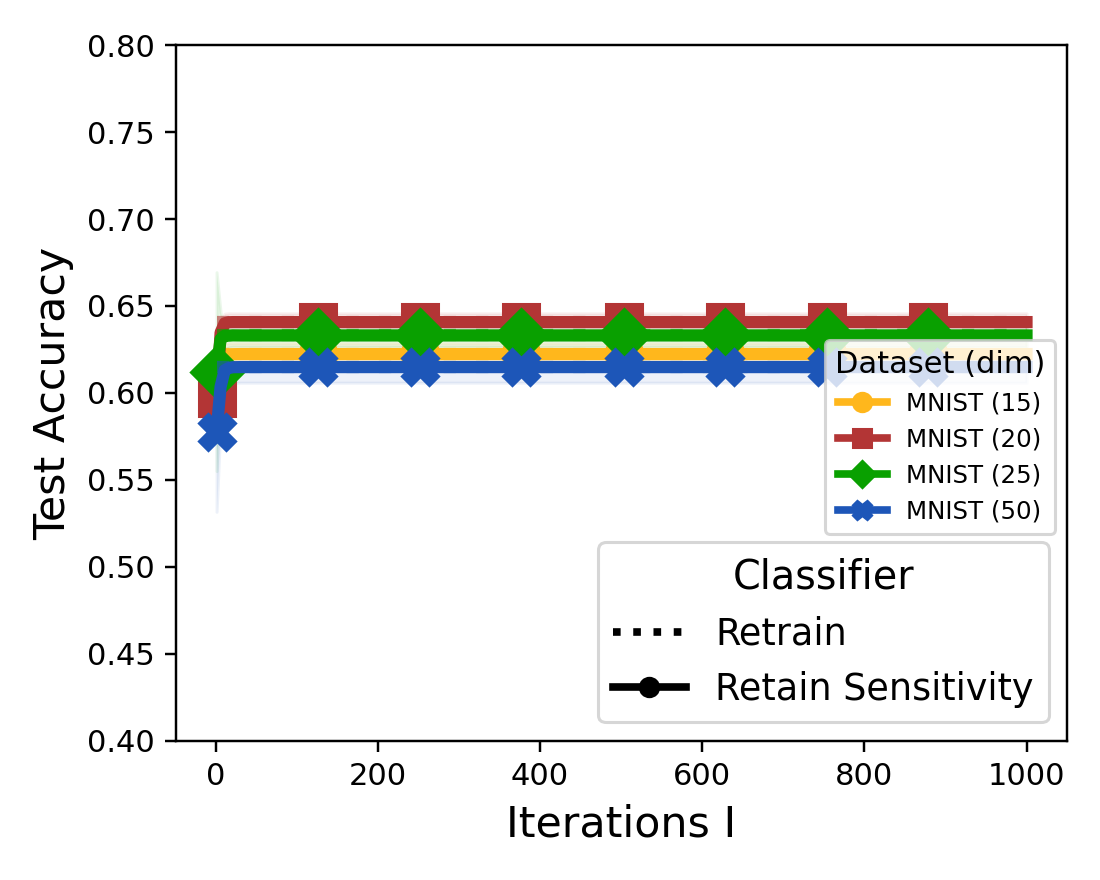}\\ \\[0.5em]

        \rotatebox{90}{\hspace{.6cm}\scriptsize $\lambda=0.01$} &
        \includegraphics[width=0.40\linewidth]{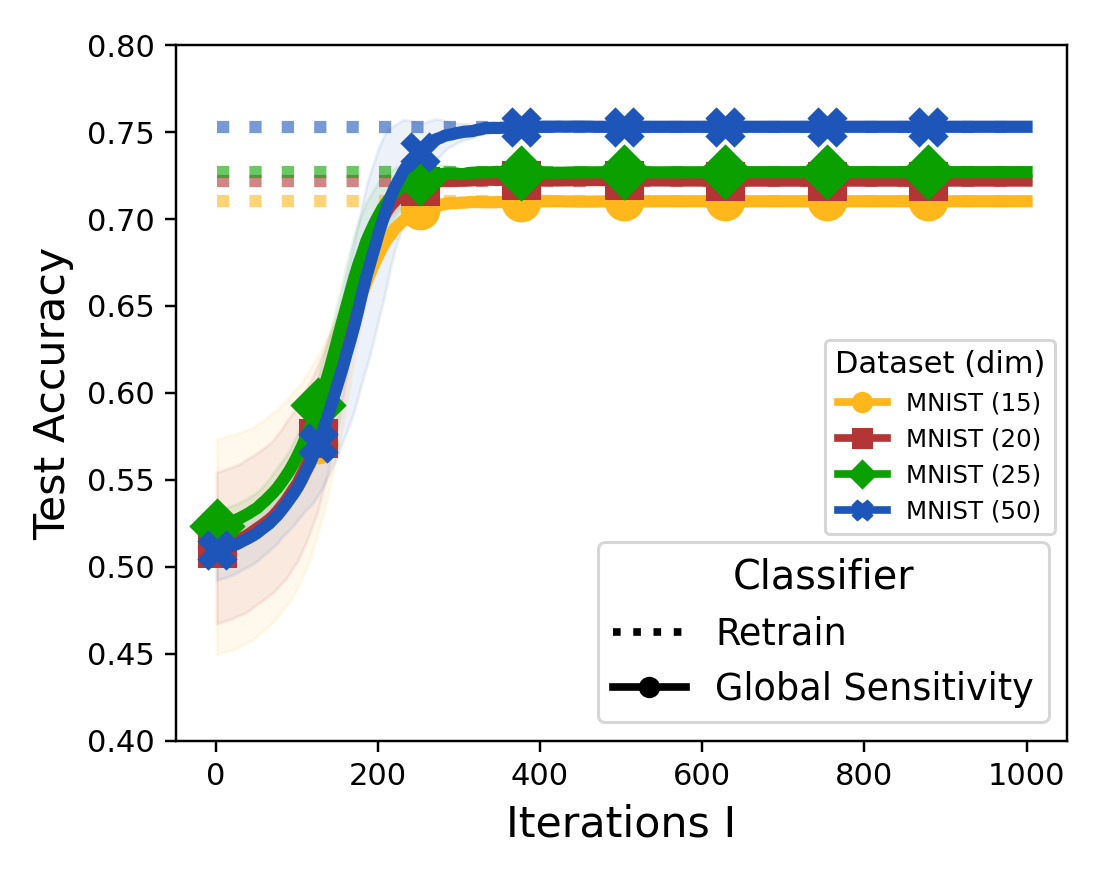} &
        \includegraphics[width=0.40\linewidth]{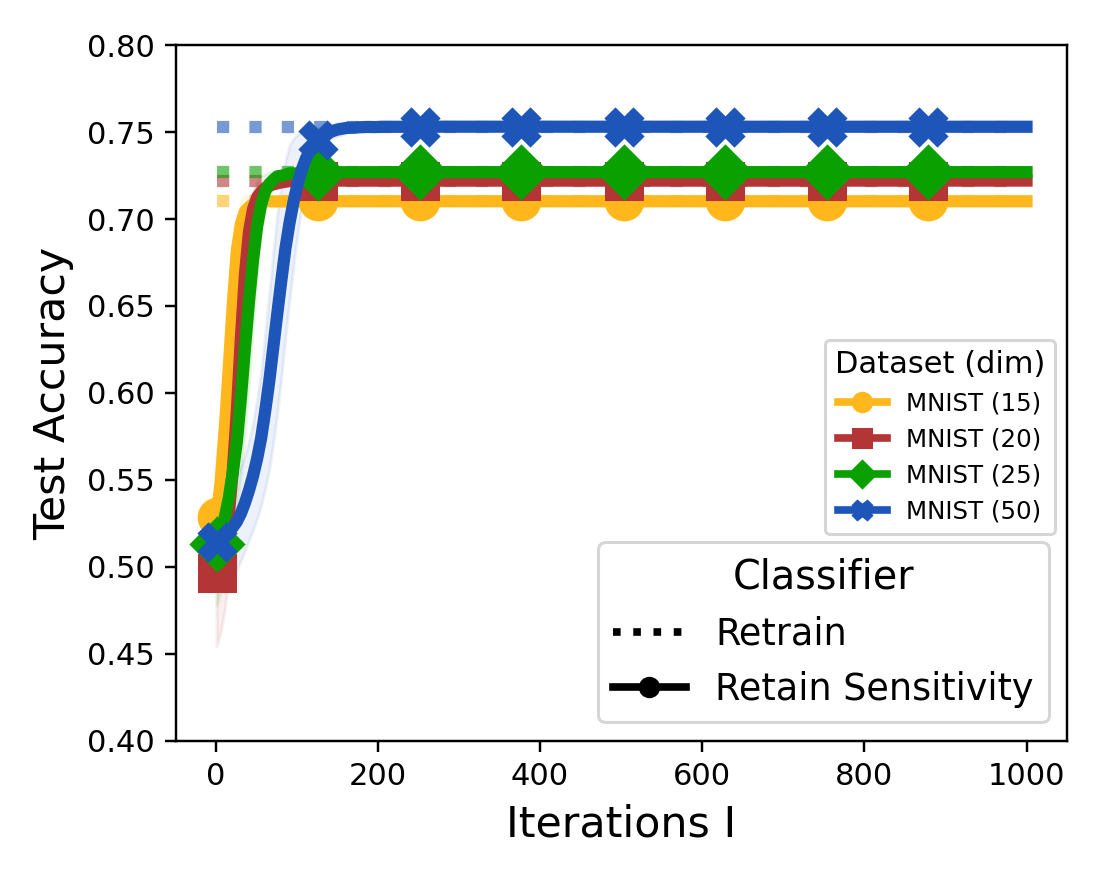}\\[0.5em]

        \rotatebox{90}{\hspace{.6cm}\scriptsize $\lambda=0.001$} &
        \includegraphics[width=0.40\linewidth]{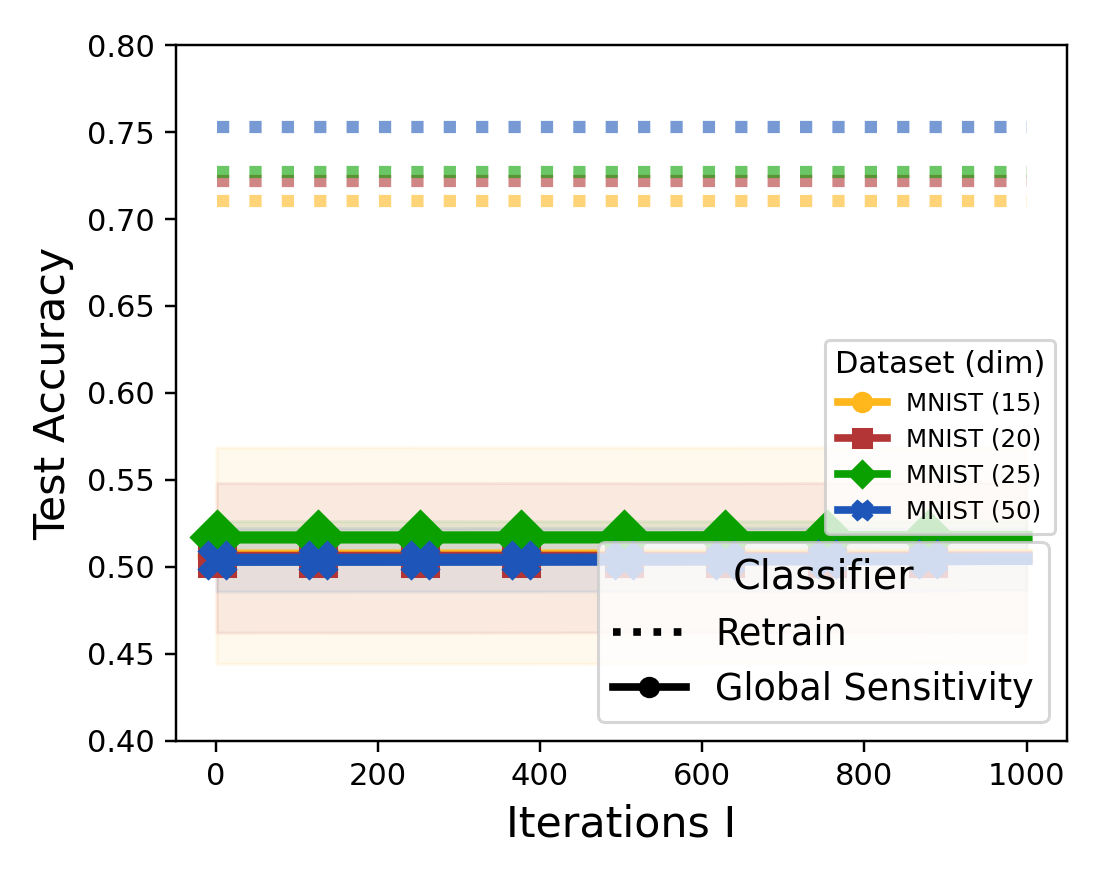} &
        \includegraphics[width=0.40\linewidth]{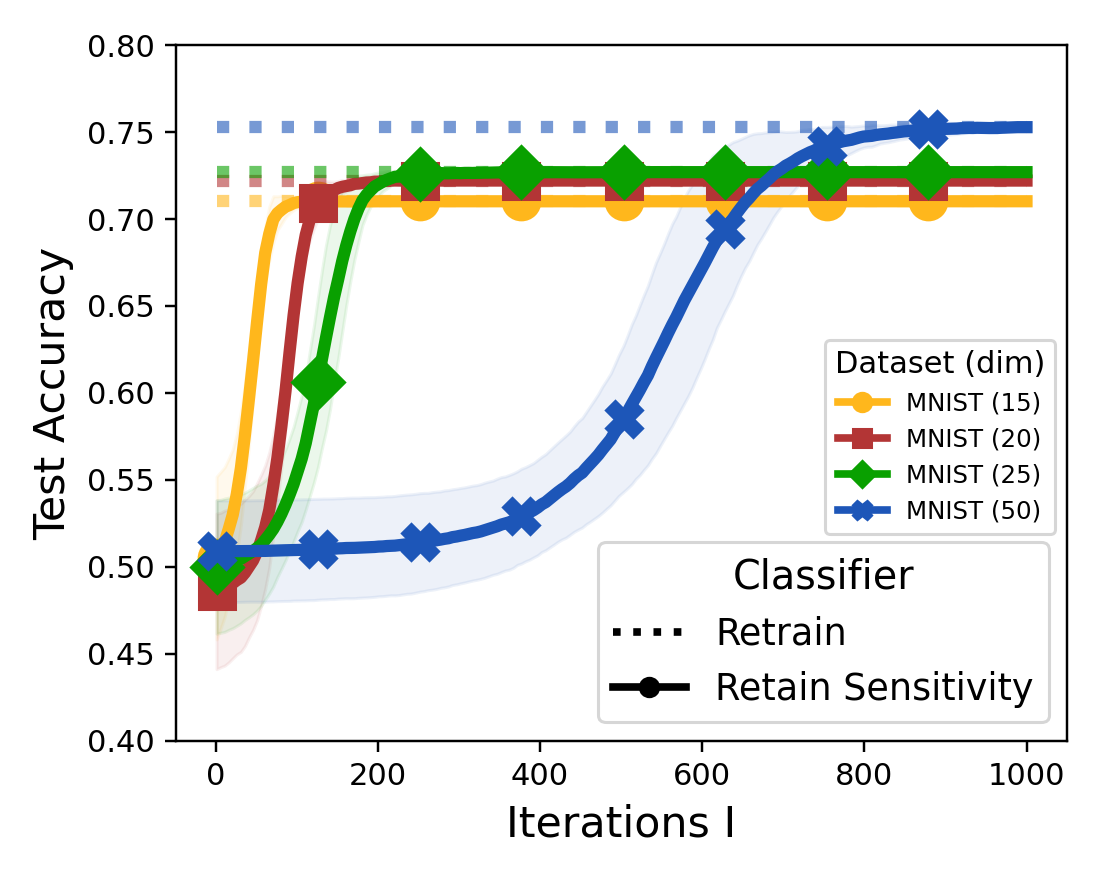}\\

        \rotatebox{90}{\hspace{.6cm}\scriptsize $\lambda=0.0001$} &
        \includegraphics[width=0.40\linewidth]{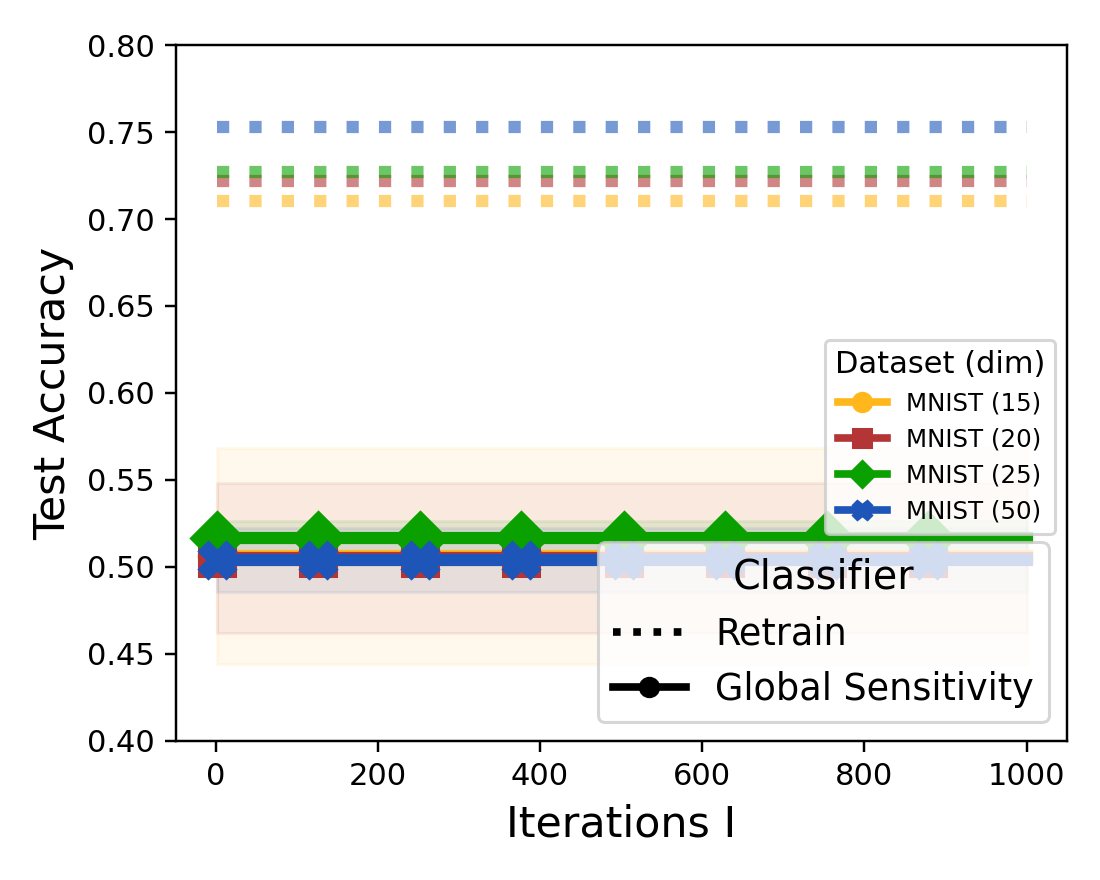} &
        \includegraphics[width=0.40\linewidth]{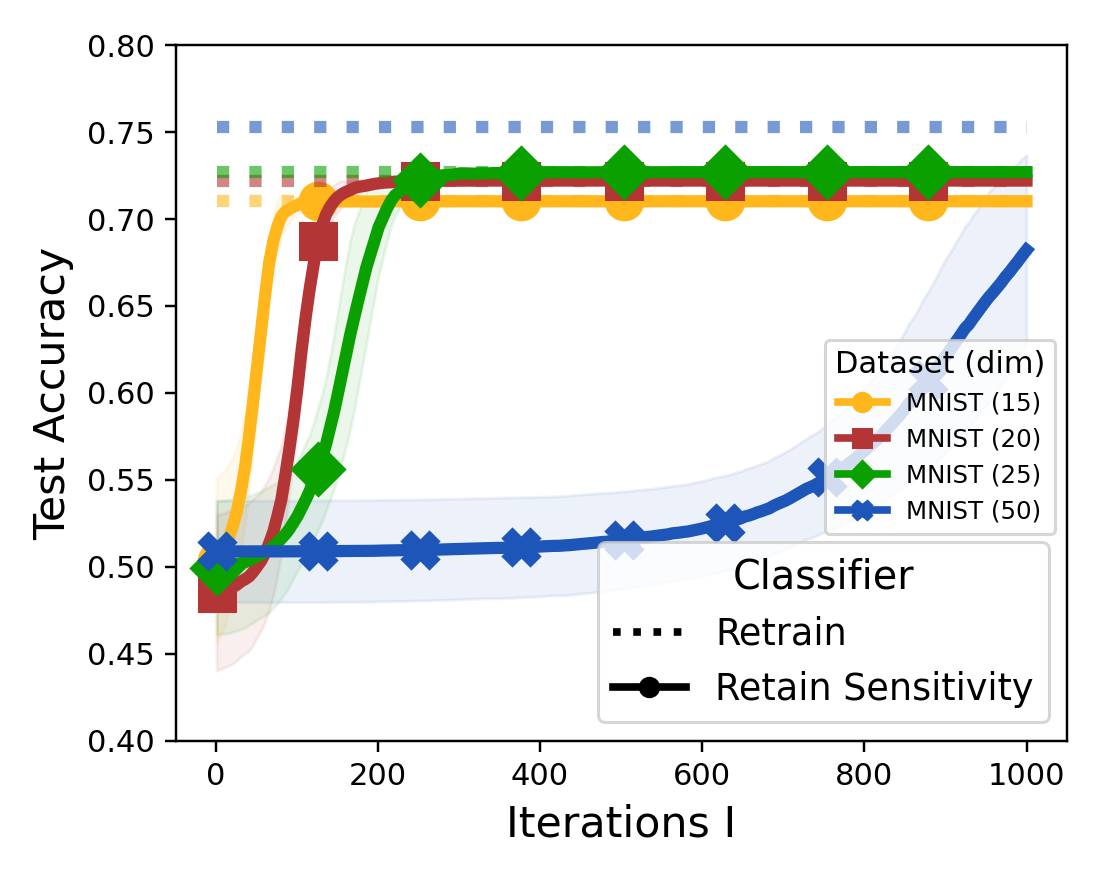}\\

        \rotatebox{90}{\hspace{.6cm}\scriptsize $\lambda=0.00001$} &
        \includegraphics[width=0.40\linewidth]{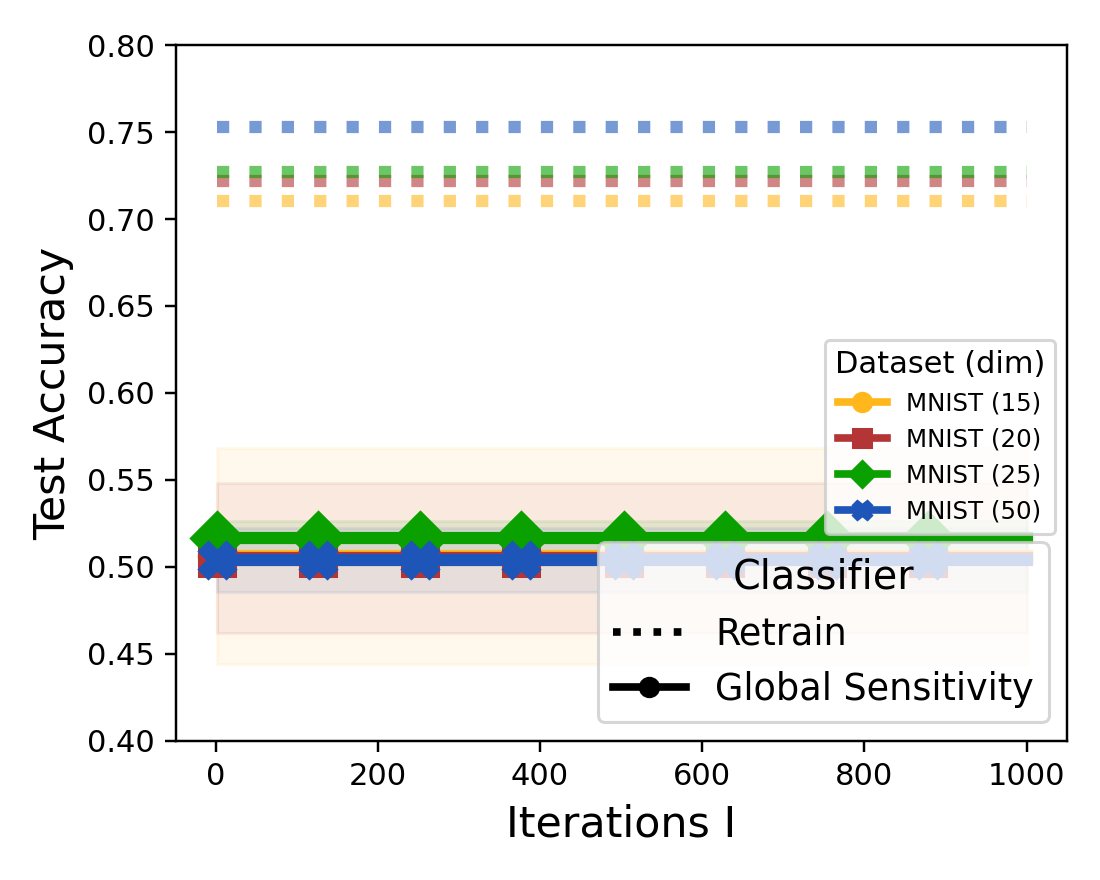} &
        \includegraphics[width=0.40\linewidth]{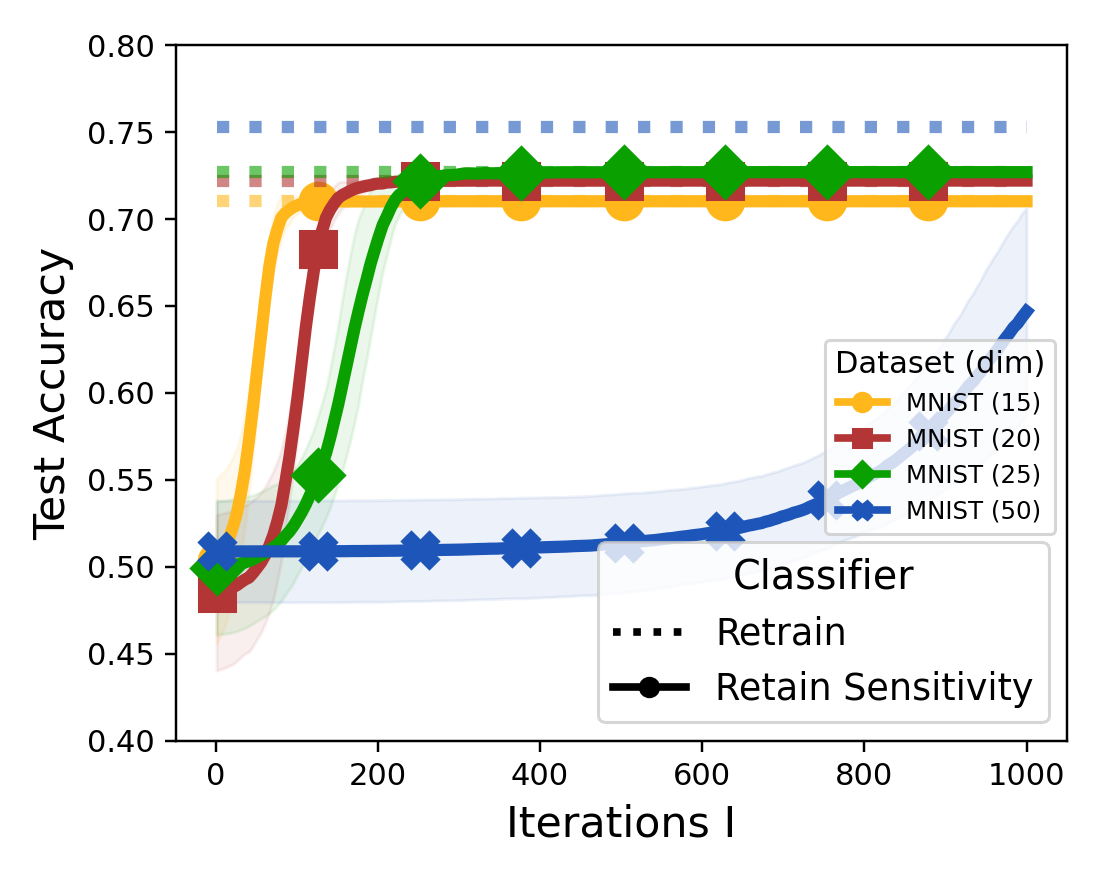}\\
        
    \end{tabular}
    \caption{\textbf{D2D (MSE loss) utility:} Test accuracy when noise is calibrated to global sensitivity (GS) versus retain sensitivity (RS) for MSE loss }
    \label{fig:d2d-mse-utility}
\end{subfigure}

\end{tabular}

\caption{\textbf{D2D utility comparison:} Left block: LogLoss. Right block: MSE. Within each block, columns correspond to noise calibrated to GS and RS for D2D \Cref{alg:d2d}, and rows correspond to $\lambda \in \{0.1, 0.01, 0.001, 0.0001, 0.00001\}$. We consider a fixed $(\varepsilon=1,\delta=10^{-9})$ certified unlearning guarantee, as a function of the iteration budget, for different values of $\lambda$. The baseline is exact retraining. Experiments are conducted on MNIST with random projections to dimensions $\{15,20,25,50\}$. The plots for log loss with $\lambda=0.0001$ are the same as in \Cref{fig:active_d2d_accuracy}.}
\label{fig:d2d-utility-logloss-mse}

\end{figure}

\begin{figure}[ht]
\begin{subfigure}{0.32\textwidth}
    \includegraphics[width=\linewidth]{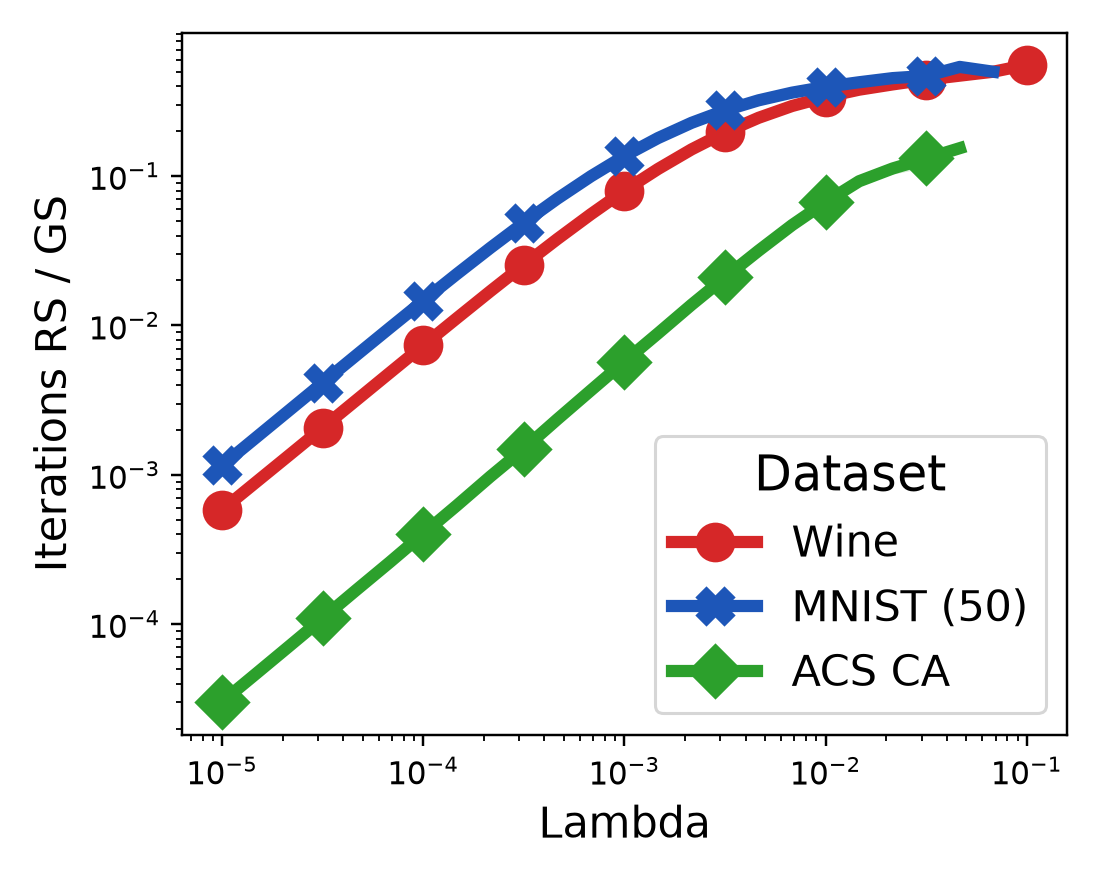}
    \caption{D2D MSE iterations}
    \label{fig:active_mse_d2d}
\end{subfigure}
\begin{subfigure}{0.32\textwidth}
    \includegraphics[width=\linewidth]{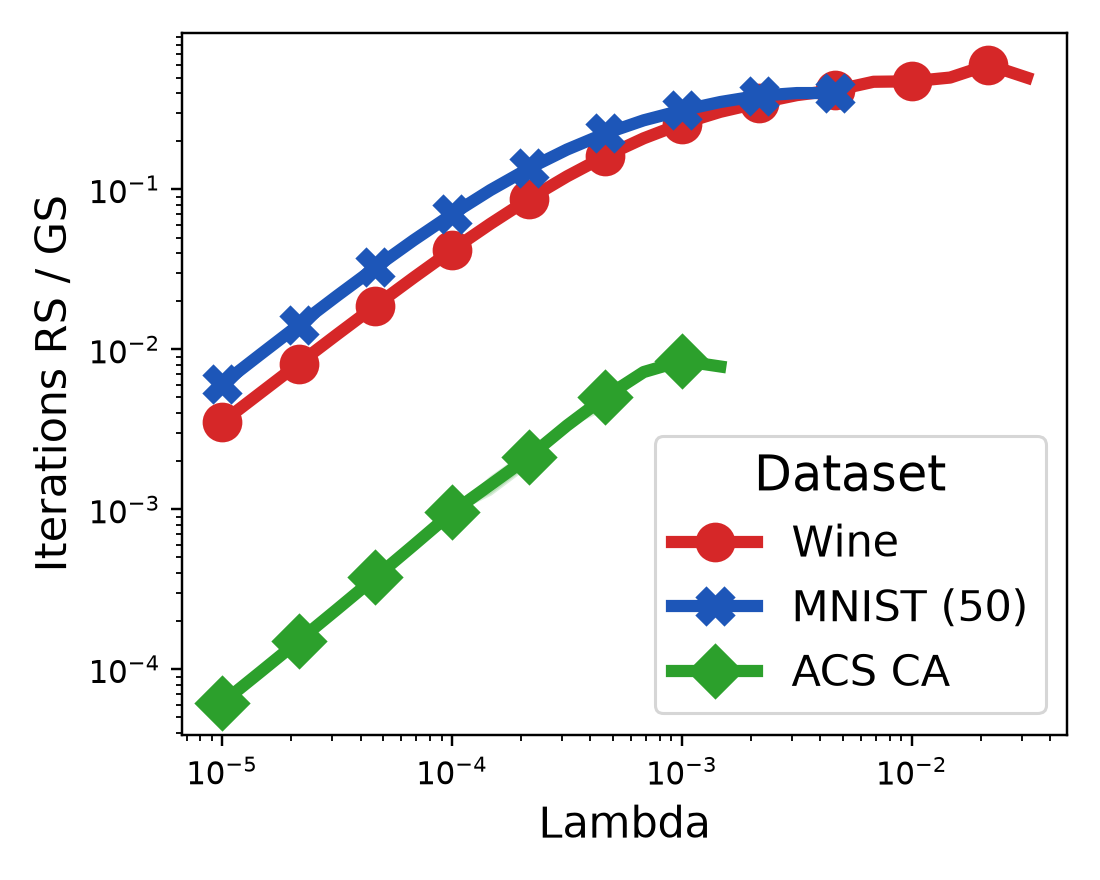}
    \caption{D2D LogLoss iterations}
    \label{fig:active_logloss_d2d}
\end{subfigure}
\begin{subfigure}{0.32\textwidth}
    \includegraphics[width=\linewidth]{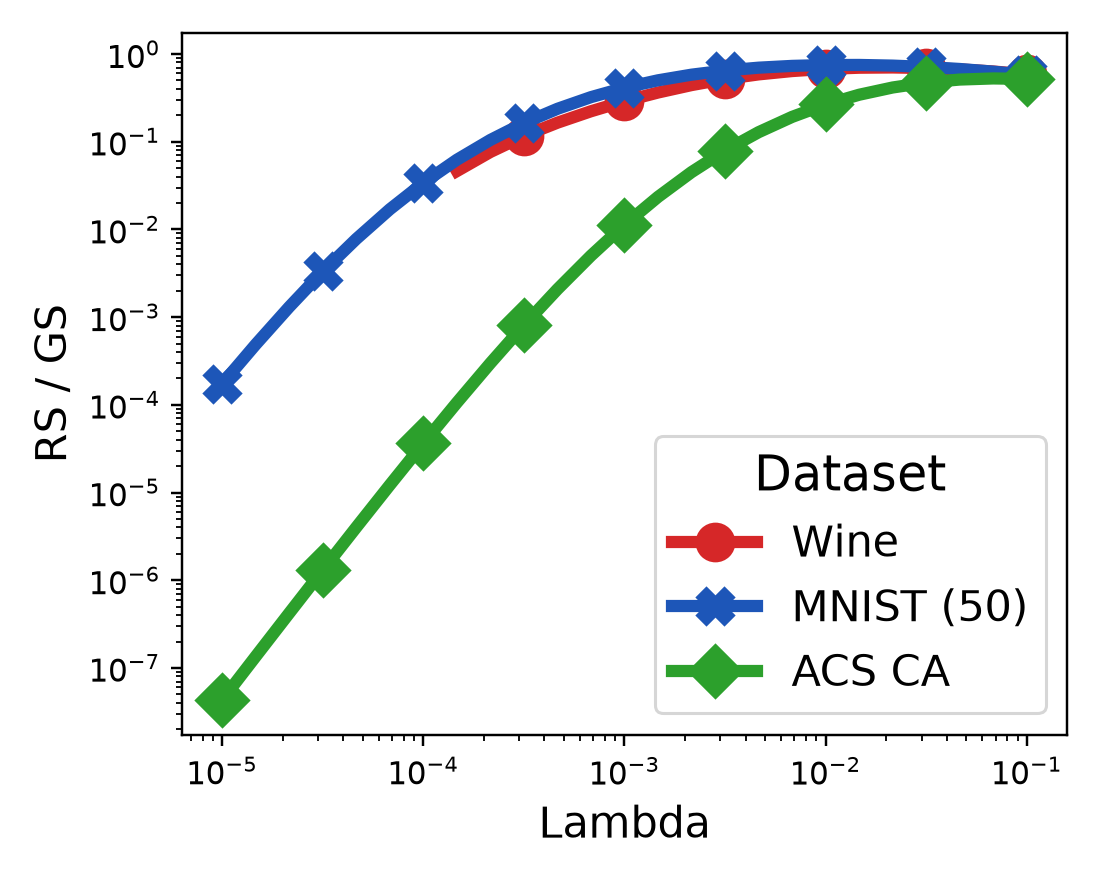}
    \caption{Newton LogLoss noise level}
    \label{fig:active_newton}
\end{subfigure}
\caption{\textbf{Retain vs. global sensitivity (Active), noise:} ~\Cref{fig:active_mse_d2d,fig:active_logloss_d2d} show the
        ratio of iteration counts \(\smash{\nicefrac{I_R}{I}}\) in
        Descent-to-Delete (D2D) to guarantee
        \(\br{\varepsilon,\delta}\)-unlearning for fixed
        \(\br{\varepsilon=1, \delta=10^{-9}, \sigma=0.1}\).~\Cref{fig:active_newton}
        shows the ratio
        \(\smash{\RS\br{R}/\Delta_{\GS}}\)
        for the Newton-step update. In all cases, smaller is better,
        and the gap is largest for small regularisation \(\lambda\). For details on experiments see \Cref{app:experiments}.}
\label{fig:active_noise}
\end{figure}

\subsection{Newton Unlearning}

\begin{algorithm}[tb]
  \caption{Unlearning via Newton Step Update (\cite{sekhari2021RememberWhatYou})}
  \begin{algorithmic}[1]\label{alg:newton}
    \item[] \textbf{Input:} Trained model: $\mathcal{A}(R\cup U=R')=w_{R'}$, deletion request: $U=\bc{z_{n+1},\dots, z_{n+m}}$, side information: $T(R')=\{\nabla^2\hat{F}_{R'}(w_{R'})\}\cup R$
    \STATE Compute $\lambda_R=\inf_{w\in\mathcal{W}}\lambda_{\min}(\nabla^2\hat{F}_R(w))$
    \STATE Set noise: $\sigma= c_{\varepsilon,\delta}\cdot \frac{m^2ML^2}{2\lambda_R^3n^2}$ ($c_{\varepsilon,\delta}$ from \Cref{fact:gaussian_mech})
    \STATE Recover Hessian on $R$: $\hat{H}\leftarrow\frac{1}{n}((n+m)\nabla^2\hat{F}_{R'}(w_{R'})-\sum_{i=1}^m\nabla^2f(w_{R'},z_{n+i}))$
    \STATE Define $\bar{w}= w_{R'}+ \frac{1}{n}\hat{H}^{-1}\sum_{i=1}^m\nabla f(w_{R'}, z_{n+i})$
    \STATE Draw $\nu\sim\mathcal{N}(0,\sigma^2 I_d)$
    \item[] \textbf{Output: }$\tilde{w} = \bar w+ \nu$.
  \end{algorithmic}
  
\end{algorithm}

See \Cref{alg:newton} for details on the Newton-step update.
\begin{restatable}{lemma}{sekharigeneral}\label{lem:rwywtf}
Let $w_R$ and $w_{R'}$ be the empirical risk minimiser over dataset $R$ and $R'=R\cup \bc{z_{n+1},\dots, z_{n+m}}$ respectively. Then for the approximation step $\bar{w}= w_{R'}+ \frac{1}{n}\hat{H}^{-1}\sum_{i=1}^m\nabla f(w_{R'}, z_{n+i})$  
where $\hat{H}=\frac{n+m}{n}\nabla^2\hat{F}_{R'}(w_{R'})-\frac{1}{n}\sum_{i=1}^m\nabla^2f(w_{R'},z_{n+i})=\nabla^2\hat{F}_R(w_{R'})$.\\
Define $\lambda_R=\inf_{w\in W}\lambda_{\min}(\nabla_w^2\hat{F}_R(w))$, and if $\lambda_R>0$ then 
    \begin{equation*}
        \lVert w_R - \bar{w} \rVert\leq \frac{m^2L^2M}{2n^2\lambda_R^3}.
    \end{equation*}
\end{restatable}

\begin{proof}
    Following the same analysis as \cite{sekhari2021RememberWhatYou}, we have by a Taylor expansion around $w_{R'}$
    \begin{equation*}
        \lVert \underbrace{\nabla\hat{F}_R(w_{R})}_{=0}-\nabla\hat{F}_R(w_{R'})-\nabla^2\hat{F}_R(w_{R'})(w_R-w_{R'})\rVert\leq \frac{M}{2}\lVert w_{R'}-w_R\rVert^2
    \end{equation*}
    where we use that $\nabla\hat{F}_R(w_{R})=0$ as $w_R$ is the minimiser of the smooth $\hat{F}_R$.
    We can simplify the lefthand side to $\lVert -\frac{1}{n}\sum_{i=1}^m\nabla f(w_{R'},z_{n+i})+\nabla^2\hat{F}_R(w_{R'})(w_R-w_{R'})\rVert$ as  $\nabla\hat{F}_R(w_{R'})=\frac{n+m}{n}\nabla\hat{F}_{R'}(w_{R'})-\frac{1}{n}\sum_{i=1}^m\nabla f(w_{R'},z_{n+i})=0-\frac{1}{n}\sum_{i=1}^m\nabla f(w_{R'},z_{n+i})$ because  $w_{R'}$ is the minimiser of the smooth $\hat{F}_{R'}$.\\
    Set $v=w_R-(w_{R'}+ \frac{1}{n}(\nabla^2\hat{F}_R(w_{R'}))^{-1}\sum_{i=1}^m\nabla f(w_{R'}, z_{n+i}))$ which implies
    \begin{align*}
        \lambda_R\lVert v\rVert&\leq \lambda_{\min}(\nabla^2\hat{F}_R(w_{R'}))\lVert v\rVert \leq \lVert \nabla^2\hat{F}_R(w_{R'})v\rVert \\
        &= \lVert -\frac{1}{n}\sum_{i=1}^m\nabla f(w_{R'},z_{n+i})+\nabla^2\hat{F}_R(w_{R'})(w_R-w_{R'})\rVert\\
        &\leq \frac{M}{2}\lVert w_{R'}-w_R\rVert^2.
    \end{align*}
    Thus dividing by $\lambda_R$ gives an upper bound $\lVert v\rVert\leq \frac{M}{2\lambda_
    R}\lVert w_{R'}-w_R\rVert^2$. Now the final claim follows by a direct application of \Cref{lem:erm_general} which gives $\lVert v\rVert\leq \frac{M}{2\lambda_
    R}(\frac{mL}{\lambda_Rn})^2=\frac{m^2ML^2}{2\lambda_R^3n^2}$. 
\end{proof}

\begin{lemma}
    Let $\lambda_R=\inf_{w\in W}\lambda_{\min}(\nabla_w^2\hat{F}_R(w))>0$.
    The unlearning algorithm from \cite{sekhari2021RememberWhatYou} with full information $T(R)=R$, which adds noise scaling with $\Delta_{\RS}^{(m)}=\frac{m^2L^2M}{2n^2\lambda_R^3}$ is $(\varepsilon,\delta)$-unlearning.
\end{lemma}

\begin{proof}
The unlearning algorithm from \cite{sekhari2021RememberWhatYou} performs first a Newton step update and then adds calibrated noise, which matches \Cref{def:approx-based} of an active learning--unlearning algorithm $(\mathcal{A},\bar{\mathcal{A}})$. To be precise, we have for $Z=\bc{z_{n+1},\dots, z_{n+m}}$,
\begin{align*}
    \bar{\mathcal{A}}_0(Z,\mathcal{A}(R\cup Z),R\cup Z)&= w_{R\cup Z}+ \frac{1}{n}(\nabla^2\hat{F}_R(w_{R\cup Z}))^{-1}\sum_{i=1}^m\nabla f(w_{R\cup Z}, z_{n+i})\\
\text{ and }\bar{\mathcal{A}}_0(\emptyset, \mathcal{A}(R),R)&=w_R.
\end{align*}
By \Cref{lem:rwywtf} we have 
\begin{equation*}
    \lVert\bar{\mathcal{A}}_0(Z,\mathcal{A}(R\cup Z),R\cup Z)-\bar{\mathcal{A}}_0(\emptyset, \mathcal{A}(R),R) \rVert \leq\frac{m^2L^2M}{2n^2\lambda_R^3}
\end{equation*}
which is independent of $Z$ and thus an upper bound on the retain sensitivity. Setting $\Delta_{\RS}^{(m)}=\frac{m^2L^2M}{2n^2\lambda_R^3}$ and sampling noise from $\mathcal{N}(0,\sigma^2I_d)$ with $\sigma=\Delta_{\RS}^{(m)}\cdot c_{\varepsilon,\delta}$  with $c_{\varepsilon,\delta}$ the standard Gaussian mechanism noise multiplier \Cref{fact:gaussian_mech}, then provides the unlearning guarantee by \Cref{thm:retain_unlearn}.
\end{proof}

We additionally include the plots depicting the noise levels in \Cref{fig:active_newton} from \Cref{lem:newton-rs}.  A utility evaluation for MSE loss for the Newton method is obsolete, as the Newton method is exact for MSE loss. The same behaviour as in the passive case \Cref{subfig:passive_erm_logloss} can be observed, amplified to the power of three.

\begin{remark}[Enforcing the parameter bound]
For D2D, we enforce $\|w\|_2\le R=1$ by solving the full-data and retraining ERMs with projected gradient and projecting after every update. For Newton, post-hoc projection generally moves the unconstrained ERM off stationarity ($\nabla\widehat F_D(w)\ne0$), invalidating the Newton-step analysis. \\
Instead, for the Newton utility experiments, we use a global parameter bound for GS. For RS, we use a retain-set-dependent model-norm bound, based on the retrained model and a certified neighborhood containing all neighboring ERMs, see \Cref{ex:retainset_dep_R_w}.
\end{remark}

\begin{example}[Retain-set-dependent parameter bound for logistic ERM]\label{ex:retainset_dep_R_w}
Consider regularised logistic loss, and define
\[
q_R=\frac{m(B+\lambda\|w_R\|_2)}{n+m},
\qquad
\mu_R(r)=\lambda_{\min}\!\left(
\lambda I+\frac{1}{n+m}\sum_{i\in R}
\frac{x_ix_i^\top}{
4\cosh^2\!\left((|x_i^\top w_R|+\|x_i\|_2r)/2\right)}
\right).
\]
If \(q_R<r\mu_R(r)\), then every neighboring ERM \(w_{R\cup Z}\), with \(|Z|\le m\), satisfies \(\|w_{R\cup Z}\|_2\le R_w^{\mathrm{RS}}(R):=\|w_R\|_2+r\). This bound depends on the retain set through \(w_R\) and its margins. By comparison, the global regularization bound is \(R_w^{\mathrm{GS}}=\sqrt{2\log(2)/\lambda}\) (regularisation gives $\frac{\lambda}{2}\|w_D\|_2^2\le\widehat F_D(w_D)\le\widehat F_D(0)=\log 2$, hence $R_w=\sqrt{2\log(2)/\lambda}$). In our implementation, we try \(r=2q_R/\mu_R(0)\) and use it only if \(q_R<r\mu_R(r)\).
\end{example}

%% file: appendix/dataset_description.tex
\section{Appendix: Experimental Setup}\label{app:experiments}

This section summarises the exact experimental protocol used to produce our empirical comparisons (plots in \Cref{fig:RS_vs_GS_passive,fig:RS_vs_GS_active} of the main paper, as well as \Cref{fig:d2d-utility-logloss-mse,fig:active_noise} in the Appendix). The code can be found in \href{https://github.com/caroheinz/retain-sensitivity}{this github repository}.
We report the ratio $\RS/\GS$ across datasets, losses, and regularisation strengths, where $\RS$ denotes retain sensitivity and $\GS$ the corresponding global sensitivity bound. Experiments were performed locally on an Apple MacBook Pro with M1 chip.

\paragraph{MST}
We evaluate the MST objective on four real-world weighted networks: \textit{physics\_collab}\footnote{\url{https://networks.skewed.de/net/physics_collab}}, \textit{openflights}\footnote{\url{https://networks.skewed.de/net/openflights}}, 
\textit{at\_migrations}\footnote{\url{https://networks.skewed.de/net/at_migrations}}, \textit{soc-sign-bitcoin-otc}\footnote{\url{https://snap.stanford.edu/data/soc-sign-bitcoin-otc.html}}.

From each source network, we construct a collection of induced subgraphs by: (1) selecting a random start node; (2) performing BFS until reaching $100$ nodes; (3) taking the induced subgraph on these nodes; and (4) discarding and resampling unless the edge density is at least $0.1$. We repeat steps (1) to (4) to obtain $500$ random subgraphs per network.

On each subgraph, we compute the MST objective value (using the dataset-provided edge weights) and evaluate sensitivity under single-edge updates consistent with our adjacency model for MST. We aggregate results over the $500$ sampled subgraphs to report typical retain-dependent versus worst-case behaviour on real weighted graphs. In \Cref{tab:summary_stats} and \Cref{tab:density_distribution} we list the weight distribution of the edges of the full graphs of the datasets, as well as the density distribution among the sampled subgraphs. 

\paragraph{SVM}
We evaluate the ratio between retain and global sensitivity for SVMs on the same datasets and experimental set up as the ERM experiments (see below). For this, we calculate empirical and true margin of the data. In practice, although the true margin is unknown, domain knowledge can often justify a conservative strictly positive lower bound, or one can use a plug-in estimator based on the available data (or a held-out public dataset). In our experiments, we use the empirical margin on the full dataset as a proxy for the population margin and compare it with the empirical margins obtained from sampled fractions of the data.

\paragraph{ERM Passive and Active Unlearning}
We evaluate on three standard datasets.
\begin{itemize}

    \item Wine Quality dataset \cite{Cortez2009ModelingWP} from UCI Machine Learning Repository. We filter for red wines and take as a target quality at least 6. 
    \item ACSIncome task from Folktables \cite{ding2021retiring}.
We download the data from the corresponding github repository \cite{folktables}.
We restrict to California (2018), standardise all features, and use the default ACSIncome target.

    \item MNIST \cite{lecun2002gradient}, accessed via torchvision. We use digits $\{0,1,2,3,4\}$ vs.\ $\{5,6,7,8,9\}$, standardise features, and apply a Gaussian random projection to $d=50$.
\end{itemize}

For each dataset we fix a sample set size of $n =10.000$ (with wine quality dataset restricted to its max of 1599), we repeat all experiments over random seeds in $\{42,43,44,45,46\}$. Error bars represent one standard deviation of the data. 
We fix the boundedness parameters used in our sensitivity instantiations to $B=1$ and $R=1$. For every dataset, we append a constant intercept feature and jointly rescale the augmented feature vectors to satisfy the $\ell_2$ norm bound $B$.

We run least-squares estimation (LSE) and logistic regression objectives.
We sweep $\lambda\in[10^{-5},10^{-1}]$ to span weak to strong regularisation.
For readability, markers are shown at regular intervals; zero and NaN values are omitted on the logarithmic scale.
In \Cref{tab:eigenvalues} we list the corresponding smallest eigenvalues of the data, which are the relevant quantities for our plots. 

\textbf{Passive evaluation: } For each configuration $(\text{dataset},\text{seed}, \text{loss}, \lambda)$, we compute the empirical retain sensitivity $\RS$ under single-point additions (per our definition) and the corresponding global sensitivity bound $\GS$ for the same objective/parameter setting, and report the ratio $\RS/\GS$.

\textbf{Active evaluation (noise ratios):} We repeat the same dataset, seed, loss, and $\lambda$ sweeps in the active setting, fixing the unlearning/privacy parameters to $\delta = 10^{-9},\varepsilon = 1, \sigma = 0.1.$ (All other settings, including preprocessing and $B=1$, $R=1$, match the passive experiments.)

For the global-sensitivity calibration, we use the regularisation-dependent parameter bound $R(\lambda)=\sqrt{2\log(2)/\lambda}$; for retain sensitivity, we use a data-dependent local certificate based on the retain set.

\textbf{Active evaluation (utility):} For utility of Newton step update we evaluate $\ell_2$-regularised logistic regression on the MNIST dataset, projected to lower dimensions $(d\in \{15,20,25,50\})$ via Gaussian Random Projection. 
We standardise the features and clip them to the unit $\ell_2$ ball.
We report test accuracy with error bounds over random seeds across varying regularisation strengths, we consider $\lambda \in \mathrm{logspace}(-5, -1, 25)$.\\
For utility of D2D, we initialize at the full-data ERM constrained to the $\ell_2$-ball of radius $R$ and project each iterate onto this ball. We consider the same dataset as above, with iterations $I\in [1,1000]$ at granularity $5$ iterations. We produce separate plots for 5 values of $\lambda\in\mathrm{logspace}(-5, -1, 5)$.

\begin{remark}
For MNIST, we use fixed random Gaussian projections as part of the learning pipeline. This is important because the gain from retain sensitivity for MSE and logistic regression is driven by the data-dependent curvature term \(\lambda_R\), which can be lower bounded by the smallest eigenvalue of the empirical covariance. In the original high-dimensional representation, this eigenvalue is close to zero, so the benefit of retain calibration is largely hidden. Random projections reduce the effective dimension and improve this curvature, making the retain-sensitivity advantage visible. The resulting eigenvalues are reported in \Cref{app:experiments}.
\end{remark}

\begin{table}[h]
\centering
\caption{Smallest eigenvalue of normalized covariance ($\frac{1}{n}X^T X$) per Dataset}
\label{tab:eigenvalues}
\begin{tabular}{l H r r}
\hline
\textbf{Dataset} & \textbf{Avg. Min. Eigenvalue} & \textbf{Full Dataset Size} & \textbf{Min. Eigenvalue}\\ \hline
mnist\_rpe15     & 251.1209                     & 10,000                & 0.0251\\
mnist\_rpe20     & 173.3016                     & 10,000                & 0.0173\\
mnist\_rpe25     & 126.2372                     & 10,000                & 0.0126\\
mnist\_rpe50     & 2.0656                       & 10,000                & 0.0002\\
folktables       & 47.5940                      & 10,000                & 0.0048\\
wine             & 1.1276                       & 1,599                 & 0.0007\\ \hline
\end{tabular}
\end{table}

\begin{table}[h]
\centering
\caption{Edge Weight Distribution Across MST Datasets}
\label{tab:summary_stats}
\begin{tabular}{l r r r r r r r}
\hline
\textbf{Dataset} & \textbf{Mean} & \textbf{Min} & \textbf{25\%} & \textbf{50\%} & \textbf{75\%} & \textbf{95\%} & \textbf{Max} \\ \hline
physics collab         & 0.1216        & 0.0132       & 0.0132        & 0.0185        & 0.1111        & 0.5000        & 10.8333      \\
soc-sign-bitcoin-otc            & 12.0120       & 1       & 12       & 12       & 13       & 16       & 21      \\
at\_migrations               & 17.3591       & 1       & 1        & 2        & 6        & 40       & 28,010  \\
openflights               & 1.8092        & 1       & 1        & 1        & 2        & 4        & 20      \\ \hline
\end{tabular}
\end{table}

\begin{table}[h]
\centering
\caption{Edge Density Distribution Among Sampled Subgraphs}
\label{tab:density_distribution}
\begin{tabular}{l r r r r r r}
\hline
\textbf{Dataset} & \textbf{Mean} & \textbf{Min} & \textbf{25\%} & \textbf{50\%} & \textbf{75\%} & \textbf{Max} \\ \hline
physics collab & 0.3517 & 0.1287 & 0.1784 & 0.3447 & 0.5548 & 0.6648 \\
soc-sign-bitcoin-otc             & 0.1191        & 0.1002       & 0.1067        & 0.1129        & 0.1271        & 0.1846       \\ 
at\_migrations               & 0.7646        & 0.5180       & 0.7283        & 0.7784        & 0.8109        & 0.9123       \\
openflights               & 0.0955        & 0.0437       & 0.0801        & 0.0926        & 0.1163        & 0.1359       \\
\hline
\end{tabular}
\end{table}

%% file: references_new.bib
@string{NEURIPS = {{Neural Information Processing Systems~(NeurIPS)}}}

@string{FOCS = {{Foundations of Computer Science~(FOCS)}}}

@string{CVPR = {{Computer Vision and Pattern Recognition~(CVPR)}}}

@string{ICML = {{International Conference on Machine Learning~(ICML)}}}

@string{ICLR = {{International Conference on Learning Representations~(ICLR)}}}

@string{JMLR = {{Journal of Machine Learning Research}}}

@string{ALT = {{Conference on Algorithmic Learning Theory~(ALT)}}}

@string{STOC = {{Symposium on Theory of Computing~(STOC)}}}

@string{TCC = {{Theory of Cryptography~(TCC)}}}

@string{AISTATS = {{International Conference on Artificial Intelligence and Statistics~(AISTATS)}}}

@string{AAAI ={{Association for the Advancement of Artificial Intelligence Conference on Artificial Intelligence~(AAAI)}} }

@string{DMLR ={{Data-Centric Machine Learning Research~(DMLR)}} }

@string{IEEESP ={{IEEE Symposium on Security and Privacy~(IEEE S\&P)}} }

@string{IEEETIFS ={{IEEE Transactions on Information Forensics and Security~(IEEE TIFS)}}}

@string{FORC = {{Symposium on Foundations of Responsible Computing~(FORC)}}}

@string{SIGSAC = {{ACM Conference on Computer and Communications Security~(SIGSAC)}}}

@article{dwork2014AlgorithmicFoundationsDifferential,
author = {Dwork, Cynthia and Roth, Aaron},
title = {The Algorithmic Foundations of Differential Privacy},
year = {2014},
journal = {Foundations and Trends in Theoretical Computer Science},
}

@inproceedings{koh2020UnderstandingBlackboxPredictions,
	title = {Understanding {Black}-box {Predictions} via {Influence} {Functions}},
	booktitle = ICML,
	author = {Koh, Pang Wei and Liang, Percy},
	year = {2017}
}

@inproceedings{sekhari2021RememberWhatYou,
	title = {Remember {What} {You} {Want} to {Forget}: {Algorithms} for {Machine} {Unlearning}},
	booktitle = NEURIPS ,
	author = {Sekhari, Ayush and Acharya, Jayadev and Kamath, Gautam and Suresh, Ananda Theertha},
	year = {2021},
}

@inproceedings{nissim2007smooth,
  title={Smooth sensitivity and sampling in private data analysis},
  author={Nissim, Kobbi and Raskhodnikova, Sofya and Smith, Adam},
  booktitle=STOC,
  year={2007}
}

@inproceedings{bertran_reconstruction_2024,
    author = {Bertran, Martin and Tang, Shuai and Kearns, Michael and Morgenstern, Jamie and Roth, Aaron and Wu, Zhiwei S.},
    booktitle = NEURIPS,
    title = {Reconstruction {Attacks} on {Machine} {Unlearning}: {Simple} {Models} are {Vulnerable}},
    year = {2024}
}

@article{hu2025online,
      title={Online Learning and Unlearning}, 
      author={Yaxi Hu and Bernhard Schölkopf and Amartya Sanyal},
      year={2025},
      journal={arXiv preprint arxiv:2505.08557},
}

@inproceedings{chen2021machine,
  title={When machine unlearning jeopardizes privacy},
  author={Chen, Min and Zhang, Zhikun and Wang, Tianhao and Backes, Michael and Humbert, Mathias and Zhang, Yang},
  booktitle=SIGSAC,
  pages={896--911},
  year={2021}
}

@inproceedings{dwork2006calibrating,
  title={Calibrating noise to sensitivity in private data analysis},
  author={Dwork, Cynthia and McSherry, Frank and Nissim, Kobbi and Smith, Adam},
  booktitle=TCC,
  year={2006}
}

@article{huang2023tight,
  author = {Yiyang Huang and Clement Canonne},
  journal = {Journal of Privacy and Confidentiality},
  number = {2},
  publisher = {Society for Privacy and Confidentiality Research},
  title = {Tight Bounds for Machine Unlearning via Differential Privacy},
  volume = {15},
  year = {2025}
}

@inproceedings{dwork2014gauss,
author = {Dwork, Cynthia and Talwar, Kunal and Thakurta, Abhradeep and Zhang, Li},
title = {Analyze gauss: optimal bounds for privacy-preserving principal component analysis},
year = {2014},
booktitle = STOC
}

@InProceedings{neel2021descent,
  title = 	 {Descent-to-Delete: Gradient-Based Methods for Machine Unlearning},
  author =       {Neel, Seth and Roth, Aaron and Sharifi-Malvajerdi, Saeed},
  booktitle = 	 ALT,
  year = 	 {2021}
}

@article{davis1970rotation,
  title={The rotation of eigenvectors by a perturbation. III},
  author={Davis, Chandler and Kahan, William Morton},
  journal={SIAM Journal on Numerical Analysis},
  volume={7},
  number={1},
  pages={1--46},
  year={1970},
  publisher={SIAM}
}

@inproceedings{dungler2025iterative,
  title={An Iterative Algorithm for Differentially Private $ k $-PCA with Adaptive Noise},
  author={D{\"u}ngler, Johanna and Sanyal, Amartya},
  booktitle=NEURIPS,
  year={2025}
}

@inproceedings{bun2016concentrated,
  title={Concentrated differential privacy: Simplifications, extensions, and lower bounds},
  author={Bun, Mark and Steinke, Thomas},
  booktitle=TCC,
  year={2016}
}

@inproceedings{cao2015towards,
  title={Towards making systems forget with machine unlearning},
  author={Cao, Yinzhi and Yang, Junfeng},
  booktitle =IEEESP,
  year={2015}
}

@book{nesterov2004introductory,
  title        = {Introductory Lectures on Convex Optimization: A Basic Course},
  author       = {Nesterov, Yurii},
  year         = {2004},
  publisher    = {Kluwer Academic Publishers},
  isbn         = {978-1-4020-7553-7}
}

@inproceedings{mu2024rewind,
  title={Rewind-to-delete: Certified machine unlearning for nonconvex functions},
  author={Mu, Siqiao and Klabjan, Diego},
  booktitle=NEURIPS,
  year={2025}
}

@inproceedings{zhang2024towards,
  title={Towards certified unlearning for deep neural networks},
  author={Zhang, Binchi and Dong, Yushun and Wang, Tianhao and Li, Jundong},
  booktitle=ICML,
  year={2024}
}

@inproceedings{chien2024langevin,
  title={Langevin unlearning: A new perspective of noisy gradient descent for machine unlearning},
  author={Chien, Eli and Wang, Haoyu and Chen, Ziang and Li, Pan},
  booktitle= NEURIPS ,
  year={2024}
}

@inproceedings{guo2019certified,
  title={Certified data removal from machine learning models},
  author={Guo, Chuan and Goldstein, Tom and Hannun, Awni and Van Der Maaten, Laurens},
  booktitle=ICML,
  year={2020}
}

@inproceedings{ginart2019making,
  title={Making {AI} forget you: Data deletion in machine learning},
  author={Ginart, Antonio and Guan, Melody and Valiant, Gregory and Zou, James Y},
  booktitle=NEURIPS,
  volume={32},
  year={2019}
}

@inproceedings{bourtoule2021machine,
  title={Machine unlearning},
  author={Bourtoule, Lucas and Chandrasekaran, Varun and Choquette-Choo, Christopher A and Jia, Hengrui and Travers, Adelin and Zhang, Baiwu and Lie, David and Papernot, Nicolas},
  booktitle=IEEESP,
  year={2021}
}

@InProceedings{courasia2023forget,
  title = 	 {Forget Unlearning: Towards True Data-Deletion in Machine Learning},
  author =       {Chourasia, Rishav and Shah, Neil},
  booktitle = 	 ICML,
  year = 	 {2023}
}

@inproceedings{golatkar2020eternal,
  title={Eternal sunshine of the spotless net: Selective forgetting in deep networks},
  author={Golatkar, Aditya and Achille, Alessandro and Soatto, Stefano},
  booktitle=CVPR,
  year={2020}
}

@inproceedings{foster2024fast,
  title={Fast machine unlearning without retraining through selective synaptic dampening},
  author={Foster, Jack and Schoepf, Stefan and Brintrup, Alexandra},
  booktitle=AAAI,
  year={2024}
}

@ARTICLE{Chundawat2023zeroshot,
  author={Chundawat, Vikram S. and Tarun, Ayush K. and Mandal, Murari and Kankanhalli, Mohan},
  journal=IEEETIFS, 
  title={Zero-Shot Machine Unlearning}, 
  year={2023}
}

@inproceedings{chowdhury2024towards,
  title={Towards scalable exact machine unlearning using parameter-efficient fine-tuning},
  author={Chowdhury, Somnath Basu Roy and Choromanski, Krzysztof and Sehanobish, Arijit and Dubey, Avinava and Chaturvedi, Snigdha},
  booktitle=ICLR,
  year={2025}
}

@inproceedings{gupta2021adaptive,
  title={Adaptive machine unlearning},
  author={Gupta, Varun and Jung, Christopher and Neel, Seth and Roth, Aaron and Sharifi-Malvajerdi, Saeed and Waites, Chris},
  booktitle=NEURIPS,
  year={2021}
}

@inproceedings{allouah2024utility,
  title={The utility and complexity of in-and out-of-distribution machine unlearning},
  author={Allouah, Youssef and Kazdan, Joshua and Guerraoui, Rachid and Koyejo, Sanmi},
  booktitle=ICLR,
  year={2025}
}

@inproceedings{schoepf2024potion,
  title={Potion: Towards poison unlearning},
  author={Schoepf, Stefan and Foster, Jack and Brintrup, Alexandra},
  booktitle=DMLR,
  year={2024}
}

@inproceedings{dou2025avoiding,
    title = {Avoiding Copyright Infringement via Large Language Model Unlearning},
    author = {Dou, Guangyao  and
      Liu, Zheyuan  and
      Lyu, Qing  and
      Ding, Kaize  and
      Wong, Eric},
    booktitle = {Findings of the Association for Computational Linguistics: NAACL 2025},
    year = {2025},

}

@article{thiel2023identifying,
  title={Identifying and eliminating {CSAM} in generative {ML} training data and models},
  author={Thiel, David},
  journal={Stanford Internet Observatory, Cyber Policy Center},
  volume={23},
  pages={3},
  year={2023}
}

@inproceedings{dwork2009differential,
  title={Differential privacy and robust statistics},
  author={Dwork, Cynthia and Lei, Jing},
  booktitle=STOC,
  pages={371--380},
  year={2009}
}

@misc{gdpr2016,
  author       = {{GDPR}},
  title        = {Regulation {(EU)} 2016/679 {(General Data Protection Regulation)}},
  year         = {2016},
  howpublished = {Official Journal of the European Union, L 119},
  note         = {European Parliament and Council of the European Union. Article 17: Right to erasure (``right to be forgotten''). Available via EUR-Lex.}
}

@inproceedings{bassily2014private,
  title={Private empirical risk minimization: Efficient algorithms and tight error bounds},
  author={Bassily, Raef and Smith, Adam and Thakurta, Abhradeep},
  booktitle=FOCS,
  year={2014}
}

@book{tao2023topics,
  title={Topics in random matrix theory},
  author={Tao, Terence},
  volume={132},
  year={2012},
  publisher={American Mathematical Society}
}

@article{chaudhuri2011differentially,
  title={Differentially private empirical risk minimization.},
  author={Chaudhuri, Kamalika and Monteleoni, Claire and Sarwate, Anand D},
  journal=JMLR,
  year={2011}
}

@inproceedings{balle2018improving,
  title={Improving the gaussian mechanism for differential privacy: Analytical calibration and optimal denoising},
  author={Balle, Borja and Wang, Yu-Xiang},
  booktitle=ICML,
  year={2018}
}

@article{yu2015useful,
  title={A useful variant of the {Davis--Kahan} theorem for statisticians},
  author={Yu, Yi and Wang, Tengyao and Samworth, Richard J},
  journal={Biometrika},
  year={2015}
}

@inproceedings{asi2020instance,
  title={Instance-optimality in differential privacy via approximate inverse sensitivity mechanisms},
  author={Asi, Hilal and Duchi, John C},
  booktitle=NEURIPS,
  year={2020}
}

@article{de2015identifying,
  title={Identifying modular flows on multilayer networks reveals highly overlapping organization in interconnected systems},
  author={De Domenico, Manlio and Lancichinetti, Andrea and Arenas, Alex and Rosvall, Martin},
  journal={Physical Review X},
  year={2015},
}

@inproceedings{kunegis2013konect,
  author    = {J{\'e}r{\^o}me Kunegis},
  title     = {{KONECT}: The Koblenz Network Collection},
  booktitle = {Proceedings of the 22nd International Conference on World Wide Web Companion (WWW '13 Companion)},
  year      = {2013},
}

@inproceedings{lecun2002gradient,
  title={Gradient-based learning applied to document recognition},
  author={LeCun, Yann and Bottou, L{\'e}on and Bengio, Yoshua and Haffner, Patrick},
  booktitle={Proceedings of the IEEE},
  year={1998},
}

@inproceedings{ding2021retiring,
  title={Retiring adult: New datasets for fair machine learning},
  author={Ding, Frances and Hardt, Moritz and Miller, John and Schmidt, Ludwig},
  booktitle=NEURIPS,
  year={2021}
}

@inproceedings{kumar2016edge,
  title={Edge weight prediction in weighted signed networks},
  author={Kumar, Srijan and Spezzano, Francesca and Subrahmanian, VS and Faloutsos, Christos},
  booktitle={IEEE 16th International Conference on Data Mining (ICDM)},
  year={2016},
}

@article{pitoski2021network,
  title={Network analysis of internal migration in {Austria}},
  author={Pitoski, Dino and Lampoltshammer, Thomas J and Parycek, Peter},
  journal={Digital Government: Research and Practice},
  year={2021},
}

@article{wojcik2019training,
  title={Training neural networks on high-dimensional data using random projection},
  author={W{\'o}jcik, Piotr Iwo and Kurdziel, Marcin},
  journal={Pattern Analysis and Applications},
  year={2019},
}

@misc{folktables,
  author       = {Frances Ding and Moritz Hardt and John P. Miller and Ludwig Schmidt},
  title        = {{folktables}: Datasets derived from the {US} Census},
  howpublished = {GitHub repository},
  note         = {Accessed 2026-01-29},
  url          = {https://github.com/socialfoundations/folktables}
}

@InProceedings{paul2013randomsvm,
  title = 	 {Random Projections for Support Vector Machines},
  author = 	 {Paul, Saurabh and Boutsidis, Christos and Magdon-Ismail, Malik and Drineas, Petros},
  booktitle = 	 AISTATS,
  year = 	 {2013},
  publisher =    {PMLR},
}

@article{kenthapadi2012privacy,
  title={Privacy via the {Johnson-Lindenstrauss} transform},
  author={Kenthapadi, Krishnaram and Korolova, Aleksandra and Mironov, Ilya and Mishra, Nina},
  journal={Journal of Privacy and Confidentiality},
  volume ={5},
  number ={1},
  year={2013}
}

@article{cherapanamjeri2025space,
  title={The Space Complexity of Learning-Unlearning Algorithms},
  author={Cherapanamjeri, Yeshwanth and Garg, Sumegha and Rajaraman, Nived and Sekhari, Ayush and Shetty, Abhishek},
  journal={arXiv preprint arXiv:2506.13048},
  year={2025}
}

@inproceedings{Haghifam2024privategeometricmedian,
author = {Haghifam, Mahdi and Steinke, Thomas and Ullman, Jonathan},
title = {Private geometric median},
year = {2024},
booktitle = NEURIPS,
}

@inproceedings{cohen2026protectingundeletedmachineunlearning,
      title={Protecting the Undeleted in Machine Unlearning}, 
      author={Aloni Cohen and Refael Kohen and Kobbi Nissim and Uri Stemmer},
      year={2026},
      booktitle=FORC, 
}

@article{regehr2026nearoptimalmachineunlearningutility,
      title={Near-Optimal Machine Unlearning Utility for Smooth Strongly Convex Losses}, 
      author={Matthew Regehr and Gautam Kamath and Andrew Lowy},
      year={2026},
      journal={arXiv preprint arXiv:2606.01527}
}

@article{benarroch2026certifiedperinstanceunlearningusing,
      title={Certified Per-Instance Unlearning Using Individual Sensitivity Bounds}, 
      author={Hanna Benarroch and Jamal Atif and Olivier Cappé},
      year={2026},
      journal={arXiv preprint arXiv:2602.15602}
}

@article{sarmasarkar2026auditingunlearningalgorithms,
      title={Auditing of Unlearning Algorithms}, 
      author={Sahasrajit Sarmasarkar and Anastasia Koloskova and Sanmi Koyejo},
      year={2026},
      journal={arXiv preprint arXiv:2607.05898},
}

@article{Cortez2009ModelingWP,
  title={Modeling wine preferences by data mining from physicochemical properties},
  author={P. Cortez and Antonio Lu{\'i}z Cerdeira and Fernando Almeida and Telmo Matos and Jos{\'e} Reis},
  journal={Decis. Support Syst.},
  year={2009},
}
